\documentclass[11pt]{article} 
\usepackage[margin=1.2in]{geometry}

\usepackage{mathtools}
\usepackage{amssymb} 
\usepackage{amsmath}
\usepackage{amsthm}
\usepackage{todonotes, comment}
\usepackage{breqn}
\usepackage{hyperref}
\usepackage{graphicx} 
\usepackage{tikz}
\usetikzlibrary{arrows.meta,calc,positioning}
\usepackage{booktabs}
\usepackage{subcaption}
\usepackage{enumerate}
\usepackage{xcolor}
\usepackage{hyperref}

\usepackage[numbers,sort&compress]{natbib}
\usepackage{array}

\theoremstyle{plain}
\newtheorem{theorem}{Theorem}[section]
\newtheorem{lemma}{Lemma}[section]
\newtheorem{proposition}{Proposition}[section]
\newtheorem{corollary}{Corollary}[section]

\theoremstyle{definition}
\newtheorem{remark}{Remark}[section]
\newtheorem{definition}{Definition}[section]

\newtheorem{example}{Example}[section]

\providecommand{\R}{\mathbb{R}} 
\providecommand{\1}{\mathbf{1}} 
\providecommand{\N}{\mathbb{N}} 
 
\providecommand{\cE}{\mathcal{E}}
\providecommand{\cA}{\mathcal{A}}
\providecommand{\cB}{\mathcal{B}}

\providecommand{\cC}{\mathcal{C}}

\title{Polyhedral Geometry 
of Time-to-First-Spike Neural Networks} 
\usepackage{authblk}
\author[1,4]{Manjot Singh}
\author[2,3]{Guido Mont\'ufar}
\author[1,4,5]{Gitta Kutyniok}
\affil[1]{Department of Mathematics, Ludwig-Maximilians-Universit\"at M\"unchen} 
\affil[2]{Departments of Mathematics and Statistics \& Data Science, University of California, Los Angeles} 
\affil[3]{Max Planck Institute for Mathematics in the Sciences, Leipzig} 
\affil[4]{Munich Center for Machine Learning}
\affil[5]{Department of Physics and Technology, University of Troms\o}

\date{}

\begin{document}

\maketitle

\begin{abstract} 
We study the expressivity of spiking neural networks, which provide a natural framework for asynchronous, event-driven computation complementary to conventional feedforward neural networks. 
We consider the time-to-first-spike model in a setting for which the input–output map is continuous and piecewise linear, with affine pieces governed by causal feasibility constraints that determine which presynaptic spikes occur before a neuron fires. 
We first show that each neuron's firing time admits a maxout-like representation with exponentially many, highly constrained affine pieces. 
We then formalize causal regions as polyhedral regions with fixed causal sets and derive upper and lower bounds on the maximal number of causal regions in both shallow and multilayer feedforward spiking networks. 
Our theoretical and experimental results show that spiking networks can generate richer partitions of the input space than conventional feedforward ReLU networks. 
\end{abstract}

\section{Introduction}
\label{section:introduction}

Spiking neural networks (SNNs) are increasingly studied as alternatives to standard artificial neural networks (ANNs) for two closely related reasons. 
First, they are event-driven models in which neurons communicate via discrete spikes, aligning with a large body of work in computational neuroscience on coding schemes that represent information through spike timing, firing rates, or spike trains \cite{gerstner2002spiking}. 
Second, from an engineering perspective, spiking computation is inherently asynchronous. When inputs are sparse or temporally structured, computation is triggered only by sparse events, offering the potential for low-latency and energy-efficient inference. 
These advantages motivate the deployment of SNNs in event-based sensing applications, such as event-based cameras and neuromorphic sensors \cite{wang2024adaptivecalibrationunifiedconversion}, 
in latency-critical perception tasks on mobile or high-mobility platforms \cite{wang2025eventcamerameetsmobile}, 
and in energy-constrained environments, such as satellite systems \cite{dold2025SNN_spaceapplication}. 

At the same time, there has been rapid progress on the algorithmic side of SNNs, including surrogate gradient methods \cite{eshraghian_2023_lessonsfromDL, Neftci_SurrogateGD_SNNs_2019, Guo_direct_training_review_2023} and spike-based learning algorithms \cite{bohte2002spikeprop, firstspike2021goltz, temporal_single_spike_backprop2020_comsa, Goeltz2025DelGrad}. 
Despite these advances, however, the theoretical understanding of SNNs remains significantly less developed than that of ANNs. 
For ANNs with piecewise linear activations, a mature body of work has developed a geometric theory based on the number and structure of activation regions (polyhedral regions of the input space on which the network computes an affine function), yielding quantitative insights into the roles of depth, width, and architectural bottlenecks \cite{montufar2014linearregions, raghu2017expressivepower, serra2018regions, hanin2019surprisinglyfewregion, montufar2022sharpboundsmaxout}. 
In contrast, a comparable theory for spiking models is still in its infancy.

SNNs encompass a broad range of neuron models and coding schemes, from rate-based representations to temporal codes and dynamical recurrent systems \cite{gerstner2002spiking, Gerstner_Kistler_Naud_Paninski_2014}. 
In this work, we focus on \emph{time-to-first-spike} (TTFS) coding  \cite{thorpe1996150ms, gerstner2002spiking, pulsed_maass1999, MAASS2001ontherelevanceoftime}, 
a widely used temporal scheme in which each neuron emits at most one spike and encodes its output through the timing of that spike. 
In feedforward TTFS-based SNNs, inputs are encoded as spike times, propagated through successive spiking layers, and decoded from the timing of the output. 
Rather than synchronous activation vectors, each layer processes real-valued spike times, with downstream computation triggered as spikes arrive. 
This perspective aligns naturally with the engineering motivations discussed above, since single-spike inference is inherently sparse and can support low-latency decision making. 

A key modeling ingredient in such networks is the synaptic response kernel (or postsynaptic response function). 
In the Spike Response Model \cite{Gerstner_Kistler_Naud_Paninski_2014}, a neuron integrates weighted copies of this response function and emits a spike when its membrane potential crosses a threshold. 
When the response function is piecewise linear in time, as for a ReLU-type response, the membrane potential is itself piecewise linear in time. 
When the weights are positive, the resulting output spike times have been shown to be continuous piecewise linear (CPWL) functions of the presynaptic spike times \cite{maass1997networks, singh2023expressivity, Neuman2025SNNTTFS}. 
This CPWL regime is practically relevant and facilitates direct comparison with ANNs with piecewise linear activations. 
We point out that TTFS-based SNNs do not generally induce CPWL mappings. In particular, when both positive and negative weights are allowed, the mappings realized by SNNs can, under certain conditions, be discontinuous. 
We isolate the positive-weights TTFS as a natural setting for studying the inductive biases of asynchronous spike-time computation. 
In what follows, we focus on positive-weight TTFS-based SNNs with CPWL response functions, referred to as \emph{pTTFS networks}. 
Unless stated otherwise, throughout this work we will use the terms \emph{SNN}, \emph{TTFS network}, and \emph{pTTFS network} interchangeably to refer to this setting.

Recent work in this direction has established approximation-theoretic and generalization bounds for positive-weights TTFS networks \cite{singh2023expressivity, Neuman2025SNNTTFS} and has begun to elucidate connections between SNNs and ReLU-ANNs, including regimes in which SNNs achieve comparable approximation rates. 
Despite this progress, separation results between ANNs and TTFS networks remain limited. Moreover, a systematic theory on the the number and geometry of their linear or causal regions is still underdeveloped, particularly for deeper networks beyond single neurons or highly specialized settings \cite{dold2025causalpiecesanalysingimproving, singh2023expressivity}. 
In particular, we lack sharp region-counting results that characterize how TTFS-specific causal constraints shape the geometry and complexity of the induced partition, give rise to structural biases distinct from those of ReLU-ANNs, and potentially identify regimes in which SNNs can outperform ReLU-ANNs.  

Our goal in this work is to develop a systematic theory of causal regions for TTFS networks. As a first step, we restrict attention to the positive-weight setting, leaving the discontinuous regime arising from negative weights as an important direction for future work. 
Within this setting, we study how architectural and parameter choices affect the number and geometry of causal regions, and compare the resulting bounds with corresponding region-count bounds for ANNs with piecewise linear activations.

\paragraph{Contributions}

We consider SNNs in the continuous piecewise linear setting and make the following contributions. 

\begin{enumerate}[(i)]

    \item 
     For a single neuron, we give an explicit geometric characterization of when a given causal set is realized. 
    This clarifies how SNNs partition the input space in contrast to standard ReLU networks.  
    In particular, we show that causal regions are governed by causal constraints, which arise as union of regions of an associated hyperplane arrangement. We also provide a polyhedral representation of the single neuron firing-time map, in which its affine pieces are encoded by a lifted TTFS polytope and an associated regular subdivision. 
    
    \item For shallow SNNs, we derive general upper bounds on the number of regions of the associated TTFS hyperplane arrangement, as well as asymptotic upper and lower bounds on the number of realizable causal regions. 
    In the shared-weight regime, where all hidden neurons share a common weight vector, we show that their causal sets are nested. Exploiting this structure, we obtain sharper lower bounds through an exact count on the realizable causal regions. 

    \item For deep SNNs, we derive a general lower bound by constructing a folding mechanism that produces exponentially many causal regions with depth. For networks with weights shared within each layer, we derive the corresponding upper bound. 
    We show that, once the firing-time order withing a layer is fixed, the causal sets of neurons in the subsequent layers must correspond to prefixes of this order.  
    We exploit this structure to derive our bound and further characterize which causal patterns can and cannot be realized. 
    
    \item We complement our theoretical results with experiments on CIFAR-10 and MNIST, focusing on the region complexity of randomly initialized networks. We study how the estimated number of causal regions in SNNs varies with architectural choices such as width and depth, compare SNNs with matched ReLU networks, and examine the effect of restricting SNNs to shared weight regime. Across these experiments, we find that SNNs can realize a large number of causal regions, with region complexity depending on architecture and initialization. 

\end{enumerate}

\paragraph{Outline} 
Section~\ref{section:related_work} reviews related work. 
Section~\ref{section:model} introduces the SNN model and notation. 
Section~\ref{section:single_neuron_ttfsmodel} analyzes the single-neuron case, introduces causal sets, causal regions, and causal patterns, and develops the geometric framework underlying our subsequent results. 
Section~\ref{section:boundsshallowsnn} establishes upper and lower bounds for shallow SNN, 
while 
Section~\ref{sec:deepSNN} develops corresponding bounds for deep networks. 
Section~\ref{section:experiments} presents experimental results. 
Finally, Section~\ref{section:conclusion} discusses open problems, the role of delays, and extensions beyond feedforward architectures.

\section{Related work}
\label{section:related_work}

In this section, we position our work relative to prior works on the expressivity and estimation of the number of activation regions in ReLU and Maxout networks, and expressivity of SNNs. 

\paragraph{Polyhedral geometry of ANNs} 
For ANNs with piecewise linear activations, a standard measure of expressivity is the number of linear, or activation, regions into which the network partitions its input space, that is, regions on which the network realizes an affine function of the input. This notion has been widely used to compare the expressivity of different architectures and to characterize depth-width trade-offs. 
Early works \cite{pascanu2013responseregions, montufar2014linearregions, raghu2017expressivepower, serra2018regions} established that depth can yield exponentially more linear regions than shallow networks with comparable numbers of neurons or parameters. 
This literature provides the conceptual reference for our study. We seek an analogous, architecture-aware understanding of how SNNs partition their input space. As we will see, the natural counterpart of activation regions in SNNs is given by \emph{causal regions}.

The geometry of linear regions in piecewise-linear neural networks is naturally described using polyhedral methods. 
For ReLU networks, region boundaries can be studied through hyperplane arrangements and bent hyperplanes in deeper layers. 
For higher-rank maxout networks \cite{maxoutnetwork2013goodfellow}, each unit has multiple affine pieces, leading to richer polyhedral subdivisions of the input space. 
More generally, CPWL functions have natural descriptions in terms of convex polytopes and support functions, as well as through tropical geometry \cite{MaclaganSturmfels2015Tropical, ETCJoswig2021, ZhangNaitzatLim2018TropicalGeometryDNN} and spline-based methods \cite{spline_theory_DL_balestriero_2018}. 
In particular, functions realized by ReLU and maxout networks can be represented as tropical rational functions \cite{ZhangNaitzatLim2018TropicalGeometryDNN, CharisopoulosMaragos2018, blm24}. 
Sharp bounds on the number of linear regions of maxout networks were obtained in \cite{montufar2022sharpboundsmaxout}. 
As we will see, a TTFS neuron with ReLU response function admits a maxout-like representation, but with strong dependencies among its affine pieces. Thus, TTFS SNNs can be regarded as highly structured maxout networks with exponentially large effective rank. This structure distinguishes them sharply from generic maxout networks. 
In particular, the affine pieces of each neuron cannot vary independently, and region-counting bounds for unconstrained maxout networks need not be tight. This motivates the development of TTFS-specific geometric and combinatorial tools for counting regions.

Although our analysis focuses on upper and lower bounds for the maximum number of regions, it is useful to briefly note how region-based viewpoints have been used more broadly in the literature. 
Maximum region counts characterize the expressive capacity of a network architecture in principle, whereas understanding the complexity typically realized in practice requires studying region statistics under parameter distributions, for example at random initialization or after training. 
For ReLU and maxout networks, expected-region analyses and empirical studies indicate that typical parameter choices may realize substantially fewer regions than the maximum possible \cite{hanin2019complexitylinearregions, hanin2019surprisinglyfewregion, tseran2021maxoutexpectedcomplexity,goujon2024num_regions}. 
More broadly, region geometry has been used to characterize qualitative aspects of ReLU networks, including the geometry of decision boundaries, the evolution of complexity during training, and robustness \cite{local_complexity_regions_DNNs_patel_2024, alfarra2022dbtropical, maksym2019robustness, huchette2025deeplearningmeetspolyhedral}. 
By comparison, a theory of typical or expected region complexity for SNNs remains largely undeveloped.

\paragraph{Expressivity of SNNs} 
The theory and practice of SNNs encompass a wide range of neuron models and coding schemes, whose choice is largely guided by the intended application or theoretical question. Broadly, SNNs are studied as models of neural computation in computational neuroscience and as a basis for efficient, brain-inspired AI systems \cite{MaassBIC2023, singhsurveysnn2026}. 
Under rate-based coding, a growing literature has established universality and approximation results for SNNs \cite{firingrateExpressivity_zhang2022, nguyen2025DLIFSNN, hundrieser2025universalrepresentationpropertyspiking}. 
For TTFS-based SNNs with a linear response function, the theoretical understanding is more developed. 
Recent works \cite{singh2023expressivity, Neuman2025SNNTTFS}, building on Maass’s early universality result \cite{maass1997networks}, have established expressivity and approximation results for networks in this regime. 
In particular, \cite{singh2023expressivity} showed that SNNs can realize discontinuous PWL maps when negative weights are allowed, and further derived complexity bounds for emulating multilayer ReLU networks with SNNs. 
In a complementary direction, \cite{Neuman2025SNNTTFS} studies positive-weight TTFS SNNs and establishes approximation and generalization bounds. 
Related work \cite{STANOJEVIC2023}, motivated primarily by training, derives a neuron-to-neuron mapping from ReLU ANN parameters to SNNs parameters that preserves the realized function. 

At the same time, ANN-SNN conversion results do not by themselves explain whether SNNs possess different computational capabilities from ANNs. 
Given the distinct mechanisms underlying spike-time computation, can we demonstrate meaningful differences in the capabilities of ANNs and SNNs? 
In this direction, Maass demonstrated advantages of SNNs over conventional neural network models on specific biologically motivated computational tasks \cite{maass1997networks}. 
Maass and Schmitt \cite{VCdimension_maass1999} further showed, through VC-dimension bounds, that programmable delays can substantially increase the computational capacity of SNNs relative to static threshold circuits. 
Yet comparatively little is known about how the internal computational structure of SNNs differs geometrically from that of ANNs. 
This is the perspective we pursue here. We use region complexity to quantify the geometric richness and structural inductive bias of the function class realized by TTFS SNNs.

Region complexity analysis for SNNs is still scarce compared to the extensive literature on ReLU ANNs. 
Initial steps in this direction appear in \cite{singh2023expressivity}, where small examples already reveal region structures that differ from those of ReLU ANNs, and more recently in \cite{dold2025causalpiecesanalysingimproving}, which introduces the notion of causal sets and empirically investigates their dependence on initialization. 
A key gap is the absence of systematic region-counting bounds for shallow and deep SNNs that account for the causal constraints inherent to TTFS computation, together with a comparison to corresponding bounds for ReLU ANNs. 
Our work addresses this gap by developing a causal-region framework for CPWL TTFS SNNs, establishing upper and lower bounds for shallow and deep networks, and a systematic description of the structures implemented by TTFS neurons.

\section{TTFS SNN model}
\label{section:model}

The study of spiking neural networks (SNNs) from computational and biological perspectives has led to a the development of a wide range of neuron models, from biophysically detailed dynamics, such as Hodgkin--Huxley-type model, to simplified abstractions, such as integrate-and-fire, leaky integrate-and-fire, and Izhikevich-type models \cite{Gerstner_Kistler_Naud_Paninski_2014}, as well as further reduced variants designed primarily for computational efficiency rather than biological realism. 
For theoretical analysis, it is common to focus on simplified dynamics that retain the essential computational structure while remaining analytically tractable. 

In this paper, we consider \emph{time-to-first-spike} (TTFS) coding in the single-spike regime, where each neuron emits at most one spike and information is represented by its firing time. 
Thus, neuron outputs are real-valued spike times, providing a natural input-output representation that can be composed by stacking layers. 
For clarity, we first introduce the underlying spike-time dynamics on a general network graph and then specialize to the feedforward architecture considered throughout the paper. Our notation largely follows \cite{singh2023expressivity, Neuman2025SNNTTFS}. 

\subsection{Spike-time dynamics}  

\begin{definition}[Spiking neural network]
\label{definition:SNN_general}
    Let $G = (V,E)$ be a graph, with subsets $V_{\rm in} \subset V$ and $V_{\rm out} \subset V$ denoting the input and output neurons, respectively, and with directed edges $E \subset V \times V$ representing synapses. 
    Each non-input neuron $v \in V\setminus V_{\rm in}$ has an associated firing threshold $\theta_v > 0$. 
    Each synapse $(u,v) \in E$ has the following associated attributes: 
    \begin{itemize}
    \item a \emph{synaptic weight} $w_{(u,v)} \in \R_{\geq 0}$, 
    \item a \emph{synaptic delay} $d_{(u,v)} \geq 0$,  
    \item a \emph{response function} $\varepsilon_{(u,v)}: \R \rightarrow \R_{\geq 0}$. 
    \end{itemize}
    We write $W := (w_{(u,v)})_{(u,v)\in E}, \: D := (d_{(u,v)})_{(u,v)\in E}, \: \cE :=(\varepsilon_{(u,v)})_{(u,v)\in E}$, and $\Theta := (\theta_u)_{u\in V\setminus V_{\rm in}}$ 
    for the corresponding collections of synaptic weights, delays, response functions, and firing thresholds, respectively. 
    An SNN is then specified by the tuple $\Phi = (G,W,D,\cE, \Theta)$. 

    The synaptic delay $d_{(u,v)} \geq 0$ represents the time it takes for a spike emitted by neuron $u$ to reach neuron $v$. Synaptic delays are specific to SNNs and have no direct counterpart in conventional ANNs. 
\end{definition}

In TTFS models, postsynaptic potentials are commonly constructed from shifted response functions, including ReLU, exponential, and alpha functions \cite{timestructure1995gerstner, temporal_single_spike_backprop2020_comsa, firstspike2021goltz, dold2025causalpiecesanalysingimproving}. 
For general response functions, the resulting input-output map need not be piecewise linear. 
In this work, we consider the ReLU response function, which under positive weights and thresholds leads to a continuous piecewise-linear input-output map while retaining the causal structure of spike-time computation. 

\begin{definition}[ReLU response] 
\label{definition:response_function}
    Let $G=(V,E)$ be a network graph. 
    For each synapse $(u,v) \in E$, we define the associated response function $\varepsilon_{(u,v)}: \R \rightarrow \R_{\geq 0}$ by 
    \begin{equation}\label{eq:response_function}
        \varepsilon_{(u,v)}(t) := \sigma(t),
    \end{equation}
    where $\sigma(t) = \max\{0,t\}$ denotes the ReLU activation. 
    Under this specification of the response functions, we suppress $\cE$ from the notation and write the SNN as $\Phi = (G,W,D,\Theta)$. 
\end{definition}

Given the firing times of the presynaptic neurons, the potential of a postsynaptic neuron accumulates the contributions of spikes that have already arrived. The neuron fires when this potential first reaches its threshold. 

\begin{definition}[Potential and firing time] 
\label{definition:firing_time}
    Let $\Phi = (G,W,D, \Theta)$ be an SNN with network graph $G = (V,E)$. 
    For each neuron $v \in V \setminus V_{\rm in}$, its potential $P_v: \R \rightarrow \R$,  
    is a function that at time $t$ takes value 
    \begin{equation}\label{eqn:P_v(t)}
    P_v(t) := \sum_{(u,v) \in E}^{} w_{(u,v)}\sigma(t - t_u - d_{(u,v)}) ,  
    \end{equation}
where $t_u \in \R$ denotes the firing time of the presynaptic neuron $u$. 
The firing time of neuron $v$ is the smallest time at which its potential reaches the threshold $\theta_v$: 
$$
t_v = \min\{t\in \R: P_v(t) = \theta_v\}.
$$
\end{definition}

Here and throughout, \emph{presynaptic} and \emph{postsynaptic} refer to the source and target neurons of a synapse, respectively, and we use the terms \emph{fire} and \emph{spike} interchangeably. 
For notational simplicity, we suppress the dependence of $P_v$ on the synaptic weights $w_{(u,v)}$, delays $d_{(u,v)}$, and presynaptic firing times $t_u$ whenever these are clear from context.

\begin{remark}[Well-posedness] 
Note that, for positive synaptic weights $w_{(u,v)}>0$, the potential $P_v(t)$ is continuous nondecreasing in $t$ and becomes strictly increasing once the first presynaptic spike reaches neuron $v$. 
For any fixed presynaptic firing times, $P_v(t)\to 0$ as $t\to-\infty$ and $P_v(t)\to+\infty$ as $t\to\infty$, provided that $v$ has at least one incoming synapse. 
Since $\theta_v>0$, 
every non-input neuron has a unique and finite firing time.  
\end{remark} 

The spike-time dynamics therefore define an input-output map. 

\begin{definition}[SNN realization]  
\label{definition:realization_SNN}
    Let $\Phi = (G,W,D,\Theta)$ be an SNN, and let $d= |V_{\rm in}|$ and $n = |V_{\rm out}|$. 
    The \emph{realization} of $\Phi$ is the function $R(\Phi) \colon \R^{d} \rightarrow \R^{n}$ that maps the firing times $(t_{v})_{v\in V_{\rm in}}$ of the input neurons to the firing times $(t_{v})_{v\in V_{\rm out}}$ of the output neurons. 
\end{definition}

\subsection{Feedforward TTFS networks} 

We are interested in feedforward SNNs, where neurons are organized into layers and synapses connect consecutive layers. 

\begin{definition}\label{def:ff_snn}
    Let $L, N_0, \dots, N_L \in \N$. 
    A \emph{feedforward SNN} of depth $L$ and layer widths $(N_\ell)_{\ell=0}^L$ is specified by a tuple 
    \begin{equation*}
        \Phi = \bigl( (W^\ell, D^\ell,\Theta^\ell)_{\ell=1}^L \bigr), 
    \end{equation*}
    where, for each $\ell=1,\dots,L$, 
    $W^\ell=(w^\ell_{ij}) \in \R_{\geq 0}^{N_\ell\times N_{\ell-1}}$ is the matrix of weights and $D^\ell=(d^\ell_{ij}) \in \R_{\geq 0}^{N_\ell\times N_{\ell-1}}$ the matrix of delays, for synapses from layer $\ell-1$ to layer $\ell$, 
    and $\Theta^\ell=(\theta^\ell_i) \in \R_{>0}^{N_\ell}$ is the vector of firing thresholds for neurons in layer $\ell$. 
    We write $W=(W^\ell)_{\ell=1}^L$, $D=(D^\ell)_{\ell=1}^L$, and $\Theta=(\Theta^\ell)_{\ell=1}^L$ for the collections of weights, delays, and thresholds of the entire network. 
    The total number of neurons is $N(\Phi) \coloneqq \sum_{\ell=0}^L N_\ell$. 
\end{definition}

Given a vector of input spike times $\vec t \in \R^{N_0}$, 
the membrane potentials and firing times are defined recursively, layer by layer, according to 
Definitions~\ref{definition:firing_time}.  
For each $\ell=1,\ldots, L$, this defines a layer map $F^\ell\colon \R^{N_{\ell-1}}\rightarrow \R^{N_\ell}$, 
and the realization of the network is the composition 
$$
F^{L}\circ F^{L-1}\circ \cdots \circ F^1. 
$$

\begin{remark}[Model considered in this work] 
In general TTFS-based SNNs, the postsynaptic potential of every neuron is given by a sum of shifted nonlinear response functions (e.g., ReLU, exponential or alpha-functions) \cite{timestructure1995gerstner, temporal_single_spike_backprop2020_comsa, firstspike2021goltz, dold2025causalpiecesanalysingimproving}, 
and the firing time is defined implicitly as the solution to a nonlinear equation which equates the potential to a threshold. 
For general response functions, the resulting input-output map need not be piecewise linear. 

Unless stated otherwise, throughout the paper we consider feedforward TTFS SNNs in the single-spike regime, with positive synaptic weights, nonnegative synaptic delays, positive firing thresholds, and ReLU response functions. 
In this setting, the realization map is continuous piecewise linear. 
\end{remark}

\begin{figure}[t]
    \centering
    \includegraphics[width=0.95\textwidth]{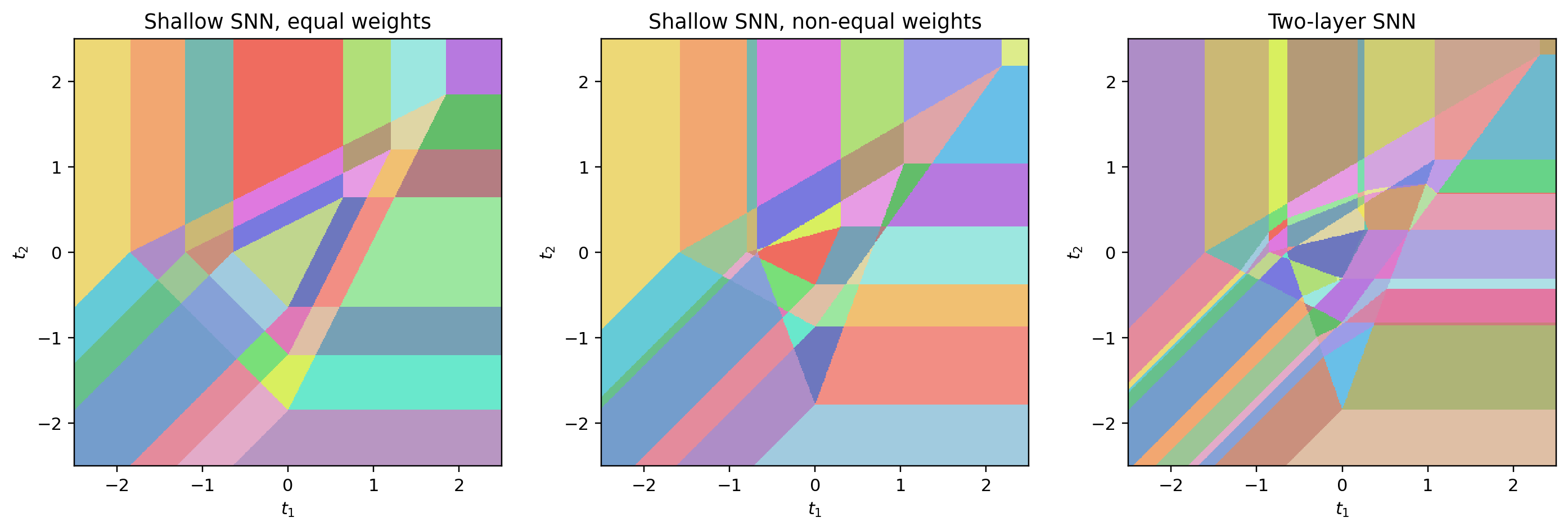}
    \caption{Causal region partitions of TTFS SNNs on the slice $t_3=0$ for input dimension $d=3$. (Left) Shallow SNN with three neurons, each with the identical weight vector. (Middle) Shallow SNN with three neurons having distinct weight vectors. (Right) Two-layer SNN with three neurons per layer and distinct positive weights.}
    \label{fig:snn_regions_d3L2L3}
\end{figure}

\subsection{Causal regions versus activation regions} 

Although the TTFS SNNs considered here and standard ReLU ANNs both realize continuous piecewise-linear maps, the mechanisms generating their linear regions are fundamentally different. In ReLU networks, the affine pieces are determined by activation patterns, which record the signs of the pre-activations. 
In TTFS SNNs, by contrast, the affine pieces are governed by causal set tuples. 

As we show below, the firing time of a TTFS neuron can be expressed as the minimum of finitely many affine functions indexed by \emph{causal sets}. 
For a given vector of input spike times, the causal set of a neuron records which presynaptic spikes arrive before the neuron fires. 
Conditional on a fixed causal set, the output firing time is an affine function of the corresponding input spike times. 
The associated \emph{causal region} is characterized by inequalities ensuring that precisely those spikes arrive before firing. 

This causal structure leads to an important geometric distinction from ReLU networks. 
A causal region need not coincide with a single cell of the hyperplane arrangement induced by the relevant affine comparisons; rather, it can be a union of multiple cells corresponding to the same causal set. 
Moreover, different neurons in a given layer can have different causal sets. 
Figure~\ref{fig:snn_regions_d3L2L3} illustrates the resulting causal-region geometry for several shallow and deep TTFS networks.

\section{Polyhedral geometry of a TTFS neuron} 
\label{section:single_neuron_ttfsmodel}

We begin by analyzing the input-output map of a single spiking neuron in the setting of positive weights and no synaptic delays. 
In this setting, the input-output map is continuous piecewise affine, with affine pieces indexed by causal sets that describe which presynaptic spikes arrive before the neuron fires. 
The resulting geometric and combinatorial structure provides the foundation for our subsequent analysis of shallow and deep SNNs. 

\subsection{Causal sets and causal regions} 
\label{subsec:single_neuron_causalsetsregions}

To keep the notation light, we work directly with the input spike times, synaptic weights, and firing threshold of the neuron under consideration, without explicit reference to the full network $\Phi$. 
Throughout this section, we consider an input layer of $d$ presynaptic neurons $u_1, \dots, u_d$ with spike-time vector $\vec t \coloneq (t_1, \dots, t_d) \in \R^d$, synaptic weights $w_1, \dots, w_d > 0$, synaptic delays set to zero, 
and a single postsynaptic neuron $v$ with firing threshold $\theta_v > 0$. 
For $d \in \N$, we write $[d] \coloneqq \{1,\dots, d\}$.

Following \eqref{eqn:P_v(t)}, the membrane potential of the postsynaptic neuron $v$ at time $t$ is given by 
\begin{equation}
\label{eq:membrane_potential}
    P_v(t) = \sum_{i=1}^d w_i \sigma(t - t_i) . 
\end{equation} 
The firing time $t_v$ is defined implicitly by the threshold-crossing condition 
\begin{equation}
\label{eq:tv_definition}
    t_v(\vec t) \coloneq \min \bigl\{ t \in \R \colon P_v(t) = \theta_v \bigr\}.
\end{equation}
For any fixed $\vec t$ and positive weights $w_1,\ldots, w_d>0$, the potential 
$P_v(t)$ is a continuous and nondecreasing in $t$, becomes strictly increasing after the first presynaptic spike, and satisfies $P_v(t)\to +\infty$ as $t\to+\infty$. 
Since $\theta_v>0$, the minimum in the definition of the firing time $t_v$ is attained uniquely. 
The \emph{causal set} at $\vec t$ is the subset of presynaptic neurons that fire strictly before the postsynaptic neuron, 
$S = \{i\in [d]\colon t_i < t_v(\vec t)\}$. Since $\theta_v>0$, the causal set is nonempty. See Figure~\ref{fig:ttfs_causal_pattern_clean} for an illustration.

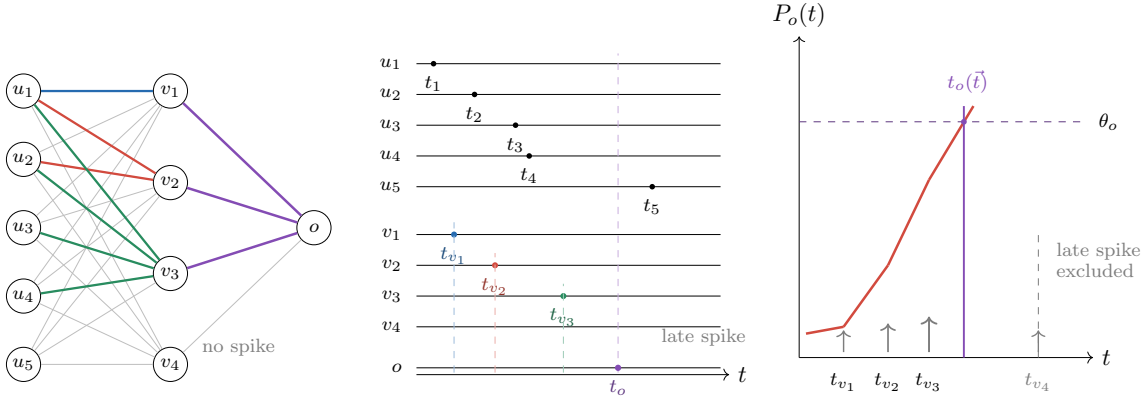
\begin{figure}[t]
\centering
\resizebox{6in}{!}{%
\begin{tikzpicture}[font=\small, line width=0.35pt]

\definecolor{cBlue}{RGB}{38,107,176}
\definecolor{cRed}{RGB}{210,74,62}
\definecolor{cGreen}{RGB}{40,140,95}
\definecolor{cPurple}{RGB}{132,76,180}
\definecolor{cTeal}{RGB}{30,150,165}
\definecolor{cGray}{RGB}{125,125,125}
\definecolor{NetBlue}{RGB}{120,170,230}
\definecolor{NetBlueDark}{RGB}{60,110,180}

\node[font=\scriptsize, circle, fill=white, draw, inner sep=0, minimum size = .5cm] (U1) at (0.80,5) {$u_1$};

\node[font=\scriptsize, circle, fill=white, draw, inner sep=0, minimum size = .5cm] (U2) at (0.80,4) {$u_2$};

\node[font=\scriptsize, circle, fill=white, draw, inner sep=0, minimum size = .5cm] (U3) at (0.80,3) {$u_3$};

\node[font=\scriptsize, circle, fill=white, draw, inner sep=0, minimum size = .5cm] (U4) at (0.80,2) {$u_4$};

\node[font=\scriptsize, circle, fill=white, draw, inner sep=0, minimum size = .5cm] (U5) at (0.80,1) {$u_5$};

\node[font=\scriptsize, circle, fill=white, draw, inner sep=0, minimum size = .5cm] (V1) at (2.95,5) {$v_1$};

\node[font=\scriptsize, circle, fill=white, draw, inner sep=0, minimum size = .5cm] (V2) at (2.95,3.66) {$v_2$};

\node[font=\scriptsize, circle, fill=white, draw, inner sep=0, minimum size = .5cm] (V3) at (2.95,2.33) {$v_3$};

\node[font=\scriptsize, circle, fill=white, draw, inner sep=0, minimum size = .5cm] (V4) at (2.95,1) {$v_4$};

\node[font=\scriptsize, circle, fill=white, draw, inner sep=0, minimum size = .5cm] (O) at (5.05,3) {$o$};

\foreach \i in {1,2,3,4,5}{
\foreach \j in {1,2,3,4}{
  \draw[cGray!50, line width=0.25pt] (U\i)--(V\j);
}
}

\foreach \j in {1,2,3,4}{
  \draw[cGray!50, line width=0.25pt] (V\j)--(O);
}

\draw[cBlue, line width=0.95pt]  (U1)--(V1);

\draw[cRed, line width=0.95pt]   (U1)--(V2);

\draw[cRed, line width=0.95pt]   (U2)--(V2);

\draw[cGreen, line width=0.95pt] (U1)--(V3);

\draw[cGreen, line width=0.95pt] (U2)--(V3);
\draw[cGreen, line width=0.95pt] (U3)--(V3);
\draw[cGreen, line width=0.95pt] (U4)--(V3);

\draw[cPurple, line width=1.05pt] (V1)--(O);

\draw[cPurple, line width=1.05pt] (V2)--(O);
\draw[cPurple, line width=1.05pt] (V3)--(O);

\node[font=\scriptsize, anchor=west, cGray] at (3.28,1.25) {no spike};

\draw[->] (6.55,0.85) -- (11.125,0.85) node[anchor=west] {$t$}; 

\node[font=\scriptsize, anchor=east] at (6.47,5.40) {$u_1$}; \draw (6.55,5.40)--(11,5.40);
\node[font=\scriptsize, anchor=east] at (6.47,4.95) {$u_2$}; \draw (6.55,4.95)--(11,4.95);
\node[font=\scriptsize, anchor=east] at (6.47,4.50) {$u_3$}; \draw (6.55,4.50)--(11,4.50);
\node[font=\scriptsize, anchor=east] at (6.47,4.05) {$u_4$}; \draw (6.55,4.05)--(11,4.05);
\node[font=\scriptsize, anchor=east] at (6.47,3.60) {$u_5$}; \draw (6.55,3.60)--(11,3.60);

\node[font=\scriptsize, anchor=east] at (6.47,2.90) {$v_1$}; \draw (6.55,2.90)--(11,2.90);
\node[font=\scriptsize, anchor=east] at (6.47,2.45) {$v_2$}; \draw (6.55,2.45)--(11,2.45);
\node[font=\scriptsize, anchor=east] at (6.47,2.00) {$v_3$}; \draw (6.55,2.00)--(11,2.00);
\node[font=\scriptsize, anchor=east] at (6.47,1.55) {$v_4$}; \draw (6.55,1.55)--(11,1.55);

\node[font=\scriptsize, anchor=east] at (6.47,0.95) {$o$}; \draw (6.55,0.95)--(11,0.95);

\fill (6.8,5.40) circle (1.15pt); \node[font=\scriptsize, anchor=north] at (6.8,5.35) {$t_1$};
\fill (7.4,4.95) circle (1.15pt); \node[font=\scriptsize, anchor=north] at (7.4,4.9) {$t_2$};
\fill (8,4.50) circle (1.15pt); \node[font=\scriptsize, anchor=north] at (8,4.45) {$t_3$};
\fill (8.2,4.05) circle (1.15pt); \node[font=\scriptsize, anchor=north] at (8.2,4) {$t_4$};
\fill (10,3.60) circle (1.15pt); \node[font=\scriptsize, anchor=north] at (10,3.55) {$t_5$};

\fill[cBlue]  (7.1,2.90) circle (1.25pt); \node[font=\scriptsize, anchor=north, cBlue!70!black] at (7.1,2.85) {$t_{v_1}$};
\fill[cRed]   (7.7,2.45) circle (1.25pt); \node[font=\scriptsize, anchor=north, cRed!70!black]  at (7.7,2.4) {$t_{v_2}$};
\fill[cGreen] (8.7,2.00) circle (1.25pt); \node[font=\scriptsize, anchor=north, cGreen!70!black] at (8.7,1.95) {$t_{v_3}$};

\draw[cBlue!45, dashed]   (7.1,0.88)--(7.1,3.08);
\draw[cRed!45, dashed]    (7.7,0.88)--(7.7,2.63);
\draw[cGreen!45, dashed]  (8.7,0.88)--(8.7,2.18);
\draw[cPurple!45, dashed] (9.50,0.88)--(9.50,5.55);

\fill[cPurple] (9.50,0.95) circle (1.30pt); \node[font=\scriptsize, anchor=north, cPurple!70!black] at (9.50,0.9) {$t_o$};

\node[font=\scriptsize, anchor=west, cGray] at (10,1.4) {late spike};

\draw[->] (12.15,1.10) -- (16.45,1.10) node[anchor=west] {$t$};
\draw[->] (12.15,1.10) -- (12.15,5.80) node[anchor=south] {$P_o(t)$};

\draw[dashed, 
cPurple!70!black
] (12.15,4.55) -- (16.35,4.55);
\node[font=\scriptsize, anchor=west, 
] at (16.40,4.55) {$\theta_o$};

\draw[cGray, line width=0.60pt, ->] (12.80,1.18) -- (12.80,1.5);
\draw[cGray, line width=0.70pt, ->] (13.45,1.18) -- (13.45,1.6);
\draw[cGray, line width=0.80pt, ->] (14.05,1.18) -- (14.05,1.7);
\draw[cGray,  line width=0.55pt, ->] (15.65,1.18) -- (15.65,1.5);

\node[font=\scriptsize, anchor=north] at (12.80,1) {$t_{v_1}$};
\node[font=\scriptsize, anchor=north] at (13.45,1) {$t_{v_2}$};
\node[font=\scriptsize, anchor=north] at (14.05,1) {$t_{v_3}$};
\node[font=\scriptsize, anchor=north, cGray] at (15.65,1) {$t_{v_4}$};

\draw[cRed, line width=1.05pt] (12.25,1.45) -- (12.80,1.55);
\draw[cRed, line width=1.05pt] (12.80,1.55) -- (13.45,2.45);
\draw[cRed, line width=1.05pt] (13.45,2.45) -- (14.05,3.70);
\draw[cRed, line width=1.05pt] (14.05,3.70) -- (14.70,4.78);

\draw[
cPurple, line width=0.80pt] (14.56,1.10) -- (14.56,4.78);
\fill[
cPurple] (14.56,4.55) circle (1.2pt);
\node[font=\scriptsize, anchor=south, cPurple 
] at (14.62,4.84) {$
t_o(\vec t)$};

\draw[cGray, dashed] (15.65,1.10) -- (15.65,2.95);
\node[font=\scriptsize, anchor=west, cGray] at (15.78,2.65) {late spike};
\node[font=\scriptsize, anchor=west, cGray] at (15.78,2.37) {excluded};

\end{tikzpicture}
}
\caption{
Schematic for the causal pattern in SNNs. 
(a) The firing time of each noninput neuron depends on presynaptic spikes that arrive before it fires (colored edges) and they form its causal set. Neuron $v_4$ spikes after the output neuron has already fired and therefore does not contribute to the output, illustrating sparse computation. 
(b) The one-shot spike raster emphasizes asynchronous signal propagation, hidden-layer neurons fire after input spikes, and some neurons that fire too late may become irrelevant for the final output. 
(c) For a single neuron, the membrane potential is a piecewise linear function; the threshold crossing time $\tau$ depends only on the spikes arriving before $\tau$.}
\label{fig:ttfs_causal_pattern_clean}
\end{figure}

We next study the firing time $t_v$ as a function of the input spike times $\vec t$. 
Observe that the potential 
$P_v(t; \vec t) = \sum_{i=1}^d w_i \sigma(t-t_i)$ 
is a continuous piecewise linear function of $(t, \vec t)\in\mathbb{R}^{d+1}$, with linear regions separated by the hyperplanes $t-t_i=0$, $i=1,\ldots, d$. 
For each subset $S\subseteq[d]$, consider the polyhedral region $L_S = \{(t,\vec t)\colon t>t_i\,\forall i\in S \text{ and } t\leq t_j\,\forall j\not\in S\}$. 
On $L_S$, the potential is linear and takes the form $P_v(t,\vec t) = \sum_{i\in S} w_i (t-t_i)$. 
Geometrically, the threshold condition defining the firing time gives the level set 
\begin{equation} 
    \Sigma_{\theta_v} \coloneq \{(t, \vec t) \in \R^{d+1} \colon P_v(t; \vec t) = \theta_v\}. 
\end{equation}
For each fixed input $\vec t$, the firing time $t_v(\vec t)$ is the unique value of $t$ such that $(t,\vec t) \in \Sigma_{\theta_v}$. 
Equivalently, the level set $\Sigma_{\theta_v}$ is the graph of the firing-time map $\vec t\mapsto t_v$. 

We now explain how the preceding geometric picture induces regions in the input space. 
Restricting $\Sigma_{\theta_v}$ to $L_S$ gives $\sum_{i\in S}w_i(t-t_i) = \theta_v$. 
Solving for $t$, yields the \emph{affine firing time} 
\begin{equation}
    \label{eq:tv_S}
        t_v^S(\vec{t}) \coloneq \frac{\theta_v + \sum_{i\in S} w_it_i}{W_S}, \quad \text{where } W_S \coloneq \sum_{i\in S}w_i . 
\end{equation}
Thus, $t_v^S$ is the firing time obtained under the assumption that precisely the presynaptic neurons indexed by $S$ fire before neuron $v$. 
This candidate coincides with the actual firing time if and only if $(t_v^S(\vec t), \vec t) \in L_S$, 
which is equivalent to the \emph{causal feasibility inequalities} 
\begin{equation}
    \label{eq:ineqs}
        t_v^S(\vec{t}) > t_i, \quad \text{for all } i\in S, \quad \text{and } t_v^S(\vec{t}) \leq t_j, \quad \text{for all } j\notin S. 
\end{equation}
We therefore define the \emph{causal region} associated with $S$ by 
\begin{equation}\label{eq:region_Rs}
    R_S = \{\vec t \in \R^d \colon t_v^S(\vec{t}) > t_i, \text{ for all } i\in S, \text{ and } t_v^S(\vec{t}) \leq t_j, \text{ for all } j\notin S \}. 
\end{equation}
Thus, $R_S$ consists precisely of those inputs $\vec t$ for which $S$ is the causal set and the firing time is given by $t_v = t_v^S(\vec t)$. 
Since $t_v^S$ is affine, $R_S$ is defined by affine inequalities and is therefore a convex polyhedron. 
Geometrically, the level set $\Sigma_{\theta_v}$ is a piecewise-affine hypersurface in $\mathbb{R}^{d+1}$, 
while its projection onto the input coordinates partitions $\mathbb{R}^d$ into causal regions on which the firing-time map is affine. 
See Figure~\ref{fig:level-set-regions} for an illustration.

\begin{figure}[t]
\centering 
\begin{tikzpicture}[
scale = 1.2, 
    x={(.45cm,-.28cm)},
y={(0.45cm,0.28cm)},
z={(0cm,1.0cm)},
    line join=round,
    line cap=round
]

\definecolor{c1}{RGB}{210,228,245}
\definecolor{c12}{RGB}{220,238,215}
\definecolor{c2}{RGB}{250,224,204}

\filldraw[
    fill=c1,
    fill opacity=.85,
    draw=black!60,
    thick
]
(-2,-1,-1) --
( 1, 2, 2) --
( 1, 4, 2) --
(-2, 1,-1) --
cycle;

\draw[->,thick]
    (-3.5,0,0) -- (3,0,0)
    node[right] {$t_1$};

\draw[->,thick]
    (0,-3.5,0) -- (0,3,0)
    node[above right] {$t_2$};

\draw[->,thick]
    (0,0,-2) -- (0,0,3)
    node[above] {$t$};

\filldraw[
    fill=c12,
    fill opacity=.90,
    draw=black!60,
    thick
]
(-2,-1,-1) --
( 1, 2, 2) --
( 2, 1, 2) --
(-1,-2,-1) --
cycle;

\filldraw[
    fill=c2,
    fill opacity=.85,
    draw=black!60,
    thick
]
(-1,-2,-1) --
( 2, 1, 2) --
( 4, 1, 2) --
( 1,-2,-1) --
cycle;

\draw[very thick]
    (-2,-1,-1) -- (1,2,2);

\draw[very thick]
    (-1,-2,-1) -- (2,1,2);

\node[
    text opacity=1,
    inner sep=2pt
]
at (-.25,3,1.9)
{$t_v=t_1+1$};

\node[
    text opacity=1,
    inner sep=2pt
]
at (-2,-2,0.1)
{$t_v=\frac{t_1+t_2+1}{2}$};

\node[
    text opacity=1,
    inner sep=2pt
]
at (3,-.5,-.5)
{$t_v=t_2+1$};

\draw[thick, dotted, opacity=.2]
    (-1,0,0) -- (-1,2,0);

\draw[thick, dotted]
    (-1,0,0) -- (0,-1,0);

\draw[thick, dotted]
    (0,-1,0) -- (2,-1,0);

\draw[->,thick]
    (2,0,0) -- (3,0,0)
    node[right] {$t_1$};
\draw[-,thick]
    (-1,0,0) -- (-3.5,0,0);    

\draw[-,thick]
    (0,-1,0) -- (0,-3.5,0);        

\draw[->,thick]
    (0,0,0.5) -- (0,0,3)
    node[above] {$t$};    

\end{tikzpicture}
\qquad 
\qquad
\begin{tikzpicture}[scale=.95, font=\small]

\definecolor{c1}{RGB}{210,228,245}
\definecolor{c12}{RGB}{220,238,215}
\definecolor{c2}{RGB}{250,224,204}

\def\xmin{-3}
\def\xmax{3}
\def\ymin{-3}
\def\ymax{3}

\fill[c1]
    (-3,-2) --
    (-3, 3) --
    ( 2, 3) --
    cycle;

\fill[c12]
    (-3,-3) --
    (-2,-3) --
    ( 3, 2) --
    ( 3, 3) --
    ( 2, 3) --
    (-3,-2) --
    cycle;

\fill[c2]
    (-2,-3) --
    ( 3,-3) --
    ( 3, 2) --
    cycle;

\draw[very thick]
    (-3,-2) -- (2,3);

\draw[very thick]
    (-2,-3) -- (3,2);

\draw[->,thick]
    (-3.3,0) -- (3.4,0)
    node[right] {$t_1$};

\draw[->,thick]
    (0,-3.3) -- (0,3.4)
    node[above] {$t_2$};

\node[rotate=90] at (-2,1.8)
{
    $t_v(\vec t) = t_1+1$
};

\node[rotate =45] at (-1.75,-1.75)
{
    $t_v(\vec t) = \frac{t_1+t_2+1}{2}$
};

\node at (1.8,-2)
{
    $t_v(\vec t) = t_2+1$
};

\node[
    rotate=45,
    text opacity=.9,
    inner sep=1.5pt
]
at (1.25,1.95)
{$t_2=t_1+1$};

\node[
    rotate=45,
    text opacity=.9,
    inner sep=1.5pt
]
at (1.65,0.95)
{$t_2=t_1-1$};

\draw[thick,dotted]
    (-1,0) -- (-1,3);

\draw[thick,dotted]
    (-1,0) -- (0,-1);

\draw[thick,dotted]
    (0,-1) -- (3,-1);

\node[
    text opacity=.9,
    inner sep=1.5pt
]
at (2,-.8)
{$t_v=0$};

\end{tikzpicture}

\caption{For $d=2$, $w_1=w_2=1$, and $\theta_v=1$, the threshold level set of the potential function, $\Sigma_{\theta_v} =\{ (t, \vec t) \colon P_v(t; \vec t) = \theta_v \} \subseteq \mathbb{R}^{d+1}$, and its projection onto the input space $\mathbb{R}^{d}$, illustrating the linear regions of the firing time $t_v(\vec t)$.} 
\label{fig:level-set-regions}
\end{figure}
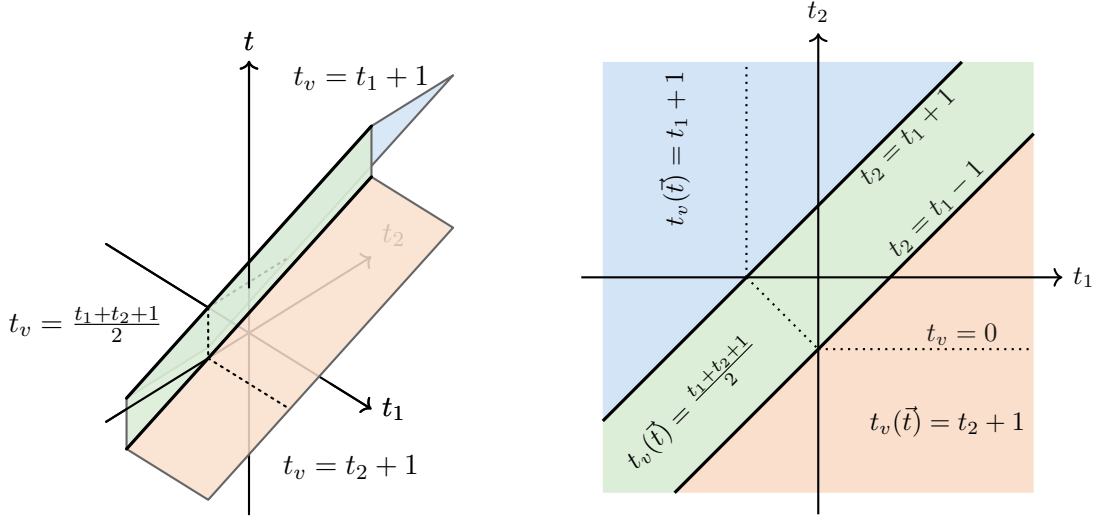

Having established that each input belongs to a unique causal region, we next show that the firing time can be written as the pointwise minimum of finitely many affine functions. In particular, the firing-time map is piecewise affine and concave. 
Although this result already appears in \cite{singh2023expressivity}, we state it here for completeness. 

\begin{proposition} 
\label{cor:tv_equals_tvs}
   The firing-time map satisfies 
    \begin{equation*}
        t_v (\vec t) = t_v^S(\vec t), \quad \text{for every } \vec t \in R_S, \text{ for every } S\subseteq [d], S\neq\emptyset.
    \end{equation*}    
    In particular, $t_v$ is affine on each causal region $R_S$. 
    Moreover, it admits the representation 
    \begin{equation}
    \label{eq:tv_as_min}
        t_v(\vec t) = \min_{S \subseteq [d], S\neq \emptyset } t^S_v(\vec t), \quad \text{for all } \vec t \in\mathbb{R}^d.
    \end{equation}
    Consequently, $t_v$ is a continuous piecewise-affine concave function. 
\end{proposition}

\begin{proof}
Fix $\vec t\in\R^d$ and let $\tau:=t_v(\vec t)$ be the corresponding firing time. 
If $S$ is the causal set of $\vec t$, then by construction $\tau=t_v^S(\vec t)$, so $t_v$ is affine on each causal region $R_S$. 

It remains to prove the minimum representation. For any nonempty $S\subseteq[d]$, define
\[
P_v^S(t;\vec t)
\coloneqq
\sum_{i\in S}w_i(t-t_i).
\]
At the firing time $\tau$, we have
\[
P_v^S(\tau;\vec t)
=
\sum_{i\in S}w_i(\tau-t_i)
\leq
\sum_{i=1}^d w_i\sigma(\tau-t_i)
=
P_v(\tau;\vec t)
=
\theta_v.
\]
Indeed, if $t_i<\tau$, the corresponding term is the same in the linear and ReLU expressions, whereas if $t_i\geq\tau$, the linear term is nonpositive while the ReLU term is zero.

On the other hand, by the definition of $t_v^S(\vec t)$,
\[
P_v^S(t_v^S(\vec t);\vec t)=\theta_v.
\]
Since $P_v^S(t;\vec t)$ is strictly increasing in $t$, these two relations imply
\[
t_v^S(\vec t)\geq \tau.
\]
If $S$ is the causal set of $\vec t$, then
$P_v^S(\tau;\vec t)=P_v(\tau;\vec t)=\theta_v$, and hence
$t_v^S(\vec t)=\tau$.
Therefore,
\[
t_v(\vec t)
=
\min_{\emptyset\neq S\subseteq[d]} t_v^S(\vec t).
\]
Since $t_v$ is the pointwise minimum of finitely many affine functions, it is continuous, piecewise affine, and concave.
\end{proof}

Proposition~\ref{cor:tv_equals_tvs} shows that, for any given input $\vec t$, once the corresponding causal set $S$ is known, 
the firing time is given by the affine expression \eqref{eq:tv_S}. 
Alternatively, \eqref{eq:tv_as_min} provides a global representation of $t_v$, without requiring prior knowledge of its causal set. 

It is worth noting that computing $t_v$ for a given input $\vec t$ does not require evaluating $t_v^S$ over all $2^d-1$ possible nonempty subsets $S\subseteq[d]$. 
Instead, ordering the input spike times as $t_{i_1}\leq \cdots\leq t_{i_d}$, the causal set must be a prefix of this ordering. Hence it suffices to consider the $d$ candidate sets of the form $S = \{i_1,\ldots, i_k\}$, $k=1,\ldots, d$. 

\begin{remark}[Number of regions] 
Since $[d]$ has $2^d-1$ nonempty subsets, the firing-time map has at most $2^d-1$ causal regions, and therefore at most $2^d-1$ linear regions. 
In fact, this bound is attained whenever $w_1,\ldots, w_d > 0$ (Lemma 13 in \cite{singh2023expressivity}). 
In particular, fix any nonempty $S$ and choose $\vec t\in\mathbb{R}^d$ with $t_i=0$ for all $i\in S$ and $t_j=c>\frac{\theta_v}{W_S}$ for all $j\not\in S$. Then $t_v^S(\vec t) = \frac{\theta_v +\sum_{i\in S} w_it_i}{W_S} = \frac{\theta_v}{W_S}>0$, and hence $t_v^S(\vec t)>t_i$ for all $i\in S$ and $t_v^S(\vec t)\leq t_j$ for all $j\not\in S$. Thus, all causal feasibility inequalities in \eqref{eq:ineqs} are satisfied and $\vec t\in R_S$. Therefore, $R_S$ is nonempty for every nonempty $S\subseteq[d]$. 
\end{remark}

\subsection{Causal hyperplanes} 
\label{sec:causal-hyperplanes}

The previous subsection shows that the causal regions are determined by the affine inequalities in \eqref{eq:ineqs}. 
For the purpose of region counting, it is useful to isolate the hyperplanes describing the boundaries of these constraints. 
We show that these hyperplanes form an arrangement whose regions refine the causal region partition.

For any $S\subseteq[d]$ and any $i\in S$, we have 
\begin{equation}\label{eq:strict_ineq}
    t_v^S(\vec{t}) > t_i \iff \frac{\theta_v+\sum_{i\in S}w_it_i}{W_S} > t_i \iff \theta_v + \sum_{i \in S}w_it_i - W_St_i > 0. 
\end{equation} 
This defines a halfspace whose boundary is the hyperplane 
\begin{equation}\label{eq:H_i,Sstrict}
    H_{S, i}^{\rm str} \coloneqq \{\vec t \in \R^d \colon \theta_v + \sum_{i \in S}w_it_i - W_St_i = 0\}.
\end{equation} 
Similarly, for 
any index $j \notin S$, 
\begin{equation}\label{eq:halfspace_condition}
    t_v^S(\vec{t}) \leq t_j \iff \frac{\theta_v+\sum_{i\in S}w_it_i}{W_S} \leq t_j \iff \theta_v + \sum_{i\in S}w_it_i - W_St_j \leq 0.
\end{equation}
This defines a halfspace in $\mathbb{R}^d$ whose boundary is the hyperplane
\begin{equation}\label{eq:HSj}
    H_{S,j} \coloneqq \{\vec{t} \in \R^d \colon \theta_v + \sum_{i\in S}w_it_i - W_St_j = 0\}.
\end{equation}

Starting from a point $\vec t$ at which $S$ is feasible, crossing any of these hyperplanes causes $S$ to become infeasible. 
Moreover, for $\lvert S \rvert \geq 2$, we have $H_{S,i}^{\rm str} = H_{S\setminus \{i\},i}$. 
For $\lvert S \rvert = 1$, the corresponding 
equality reduces to $\theta_v = 0$, which is impossible since $\theta_v >0$. 
Thus, all genuine boundaries are of the form \eqref{eq:HSj}. This motivates the following arrangement. 

\begin{definition}
The \emph{TTFS arrangement} is the hyperplane arrangement 
    \begin{equation}\label{eq:Attfs_def}
        \cA_{\rm TTFS} \coloneqq \{H_{S,j} \colon \emptyset \neq S\subseteq [d], j\notin S\} . 
    \end{equation}
We denote by $\cC(\cA_{\rm TTFS})$ the set of connected components of $\R^d \setminus \cA_{\rm TTFS}$. 
\end{definition}

Equivalently, $\cA_{\rm TTFS}$ consists of the hyperplanes
$t_v^S(\vec t)=t_j$ obtained by restricting the boundaries $t=t_j$ of the extended linear regions $L_S$ to the threshold level set and expressing them in input-space coordinates.
Within each region of $\cC(\cA_{\rm TTFS})$, none of the feasibility inequalities change sign. 
We record this observation in the following lemma. 

\begin{lemma}\label{lem:causal_feas_and_order_constant}
   Let $C \in \cC(\cA_{\rm TTFS})$. 
   For any nonempty $S \subseteq [d]$, the sign of all inequalities in \eqref{eq:ineqs} is constant over $\vec t \in C$. 
   Consequently, there exists a unique nonempty $S\subseteq[d]$ such that $C\subseteq R_S$. 
\end{lemma}

In particular, the causal set is constant on each region of the hyperplane arrangement. 
We therefore define the 
\emph{TTFS region map} 
\begin{equation*}
        \ell \colon \cC(\cA_{\rm TTFS}) \to 2^{[d]}\setminus\{\emptyset\}, \qquad \ell(C) \coloneqq S, 
\end{equation*} 
where $S$ is the unique causal set with $C\subseteq R_S$. This map is well-defined by Lemma~\ref{lem:causal_feas_and_order_constant}.

\begin{proposition}
\label{prop:arrangement_refines_causal_regions}
For any nonempty $S\subseteq[d]$, we have
\begin{equation}
\label{eq:R_S_union}
R_S \setminus \bigcup_{H\in\cA_{\rm TTFS}} H
=
\bigcup_{\substack{C\in\cC(\cA_{\rm TTFS})\\ \ell(C)=S}} C.
\end{equation}
In particular, every region of the TTFS arrangement is contained in a unique causal region. Thus, the region decomposition induced by $\cA_{\rm TTFS}$ refines the decomposition into causal regions.
\end{proposition}

\begin{proof}
By Lemma~\ref{lem:causal_feas_and_order_constant}, the causal set is constant on every region $C\in\cC(\cA_{\rm TTFS})$. Hence, if $\ell(C)=S$, then $C\subseteq R_S$, which proves the inclusion from right to left in \eqref{eq:R_S_union}.

Conversely, let
\[
\vec t\in R_S\setminus\bigcup_{H\in\cA_{\rm TTFS}}H.
\]
Since $\vec t$ does not lie on any hyperplane of the arrangement, it belongs to a unique region $C\in\cC(\cA_{\rm TTFS})$. Since $\vec t\in R_S$, its causal set is $S$, and therefore $\ell(C)=S$. This proves the reverse inclusion.
\end{proof}

\begin{remark}
The arrangement $\cA_{\rm TTFS}$ typically gives a strict refinement of the causal region partition. 
Thus, a single causal regions may comprise several regions of the TTFS arrangement. 
The situation is illustrated in Figure~\ref{fig:causal_regions_d2_d3}, where 
a causal region comprises several regions of the TTFS arrangement. 
Nonetheless, the refinement provided by $\cA_{\rm TTFS}$ gives a simple way to upper bound the number of causal regions. 
\end{remark}

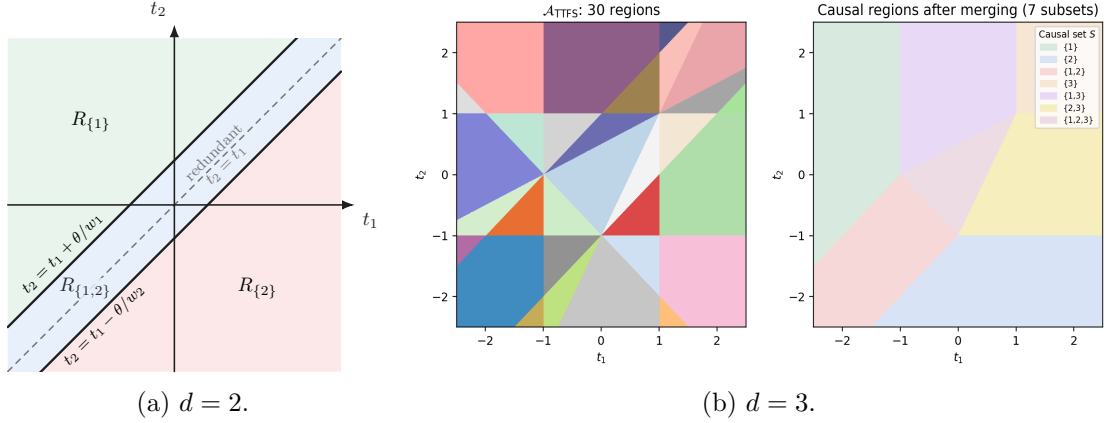
\begin{figure}[t]
\centering 
\begin{subfigure}[t]{0.35\linewidth}
\centering
\vspace{0pt}
\resizebox{!}{5cm}{
\begin{tikzpicture}[>=Latex, line cap=round, line join=round, font=\small]
  \def\xmin{-3}
  \def\xmax{ 3}
  \def\ymin{-3}
  \def\ymax{ 3}

  \def\deltaOne{0.8}
  \def\deltaTwo{0.6}

  \definecolor{RegionTop}{RGB}{232,244,236}   
  \definecolor{RegionMid}{RGB}{232,240,252}   
  \definecolor{RegionBot}{RGB}{252,232,232}   
  \definecolor{LineMain}{RGB}{30,30,30}
  \definecolor{LineAux}{RGB}{120,120,120}

  \begin{scope}
    \path[clip] (\xmin,\ymin) rectangle (\xmax,\ymax);

    \fill[RegionTop]
      (\xmin,\ymax) -- (\xmax,\ymax)
      -- (\xmax,{\xmax+\deltaOne})
      -- (\xmin,{\xmin+\deltaOne}) -- cycle;

    \fill[RegionMid]
      (\xmin,{\xmin+\deltaOne})
      -- (\xmax,{\xmax+\deltaOne})
      -- (\xmax,{\xmax-\deltaTwo})
      -- (\xmin,{\xmin-\deltaTwo}) -- cycle;

    \fill[RegionBot]
      (\xmin,\ymin) -- (\xmax,\ymin)
      -- (\xmax,{\xmax-\deltaTwo})
      -- (\xmin,{\xmin-\deltaTwo}) -- cycle;

    \draw[LineMain, very thick]
      (\xmin,{\xmin+\deltaOne}) -- (\xmax,{\xmax+\deltaOne});

    \draw[LineMain, very thick]
      (\xmin,{\xmin-\deltaTwo}) -- (\xmax,{\xmax-\deltaTwo});

    \draw[LineAux, thick, dashed]
      (\xmin,{\xmin}) -- (\xmax,{\xmax});

    \node[align=center] at (-1.5,1.5) {\(\,R_{\{1\}}\,\)};
    \node[align=center, fill= RegionMid, fill opacity=.85, inner sep=0pt] at (-1.6,-1.5) {\(R_{\{1,2\}}\)};
    \node[align=center] at (1.5,-1.5) {\(\,R_{\{2\}}\,\)};

    \node[anchor=west, 
    text opacity=1, inner sep=2pt, rotate=45]
      at (-2.75,{\deltaOne-2.4})
      {\scriptsize \(t_2=t_1+\theta/w_1\)};

    \node[anchor=west, 
    text opacity=1, inner sep=2pt, rotate=45]
      at (-2,{-2.3-\deltaTwo})
      {\scriptsize \(t_2=t_1-\theta/w_2\)};

    \node[
      anchor=west,
      text opacity=1,
      color=LineAux, 
      inner sep=0pt, rotate=45
    ] at (0.35,.35) {\scriptsize
      $\begin{matrix}
      \text{redundant}\\[-1pt]
      t_2=t_1
      \end{matrix}$
    };
  \end{scope}

  \draw[LineMain, thick, ->] (\xmin,0) -- (\xmax+0.25,0) node[below right] {$t_1$};
  \draw[LineMain, thick, ->] (0,\ymin) -- (0,\ymax+0.25) node[above left] {$t_2$};

\end{tikzpicture}%
}
\caption{$d=2$.}
\label{fig:d2_causal_regions}
\end{subfigure}
\begin{subfigure}[t]{0.6\linewidth}
\centering
\vspace{0pt}
\includegraphics[height=5cm, width=\linewidth]{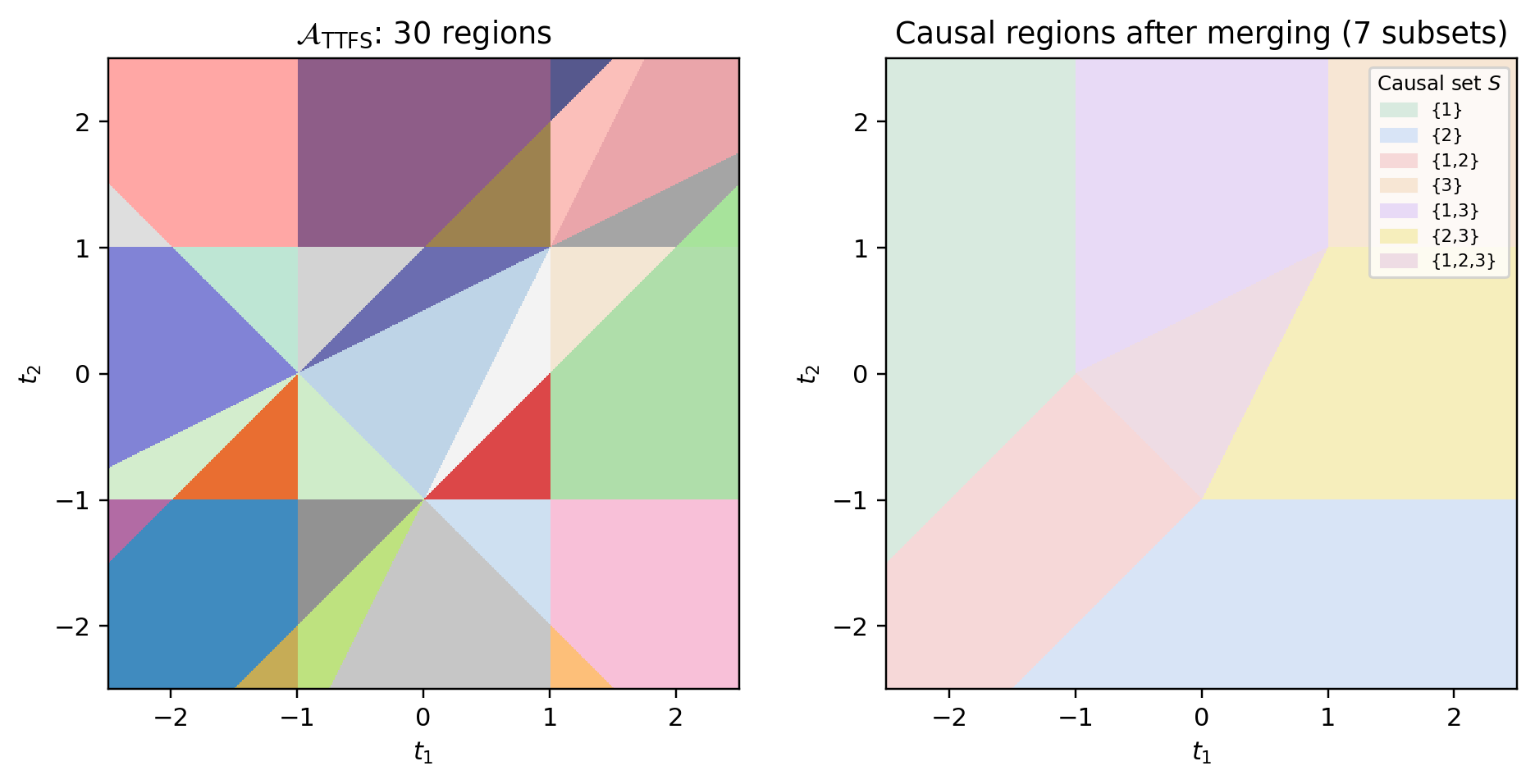}
\caption{$d=3$.}
\label{fig:d3_causal_regions}
\end{subfigure}

\caption{
Causal partitions for a single neuron. 
(a) For $d=2$, the two boundary lines $t_2 = t_1 + \theta/w_1$ and $t_2 = t_1 - \theta/w_2$ partition the plane into the three causal regions $R_{\{1\}}, R_{\{1,2\}}, R_{\{2\}}$; the dashed line $t_2=t_1$ is redundant for the causal partition. 
(b) For $d=3$, the TTFS arrangement $\cA_{\rm TTFS}$ consists of $9$ distinct hyperplanes, visualized through their intersections with the plane $t_3=0$. The arrangement regions (left) refine the causal regions (right). 
}
\label{fig:causal_regions_d2_d3}
\end{figure}

For a layer with multiple neurons, and more generally for a deep network, causal regions can still be described as unions of regions of a large hyperplane arrangement, as in \eqref{eq:R_S_union}. 
However, the resulting hyperplane decomposition is typically a strict refinement of the causal region partition, making it cumbersome to work with directly. 
We therefore introduce the notion of a \emph{causal pattern} of a network. A causal pattern records, for each neuron, the set of presynaptic neurons that fire before it.

We observe that all hyperplanes in $\cA_{\rm TTFS}$ are translation invariant along the all-ones direction $\1 = (1,\dots, 1) \in \R^d$. 
Let $e_1,\ldots, e_d$ denote the standard basis vectors of $\R^d$. 
The normal vector $n_{S,j}$ of $H_{S,j}$ satisfies 
\begin{equation}\label{eq:normal_HSj}
n_{S,j} = \sum_{i\in S} w_i e_i - W_S e_j, \qquad \langle n_{S,j},\1 \rangle = \sum_{i\in S}w_i - W_S = 0.
\end{equation} 
Thus, every hyperplane in $\cA_{\rm TTFS}$ is parallel to $\1$, and the arrangement is invariant under translations $\vec t \mapsto \vec t + c\1$. 

Accordingly, the arrangement can be essentialized by quotienting out the direction $\1$ \cite{stanley_2011_enumcombinatorics}. 
Identifying the quotient with 
\begin{equation}
\label{eq:T}
    T \coloneqq \1^\perp = \{\vec{u} \in \R^d : \langle \vec{u}, \1 \rangle = 0 \} \cong \R^{d-1}
\end{equation}
we define the essentialized arrangement $\cA_{\rm TTFS}^{\rm ess}\coloneq \{H^{\rm ess} = H \cap T\colon H \in \cA_{\rm TTFS} \}$. 
The arrangements $\cA_{\rm TTFS}^{\rm ess}$ and $\cA_{\rm TTFS}$ have the same region combinatorics. 

This invariance has a direct interpretation in terms of spike times. Adding the same constant $c$ to all input spike times shifts the output firing time by $c$ but leaves all causal relations unchanged. Consequently, the causal geometry depends only on relative spike times and can naturally be viewed on the quotient $\R^{d} / \operatorname{span}\{1\}\cong \mathbb{R}^{d-1}$. 
This is precisely the geometry captured by the essentialized arrangement $\cA_{\rm TTFS}^{\rm ess}$. This is illustrated in Figure~\ref{fig:levelsetd2}.

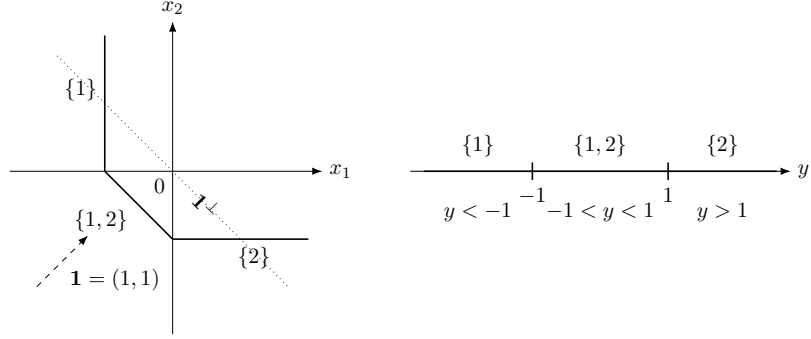
\begin{figure}[t]
\centering
\resizebox{!}{4.5cm}{
\begin{tikzpicture}[scale=1.2, >=Latex]

\begin{scope}

\draw[->] (-2.4,0) -- (2.2,0) node[right] {\(x_1\)};
\draw[->] (0,-2.4) -- (0,2.2) node[above] {\(x_2\)};

\draw[thick] (-1,0) -- (-1,2.0);
\draw[thick] (-1,0) -- (0,-1);
\draw[thick] (0,-1) -- (2.0,-1);

\node[left] at (-1,1.2) {\(\{1\}\)};
\node[below left] at (-0.55,-0.45) {\(\{1,2\}\)};
\node[below] at (1.2,-1) {\(\{2\}\)};

\node[below left] at (0,0) {\(0\)};

\draw[->, dashed] (-2.0,-1.7) -- (-1.25,-0.95)
node[midway, below right] {\(\mathbf 1=(1,1)\)};

\draw[dotted] (-1.7,1.7) -- (1.7,-1.7);
\node[rotate=-45] at (0.5, -0.5) {$\mathbf 1^\perp$};
\end{scope}

\begin{scope}[xshift=6.3cm]

\draw[->] (-2.8,0) -- (2.8,0) node[right] {\(y\)};

\draw[thick] (-1,0.13) -- (-1,-0.13);
\draw[thick] (1,0.13) -- (1,-0.13);
\node[below] at (-1,-0.13) {\(-1\)};
\node[below] at (1,-0.13) {\(1\)};

\draw[thick] (-2.6,0) -- (-1,0);
\draw[thick] (-1,0) -- (1,0);
\draw[thick] (1,0) -- (2.6,0);

\node[above] at (-1.8,0.1) {$\{1\}$};
\node[above] at (0,0.1) {$\{1,2\}$};
\node[above] at (1.8,0.1) {$\{2\}$};

\node[below] at (-1.8,-0.35) {$y<-1$};
\node[below] at (0,-0.35) {$-1<y<1$};
\node[below] at (1.8,-0.35) {$y>1$};
\end{scope}

\end{tikzpicture}
}
\caption{For $d=2$, $w_1=2_2=1$, $\theta_v= 1$, the level set $\sum_{i=1}^2\sigma(-x_i) = 1$, where $x_i = t_i - t$, consists of three pieces. 
Projecting along the all-ones direction $\1 = (1,1)$ onto the quotient coordinate $y = x_1-x_2 = t_1-t_2$ maps these pieces to the three causal regions in the quotient input space.}
\label{fig:levelsetd2}
\end{figure}

\subsection{The lifted TTFS configuration and polytope} 
\label{subsec:ttfs_polytope}

We present a polyhedral geometry description of the firing-time map $\vec t \mapsto t_v(\vec t)$. 
Our goal is to provide a compact representation of the piecewise-affine structure of firing-time map and connect the TTFS neuron model to familiar constructions from polyhedral and tropical geometry. 
We present the relevant polyhedral notions at a descriptive level and refer the reader to \cite{MaclaganSturmfels2015Tropical} for a systematic treatment of regular subdivisions, duality, and related constructions. 

Recall from Proposition~\ref{cor:tv_equals_tvs} that 
$t_v(\vec t) = \min_{\emptyset \neq S \subseteq[d]} t_v^S(\vec t)$, 
where 
\begin{equation}
\label{eq:aSbS}
    t_v^S(\vec t) = \langle \vec a_S, \vec t \rangle + b_S,
    \quad \text{where  } \vec a_S \coloneqq \sum_{i \in S}\frac{w_i}{W_S} \vec e_i, \quad b_S \coloneqq \frac{\theta_v}{W_S}, \quad 
    W_S \coloneqq\sum_{i\in S}w_i . 
\end{equation}
It is convenient to work with the convex piecewise-affine function 
\begin{equation} 
f(\vec t) \coloneqq -t_v(\vec t) = \max_{\emptyset\neq S\subseteq[d]} \bigl(-\langle \vec a_S, \vec t \rangle - b_S \bigr) .
\label{eq:minus_tv}
\end{equation}

\begin{definition}[Lifted TTFS configuration and polytope] 
\label{def:lifted_ttfs_polytope} 
For every nonempty $S\subseteq[d]$, define the projected coefficient point and lifted coefficient point \[ \vec p_S\coloneqq-\vec a_S\in\R^d, \qquad \vec y_S\coloneqq(-\vec a_S,-b_S) =(\vec p_S,-b_S)\in\R^{d+1}. \] The collection \[ \mathcal Y(w,\theta_v) \coloneqq \{\vec y_S:\emptyset\neq S\subseteq[d]\} \] is the \emph{lifted TTFS configuration}, and \[ Q(w,\theta_v) \coloneqq \operatorname{conv} \{\vec y_S:\emptyset\neq S\subseteq[d]\} \subseteq\R^{d+1} \] is the \emph{lifted TTFS polytope}. \end{definition}

The connection between $Q$ and $f$ can be seen directly through the support function of $Q$, defined as $h_Q(\vec z) = \max_{\vec y\in Q}\langle \vec z, \vec y\rangle$. 
For $\vec z = (\vec t, 1)$, we obtain 
$$
h_Q\bigl((\vec t,1)\bigr) = \max_{\emptyset\neq S\subseteq[d]}
\bigl(
-\langle\vec a_S,\vec t\rangle-b_S
\bigr)
=
f(\vec t).
$$
Thus, $f$ is obtained by restricting the support function of $Q$ to points whose last coordinate is $1$. 
This representation gives a direct correspondence between the polyhedral geometry of $Q$ and the piecewise-affine structure of $f$. 
For a given input $\vec t$, the affine pieces active at $\vec t$ are precisely those whose lifted coefficient vectors $\vec y_S$ lie on the face of $Q$ exposed by the direction $(\vec t,1)$. In particular, if the active affine piece is unique, the corresponding $\vec y_S$ is an exposed vertex of $Q$. 
Consequently, the upper faces of $Q$, namely those exposed by directions with last coordinate equal to $1$, encode the piecewise-affine structure of $f$.

\paragraph{Structure of the polytope} 

The polytope $Q$ has a highly constrained structure. 
First observe that the coefficient vectors $\vec a_S$ satisfy 
$(\vec a_S)_i \geq 0$, $\sum_{i=1}^d(\vec a_S)_i = 1$. 
Thus, every $\vec a_S$ lies in the standard simplex $\Delta_{d-1} = \operatorname{conv}\{\vec e_1, \dots, \vec e_d\}$. 
More precisely, $\vec a_S$ lies in the relative interior of the face $\operatorname{conv}\{\vec e_i\colon i\in S\}$. 
Since $\vec a_{\{i\}}=\vec e_i$, the convex hull of all points $\vec p_S = -\vec a_S$ is exactly the simplex: 
$$
\operatorname{conv}\bigl\{\vec p_S \colon \emptyset \neq S\subseteq[d] \bigr\} = -\Delta_{d-1}.
$$
Moreover, all lifted points lie in the affine hyperplane 
\begin{equation} 
\label{eq:H_ttfs} 
H \coloneqq \left\{ (\vec p,\beta)\in\R^d\times\R: \langle\vec p,\1\rangle=-1 \right\}. \end{equation} 
Thus $Q$ has affine dimension at most $d$, although it is naturally embedded in $\R^{d+1}$. 
This constraint reflects the translation invariance of the firing-time map, whereby adding the same constant to all input spike times shifts the firing time by the same constant without changing the causal structure.

The structure is considerably more rigid than the elementary observations above. 
To discuss this, we introduce the scale-free parameters 
\begin{equation} 
\label{eq:lambda_ttfs} 
\lambda_i \coloneqq \frac{w_i}{\theta_v}>0, \qquad \lambda_S \coloneqq \sum_{i\in S}\lambda_i. \end{equation} 
Then 
\begin{equation} 
\label{eq:yS_lambda} 
\vec a_S = \frac{1}{\lambda_S} \sum_{i\in S}\lambda_i\vec e_i, 
\qquad 
b_S=\frac{1}{\lambda_S},
\end{equation} 
and therefore 
\begin{equation} \label{eq:yS_lambda_lifted} 
\vec y_S = \left( -\frac{\sum_{i\in S}\lambda_i\vec e_i}{\lambda_S}, -\frac{1}{\lambda_S} \right). 
\end{equation}

The following theorem gives an intrinsic geometric description of $Q$. 

\begin{theorem}[Structure of the lifted TTFS polytope] 
\label{thm:ttfs_polytope_structure} 
Let $Q=Q(w,\theta_v)$ be the lifted TTFS polytope of a neuron with $d$ positive input weights, and let $\lambda_i=w_i/\theta_v$. 
Define the weighted box 
$
B_\lambda \coloneqq \bigtimes_{i=1}^d[0,\lambda_i] \subseteq\R^d 
$
and, for $S\subseteq[d]$, its vertices 
$
\vec x_S \coloneqq \sum_{i\in S}\lambda_i\vec e_i$. 
Then: 
\begin{enumerate} 
\item 
The lifted TTFS points are obtained from the nonzero vertices $\vec x_S$ by first embedding them into $\R^{d+1}$ via $\vec x\mapsto(\vec x,1)$ and then centrally projecting from the origin onto 
\[
H=\{(\vec p,\beta)\in\R^{d+1}:\langle\vec p,\1\rangle=-1\}.
\]
Explicitly,
\[
\vec y_S
=
-\frac{(\vec x_S,1)}{\langle\vec x_S,\1\rangle}.
\]

\item
The upper hull of $Q$ is the graph over $-\Delta_{d-1}$ of the concave
piecewise-linear function
\[
g_\lambda(\vec p)
\coloneqq
\min_{i\in[d]}\frac{p_i}{\lambda_i}.
\]
Equivalently, its hypograph is the polyhedron
\begin{equation}
\label{eq:Qhat}
\widehat Q_\lambda
\coloneqq
\left\{
(\vec p,\beta)\in H:
\lambda_i\beta\leq p_i\leq0,\; i\in[d]
\right\} 
=
Q+\R_{\geq0}(0,\ldots,0,-1).
\end{equation} 

\item 
The vertices of $Q$ are naturally indexed by the $2^d-1$ nonempty subsets $S\subseteq[d]$. The vertex $\vec y_S = (\vec p,\beta)$ indexed by $S$ is obtained by making exactly one of the two inequalities $\lambda_i\beta\leq p_i\leq0$ tight for every $i\in[d]$, with $p_i=\lambda_i\beta$ iff $i\in S$ and $p_i=0$ iff $i\notin S$. 

\end{enumerate} 
\end{theorem}

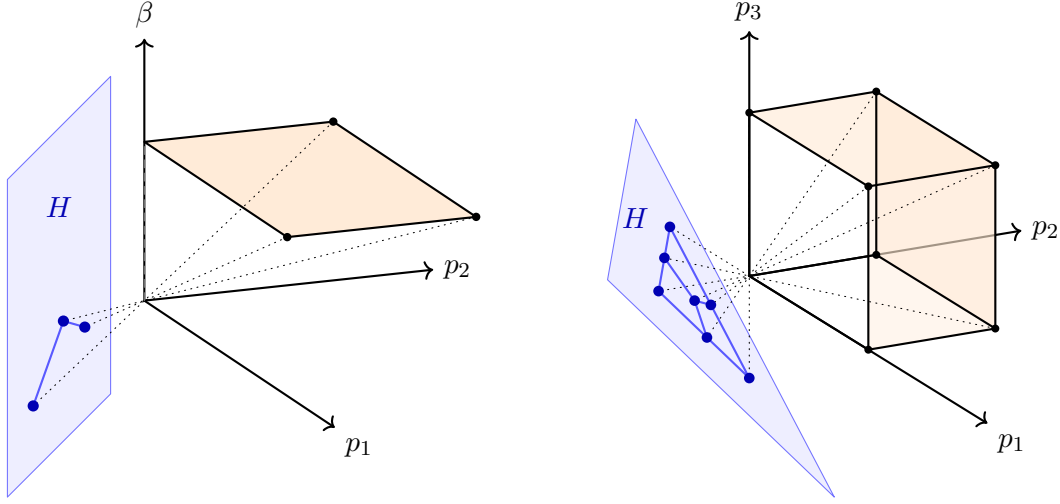
\begin{figure}
\centering 
\begin{tikzpicture}[
    x={(.75cm,-.5cm)},
    y={(1.4cm,0.15cm)},
    z={(0cm,2cm)},
    line cap=round,
    line join=round,
    scale=1.05
]

\def\la{2.4}   
\def\lb{1.7}   

\pgfmathsetmacro{\ba}{-1/\la}
\pgfmathsetmacro{\bb}{-1/\lb}
\pgfmathsetmacro{\bc}{-1/(\la+\lb)}
\pgfmathsetmacro{\pa}{-\la/(\la+\lb)}
\pgfmathsetmacro{\pb}{-\lb/(\la+\lb)}

\coordinate (O)   at (0,0,0);

\coordinate (X0)  at (0,0,1);
\coordinate (X1)  at (\la,0,1);
\coordinate (X2)  at (0,\lb,1);
\coordinate (X12) at (\la,\lb,1);

\coordinate (Y1)  at (-1,0,\ba);
\coordinate (Y2)  at (0,-1,\bb);
\coordinate (Y12) at (\pa,\pb,\bc);

\draw[->, thick] (O) -- (3.2,0,0)
    node[below right] {$p_1$};

\draw[->, thick] (O) -- (0,2.6,0)
    node[right] {$p_2$};

\draw[->, thick] (O) -- (0,0,1.65)
    node[above] {$\beta$};

\coordinate (H1) at (-1.5, 0.5,-1.0);
\coordinate (H2) at ( 0.5,-1.5,-1.0);
\coordinate (H3) at ( 0.5,-1.5, 1);
\coordinate (H4) at (-1.5, 0.5, 1);

\fill[blue!10, opacity=.65]
(H1) -- (H2) -- (H3) -- (H4) -- cycle;
\draw[blue!55, thin]
(H1) -- (H2) -- (H3) -- (H4) -- cycle;

\node[blue!70!black]
    at (-.5,-.5,.5)
    {$H$};

\fill[orange!18, opacity=.85]
    (X0) -- (X1) -- (X12) -- (X2) -- cycle;

\draw[thick]
    (X0) -- (X1) -- (X12) -- (X2) -- cycle;

\draw[densely dashed, gray]
    (O) -- (X0);

\fill (X1)  circle (1.5pt);
\fill (X2)  circle (1.5pt);
\fill (X12) circle (1.5pt);

\draw[dotted]
    (X1) -- (O) -- (Y1);

\draw[dotted]
    (X2) -- (O) -- (Y2);

\draw[dotted]
    (X12) -- (O) -- (Y12);

\draw[blue!65, thick]
    (Y1) -- (Y12) -- (Y2);

\fill[blue!70!black] (Y1) circle (2.0pt);
\fill[blue!70!black] (Y2) circle (2.0pt);
\fill[blue!70!black] (Y12) circle (2.0pt);

\end{tikzpicture}
\qquad \qquad
	\begin{tikzpicture}[
	x={(1.05cm,-0.65cm)},
	y={(1.20cm,0.20cm)},
	z={(0cm,1.35cm)},
	line cap=round,
	line join=round,
	scale=1.0
	]
	
	\def\la{1.5}
	\def\lb{1.4}
	\def\lc{1.6}

	\pgfmathsetmacro{\pabone}{-\la/(\la+\lb)}
	\pgfmathsetmacro{\pabtwo}{-\lb/(\la+\lb)}
	
	\pgfmathsetmacro{\pacone}{-\la/(\la+\lc)}
	\pgfmathsetmacro{\pacthree}{-\lc/(\la+\lc)}
	
	\pgfmathsetmacro{\pbctwo}{-\lb/(\lb+\lc)}
	\pgfmathsetmacro{\pbcthree}{-\lc/(\lb+\lc)}
	
	\pgfmathsetmacro{\pabcone}{-\la/(\la+\lb+\lc)}
	\pgfmathsetmacro{\pabctwo}{-\lb/(\la+\lb+\lc)}
	\pgfmathsetmacro{\pabcthree}{-\lc/(\la+\lb+\lc)}
	
	\coordinate (O) at (0,0,0);

	\coordinate (X0)   at (0,0,0);
	\coordinate (X1)   at (\la,0,0);
	\coordinate (X2)   at (0,\lb,0);
	\coordinate (X3)   at (0,0,\lc);
	\coordinate (X12)  at (\la,\lb,0);
	\coordinate (X13)  at (\la,0,\lc);
	\coordinate (X23)  at (0,\lb,\lc);
	\coordinate (X123) at (\la,\lb,\lc);
	
	\coordinate (P1)   at (-1,0,0);
	\coordinate (P2)   at (0,-1,0);
	\coordinate (P3)   at (0,0,-1);
	
	\coordinate (P12)  at (\pabone,\pabtwo,0);
	\coordinate (P13)  at (\pacone,0,\pacthree);
	\coordinate (P23)  at (0,\pbctwo,\pbcthree);
	
	\coordinate (P123) at (\pabcone,\pabctwo,\pabcthree);
	
	\draw[->, thick] (O) -- (3,0,0)
	node[below right] {$p_1$};
	
	\draw[->, thick] (O) -- (0,3,0)
	node[right] {$p_2$};
	
	\draw[->, thick] (O) -- (0,0,2.4)
	node[above] {$p_3$};

	\coordinate (H1) at (-2, 0.5, 0.5);
	\coordinate (H2) at ( 0.5,-2, 0.5);
	\coordinate (H3) at ( 0.5, 0.5,-2);
	
	\fill[blue!10, opacity=.65]
	(H1) -- (H2) -- (H3) -- cycle;
	\draw[blue!55, thin]
	(H1) -- (H2) -- (H3) -- cycle;
	
	\node[blue!70!black]
	at (-0.67,-0.67,.34)
	{$H$};

\fill[orange!18, opacity=.55]
    (X3) -- (X13) -- (X123) -- (X23) -- cycle;

\fill[orange!18, opacity=.45]
    (X2) -- (X12) -- (X123) -- (X23) -- cycle;

\fill[orange!18, opacity=.35]
    (X1) -- (X12) -- (X123) -- (X13) -- cycle;    
	
	\draw[thick]
	(X0) -- (X1) -- (X12) -- (X2) -- cycle;
	
	\draw[thick]
	(X0) -- (X3) -- (X13) -- (X1);
	
	\draw[thick]
	(X2) -- (X23) -- (X3);
	
	\draw[thick]
	(X12) -- (X123) -- (X13);
	
	\draw[thick]
	(X23) -- (X123);
	
	\foreach \X in {X1,X2,X3,X12,X13,X23,X123}
	\fill (\X) circle (1.5pt);
	
	\draw[dotted] (X1)   -- (O) -- (P1);
	\draw[dotted] (X2)   -- (O) -- (P2);
	\draw[dotted] (X3)   -- (O) -- (P3);
	
	\draw[dotted] (X12)  -- (O) -- (P12);
	\draw[dotted] (X13)  -- (O) -- (P13);
	\draw[dotted] (X23)  -- (O) -- (P23);
	
	\draw[dotted] (X123) -- (O) -- (P123);
	
	\draw[blue!65, thick]
	(P1) -- (P2) -- (P3) -- cycle;
	
	\draw[thick, blue!65]
	(P12) -- (P123)
	(P13) -- (P123)
	(P23) -- (P123);

	\foreach \P in {P1,P2,P3,P12,P13,P23,P123}
\fill[blue!70!black] (\P) circle (2pt);
	
	\end{tikzpicture}
\caption{Illustration of Theorem~\ref{thm:ttfs_polytope_structure}. The left panel shows $d=2$. 
The lifted box $B_\lambda$ (orange) in $\mathbb{R}^{d+1}$. 
The nonzero vertices are mapped by central projection from origin onto the affine hyperplane $H$ (blue) to produce the TTFS configuration. 
The right panel shows $d=3$. Shown is the projection onto the first $d$ coordinates, illustrating the projected configuration without the height. 
}
\label{fig:projection-polytope-structure}

\end{figure}

Theorem~\ref{thm:ttfs_polytope_structure} shows that the exponentially many vertices of $Q$ arise from a simple construction: 
they are a perspective image of the nonzero vertices of an axis-aligned box. This is illustrated in Figure~\ref{fig:projection-polytope-structure}. 
In particular, the $2^d-1$ lifted points are governed by only $d$ positive parameters.

\begin{proof}[Proof of Theorem~\ref{thm:ttfs_polytope_structure}] 
For every nonempty $S\subseteq[d]$,
\[
\langle\vec x_S,\1\rangle
=
\sum_{i\in S}\lambda_i
=
\lambda_S. 
\]
After the embedding $\vec x_S\mapsto(\vec x_S,1)$, the line through the origin and $(\vec x_S,1)$ consists of the points $c(\vec x_S,1)$, $c\in\R$. 
Its intersection with $H$ is determined by $
c\langle\vec x_S,\1\rangle=-1$, 
and hence by
$
c=-\frac{1}{\lambda_S}$. 
The resulting point is
\[
-\frac{(\vec x_S,1)}{\langle\vec x_S,\1\rangle}
=
\left(
-\frac{\sum_{i\in S}\lambda_i\vec e_i}{\lambda_S},
-\frac{1}{\lambda_S}
\right)
=
\vec y_S,
\]
where the last equality follows from~\eqref{eq:yS_lambda_lifted}. This proves
the first claim.

We next consider $\widehat Q_\lambda$. Since
$(\vec p,\beta)\in H$ means $\sum_i p_i=-1$, the constraints
$p_i\leq0$, $i\in[d]$, are equivalent to
$\vec p\in-\Delta_{d-1}$. 
Moreover, since $\lambda_i>0$, 
the condition $\lambda_i\beta\leq p_i$ is equivalent to $\beta\leq\frac{p_i}{\lambda_i}$. 
It follows that 
\[
(\vec p,\beta)\in\widehat Q_\lambda
\quad\text{if and only if}\quad
\vec p\in-\Delta_{d-1}
\quad\text{and}\quad
\beta\leq
\min_{i\in[d]}\frac{p_i}{\lambda_i}
=
g_\lambda(\vec p).
\]
Thus $\widehat Q_\lambda$ is precisely the hypograph of $g_\lambda$ over
$-\Delta_{d-1}$.

We now determine the vertices. 
Let $(\vec p,\beta)$ be a finite vertex of $\widehat Q_\lambda$. 
For each $i\in[d]$, at least one of the two inequalities $
\lambda_i\beta\leq p_i\leq0$ must be tight. 
Indeed, if for some $i$ both inequalities were strict, then
the active inequalities could involve at most the remaining $d-1$ coordinates. 
Their normals therefore could not span the $d$-dimensional
tangent space of $H$, contradicting that $(\vec p,\beta)$ is a vertex.
The two inequalities cannot both be tight for any $i$. Otherwise
$p_i=0=\lambda_i\beta$, so $\beta=0$. 
The inequalities
$\lambda_j\beta\leq p_j\leq0$ would then imply $p_j=0$ for every $j$,
contradicting $\sum_jp_j=-1$. Hence exactly one of the two inequalities is
tight for each $i$. 
Define 
$S \coloneqq \{i\in[d]:p_i=\lambda_i\beta\}$. 
Then
\[
p_i
=
\begin{cases}
\lambda_i\beta,&i\in S,\\
0,&i\notin S.
\end{cases}
\]
The set $S$ is nonempty, since otherwise $\vec p=0$. Using
$\sum_i p_i=-1$ gives
\[
-1
=
\sum_{i=1}^d p_i
=
\beta\sum_{i\in S}\lambda_i
=
\beta\lambda_S, 
\]
and therefore
\[
\beta=-\frac{1}{\lambda_S},
\qquad
\vec p
=
-\frac{1}{\lambda_S}
\sum_{i\in S}\lambda_i\vec e_i.
\]
Thus
\[
(\vec p,\beta)=\vec y_S.
\]

Conversely, for every nonempty $S\subseteq[d]$, the point $\vec y_S$
satisfies
\[
p_i=\lambda_i\beta,\quad i\in S,
\qquad
p_i=0,\quad i\notin S.
\]
These $d$ equalities, together with the affine equation
$\sum_i p_i=-1$, uniquely determine $(\vec p,\beta)$; hence $\vec y_S$ is
a vertex of $\widehat Q_\lambda$. Therefore the finite vertices of
$\widehat Q_\lambda$ are precisely the $2^d-1$ points $\vec y_S$, $\emptyset\neq S\subseteq[d]$. 

It remains to determine the recession cone of $\widehat Q_\lambda$. 
A vector $(\vec r,\rho)\in\R^{d+1}$ is a recession direction if
\[
(\vec p,\beta)+t(\vec r,\rho)\in\widehat Q_\lambda
\qquad\text{for all }t\geq0
\qquad 
\text{whenever $(\vec p,\beta)\in\widehat Q_\lambda$.} 
\]
Since $\widehat Q_\lambda\subset H$, this requires $\sum_{i=1}^d r_i=0$. 
Moreover, preserving the inequalities
$\lambda_i\beta\leq p_i\leq0$ for all $t\geq0$ requires
\[
\lambda_i\rho\leq r_i\leq0,
\qquad i\in[d].
\]
The conditions $r_i\leq0$ for every $i$ and $\sum_i r_i=0$, imply
$r_i=0$ for all $i$, and thus $\rho\leq0$. 
Hence
\[
\operatorname{rec}(\widehat Q_\lambda)
=
\R_{\geq0}(0,\ldots,0,-1).
\]
Since the finite vertices of $\widehat Q_\lambda$ are precisely the points
$\vec y_S$, the Minkowski-Weyl theorem gives
\[
\widehat Q_\lambda
=
\operatorname{conv}
\{\vec y_S:\emptyset\neq S\subseteq[d]\}
+
\R_{\geq0}(0,\ldots,0,-1)
=
Q+\R_{\geq0}(0,\ldots,0,-1).
\]
Since the recession cone is vertically downward, the upper boundary of
$\widehat Q_\lambda$ is the upper hull of $Q$. Since $\widehat Q_\lambda$ is
the hypograph of $g_\lambda$, this upper boundary is precisely the graph of
$g_\lambda$. 
\end{proof}

\paragraph{Constraints on the lifted TTFS configuration} 

The next result makes the rigidity of the TTFS configuration explicit by characterizing the algebraic relations among the lifted points.

\begin{theorem}[Constraints among the lifted TTFS points] \label{thm:ttfs_polytope_constraints} 
Let 
\[ 
\mathcal Y = \{\vec y_S=(\vec p_S,\beta_S): \emptyset\neq S\subseteq[d]\} \subset H 
\] 
be a collection of points with $\beta_S<0$. Then $\mathcal Y$ is the lifted TTFS configuration of a neuron with positive weights and positive threshold if and only if the following relations hold: 
\begin{align} \vec p_{\{i\}} &=-\vec e_i, &&i\in[d], \label{eq:constraint_singletons}
\\ 
p_{S,i} &=0, &&i\notin S, \label{eq:constraint_support}
\\ 
p_{S,i}\beta_{\{i\}} &=-\beta_S, &&i\in S, \label{eq:constraint_coordinate_height}
\\ 
\frac1{\beta_S} &= \sum_{i\in S}\frac1{\beta_{\{i\}}}, &&\emptyset\neq S\subseteq[d]. \label{eq:constraint_reciprocal} 
\end{align} 
\end{theorem} 

From Theorem~\ref{thm:ttfs_polytope_constraints} we see that, setting 
$r_S\coloneqq-\frac1{\beta_S}>0$, $r_\emptyset\coloneqq0$, 
the reciprocal heights form a strictly positive modular set function: 
\begin{equation} \label{eq:r_modular} 
r_S = \sum_{i\in S}r_{\{i\}}, 
\end{equation} 
and therefore, 
\begin{equation} 
\label{eq:r_modular_lattice} 
r_S+r_T = r_{S\cup T}+r_{S\cap T} \qquad \text{for all }S,T\subseteq[d]. 
\end{equation} 
Moreover, setting 
$\vec q_S \coloneqq \frac{\vec y_S}{\beta_S}$ and adjoining 
$\vec q_\emptyset\coloneqq(0,\ldots,0,1)$, 
the normalized points satisfy the 
modular relations 
\begin{equation} 
\label{eq:q_boolean_relation} 
\vec q_S+\vec q_T = \vec q_{S\cup T}+\vec q_{S\cap T} \qquad \text{for all }S,T\subseteq[d]. 
\end{equation} 
Indeed, since $\vec q_S = \frac{\vec y_S}{\beta_S} = \left( \sum_{i\in S}\lambda_i\vec e_i, 1 \right)$, one has 
\begin{align*} 
\vec q_S+\vec q_T &= \left( \sum_i\lambda_i \bigl(\mathbf 1_{\{i\in S\}}+\mathbf 1_{\{i\in T\}}\bigr)\vec e_i, 2 \right)
\\ 
&= \left( \sum_i\lambda_i \bigl(\mathbf 1_{\{i\in S\cup T\}} +\mathbf 1_{\{i\in S\cap T\}}\bigr)\vec e_i, 2 \right) \\ 
&= \vec q_{S\cup T}+\vec q_{S\cap T}. \end{align*}

\begin{proof}[Proof of Theorem~\ref{thm:ttfs_polytope_constraints}] 
Suppose first that $\mathcal Y$ is the lifted TTFS configuration of a neuron. For a singleton $S=\{i\}$, \eqref{eq:aSbS} gives 
$\vec a_{\{i\}}=\vec e_i$ and hence 
$
\vec p_{\{i\}}=-\vec e_i$, 
which proves~\eqref{eq:constraint_singletons}. 
By~\eqref{eq:yS_lambda_lifted}, 
$\beta_S=-\frac1{\lambda_S}$ and 
$p_{S,i} = 
\begin{cases} -\dfrac{\lambda_i}{\lambda_S},& i\in S\\ 
0,& i\notin S\end{cases}$, which proves~\eqref{eq:constraint_support}. 
Moreover, 
$\beta_{\{i\}} = -\frac1{\lambda_i}$, 
and hence, for $i\in S$, 
\[ 
p_{S,i}\beta_{\{i\}} = \left(-\frac{\lambda_i}{\lambda_S}\right) \left(-\frac1{\lambda_i}\right) = \frac1{\lambda_S} = -\beta_S, 
\] 
which proves~\eqref{eq:constraint_coordinate_height}. 
Furthermore, 
$
\frac1{\beta_S} = -\lambda_S = -\sum_{i\in S}\lambda_i = \sum_{i\in S}\frac1{\beta_{\{i\}}}$, 
proving~\eqref{eq:constraint_reciprocal}.

Conversely, suppose that the stated constraints hold. Define $\lambda_i \coloneqq -\frac1{\beta_{\{i\}}}>0$. 
By~\eqref{eq:constraint_reciprocal}, $\frac1{\beta_S} = -\sum_{i\in S}\lambda_i$, and therefore 
\[ 
\beta_S = -\frac1{\sum_{i\in S}\lambda_i} = -\frac1{\lambda_S}. 
\] 
For $i\in S$, \eqref{eq:constraint_coordinate_height} gives 
$p_{S,i} = -\frac{\beta_S}{\beta_{\{i\}}} = -\frac{\lambda_i}{\lambda_S}$, whereas~\eqref{eq:constraint_support} gives $p_{S,i}=0$ for $i\notin S$. Hence 
\[ 
\vec y_S = \left( -\frac{\sum_{i\in S}\lambda_i\vec e_i}{\lambda_S}, -\frac1{\lambda_S} \right). 
\] 
Choosing, for example, $\theta_v=1$, $w_i=\lambda_i$, recovers exactly the given configuration from Definition~\ref{def:lifted_ttfs_polytope}. Thus $\mathcal Y$ is a lifted TTFS configuration. \end{proof}

\paragraph{Degrees of freedom} Theorem~\ref{thm:ttfs_polytope_constraints} makes clear that the $2^d-1$ lifted points are far from independent. 

A lifted TTFS configuration has only $d$ degrees of freedom. 
Indeed, the neuron is parametrized by 
$(w_1,\ldots,w_d,\theta_v)\in\R_{>0}^{d+1}$, but the simultaneous rescaling $(w_1,\ldots,w_d,\theta_v) \mapsto (cw_1,\ldots,cw_d,c\theta_v)$, $c>0$, leaves every $\vec a_S$ and $b_S$ unchanged. 
Thus the lifted configuration depends only on the $d$ ratios 
\[ 
\lambda_i=\frac{w_i}{\theta_v}, \qquad i\in[d]. 
\] 
Equivalently, it is completely determined by the $d$ singleton heights 
\[ 
\beta_{\{i\}} = -\frac{\theta_v}{w_i}. 
\] 
Once these are fixed, every nonsingleton point is forced: 
\begin{equation} 
\label{eq:yS_from_singletons} 
\beta_S = \left( \sum_{i\in S}\frac1{\beta_{\{i\}}} \right)^{-1}, \qquad p_{S,i} = \begin{cases} -\dfrac{\beta_S}{\beta_{\{i\}}},&i\in S,\\[2mm] 0,&i\notin S. \end{cases} 
\end{equation} 

By contrast, an arbitrary collection of $2^d-1$ points in the $d$-dimensional affine space $H$ has $d(2^d-1)$ degrees of freedom. 
The subset of such configurations subject to support structure compatibility has 
$d 2^{d-1}$ degrees of freedom. 
For a point indexed by $S$, the conditions $p_{S,i}=0$, $i\notin S$, and $\sum_{i\in S}p_{S,i}=-1$ leave $|S|-1$ degrees of freedom in $\vec p_S$, while the height $\beta_S$ contributes one additional degree of freedom. Thus an arbitrary support-compatible point $\vec y_S$ has $|S|$ degrees of freedom, and the sum over all nonempty subsets gives $ \sum_{\emptyset\neq S\subseteq[d]}|S| = d\,2^{d-1}$.

\subsection{Regular subdivision and causal complex} 

The lifting induces a regular subdivision of the coefficient simplex. 
Specifically, projecting the upper faces of $Q$ onto the first $d$ coordinates gives a subdivision of $-\Delta_{d-1}$. 
This subdivision is dual to the piecewise-affine decomposition of the domain of $f$: 
vertices of the subdivision correspond to full-dimensional linear regions of $f$, while edges correspond to boundaries between adjacent linear regions. 
Since the linear regions of $f=-t_v$ are the causal regions of the firing-time map, this duality provides a polyhedral representation of the causal region complex.

Figure~\ref{fig:ttfs_polytope_duality} illustrates this construction for a neuron with three inputs, $d=3$. 
The lifted coefficient polytope $Q$ shown in Figure~\ref{fig:ttfs_polytope_a} has one vertex $\vec y_S$ for each nonempty subset $S\subseteq [3]$. The vertices satisfy affine relations, so some facets of $Q$ are non-simplicial. 
Figure~\ref{fig:ttfs_polytope_b} shows the corresponding regular subdivision, obtained by projecting the upper faces of $Q$ onto the first $d$ coordinates. 
The projected points are $p_S = -\vec a_S$, whose convex hull is $-\Delta_{d-1}$. 
For $d=3$ with equal weights, the singleton subsets correspond to the three vertices of the triangle, the two-element subsets to the edge midpoints, and $\{1,2,3\}$ to the centroid. 
The locations of these points are determined by the weights, while their heights $-b_S$ determine the upper faces of $Q$ and hence the resulting regular subdivision. 
Figure~\ref{fig:ttfs_polytope_c} shows the dual causal-region complex for the same example. 
In particular, the internal $Y$-shaped structure in the regular subdivision is dual to the three boundaries of the central causal region $R_{\{1,2,3\}}$, while the subdivided outer edges are dual to the boundaries of the surrounding unbounded causal regions.

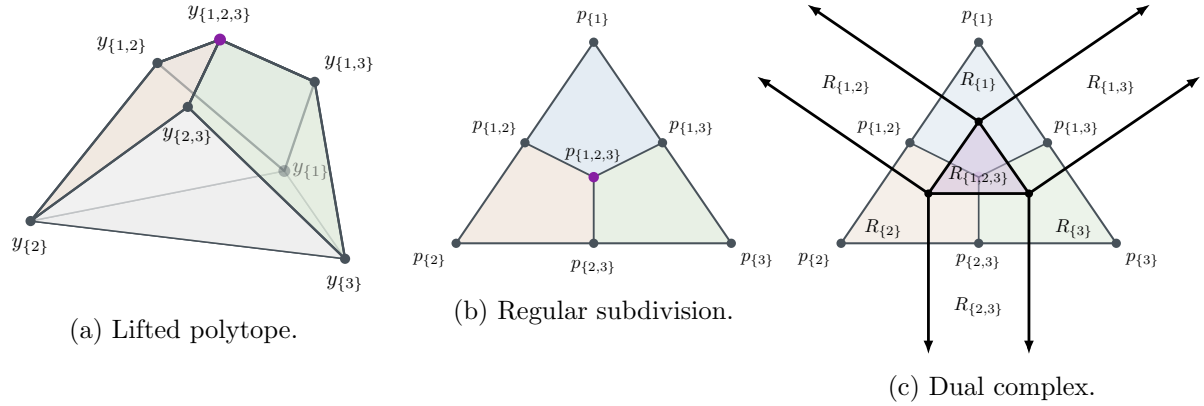
\begin{figure}[t]
\centering

\definecolor{TTFSTop}{RGB}{220,230,239}       
\definecolor{TTFSLeft}{RGB}{241,231,222}      
\definecolor{TTFSRight}{RGB}{227,237,223}     
\definecolor{TTFSCenter}{RGB}{210,190,222}    

\definecolor{TTFSEdge}{RGB}{72,82,90}
\definecolor{TTFSGuide}{RGB}{145,155,165}

\definecolor{TTFSRed}{RGB}{198,45,45}
\definecolor{TTFSBlue}{RGB}{30,104,198}
\definecolor{TTFSOrange}{RGB}{235,112,0}
\definecolor{TTFSPurple}{RGB}{135,31,163}

\begin{subfigure}[t]{0.32\textwidth}
\centering
\vspace{0pt}

\scalebox{.8}{
\begin{tikzpicture}[
    x={(1.00cm,-0.12cm)},
    y={(0.42cm,0.30cm)},
    z={(0cm,2.2cm)},
    line cap=round,
    line join=round,
    font=\small
]

\coordinate (Y1)   at ( 0.00, 3.80,-2);
\coordinate (Y2)   at (-2.60, 0.00,-2);
\coordinate (Y3)   at ( 2.60, 0.00,-2);

\coordinate (Y12)  at (-1.30,1.90,-1);
\coordinate (Y13)  at ( 1.30,1.90,-1);
\coordinate (Y23)  at ( 0.00,0.00,-1);

\coordinate (Y123) at (0.00,1.2667,-0.6666);

\filldraw[
    fill=gray!18,
    fill opacity=.75,
    draw=TTFSEdge,
    line width=.7pt
]
    (Y1)--(Y2)--(Y3)--cycle;

\filldraw[
    fill=gray!12,
    fill opacity=.72,
    draw=TTFSEdge,
    line width=.7pt
]
    (Y1)--(Y2)--(Y12)--cycle;

\filldraw[
    fill=gray!12,
    fill opacity=.72,
    draw=TTFSEdge,
    line width=.7pt
]
    (Y1)--(Y3)--(Y13)--cycle;

\filldraw[
    fill=gray!12,
    fill opacity=.72,
    draw=TTFSEdge,
    line width=.7pt
]
    (Y2)--(Y3)--(Y23)--cycle;

\filldraw[
    fill=TTFSTop,
    fill opacity=.78,
    draw=TTFSEdge,
    line width=1.0pt
]
    (Y1)--(Y12)--(Y123)--(Y13)--cycle;

\filldraw[
    fill=TTFSLeft,
    fill opacity=.82,
    draw=TTFSEdge,
    line width=1.0pt
]
    (Y2)--(Y12)--(Y123)--(Y23)--cycle;

\filldraw[
    fill=TTFSRight,
    fill opacity=.82,
    draw=TTFSEdge,
    line width=1.0pt
]
    (Y3)--(Y13)--(Y123)--(Y23)--cycle;

\draw[TTFSEdge,line width=1.1pt, opacity=.3]
    (Y1)--(Y12)--(Y123)--(Y13)--cycle;
    
\draw[TTFSEdge,line width=1.1pt]
    (Y12)--(Y123)--(Y13);    

\draw[TTFSEdge,line width=1.1pt]
    (Y2)--(Y12)--(Y123)--(Y23)--cycle;

\draw[TTFSEdge,line width=1.1pt]
    (Y3)--(Y13)--(Y123)--(Y23)--cycle;

\foreach \P in {Y2,Y3,Y12,Y13,Y23}
    \fill[TTFSEdge] (\P) circle (2.4pt);

\fill[TTFSEdge, opacity =.4] (Y1) circle (2.8pt);

\fill[TTFSPurple] (Y123) circle (2.8pt);

\node[right=0pt, opacity=.5] at (Y1)
    {$y_{\{1\}}$};

\node[below=4pt] at (Y2)
    {$y_{\{2\}}$};

\node[below=4pt] at (Y3)
    {$y_{\{3\}}$};

\node[above left=1pt] at (Y12)
    {$y_{\{1,2\}}$};

\node[above right=1pt] at (Y13)
    {$y_{\{1,3\}}$};

\node[below=6pt] at (Y23)
    {$y_{\{2,3\}}$};

\node[above=4pt] at (Y123)
    {$y_{\{1,2,3\}}$};

\end{tikzpicture}
}

\caption{Lifted polytope.}
\label{fig:ttfs_polytope_a}
\end{subfigure}
\hfill
\begin{subfigure}[t]{0.3\textwidth}
\centering
\vspace{0pt}
\scalebox{.7}{%
\hspace{-1cm}
\begin{tikzpicture}[
    line cap=round,
    line join=round,
    font=\small
]

\node[] at (0,4.5){}; 

\coordinate (P1)   at ( 0.00,3.80);
\coordinate (P2)   at (-2.60,0.00);
\coordinate (P3)   at ( 2.60,0.00);

\coordinate (P12)  at ($(P1)!0.50!(P2)$);
\coordinate (P13)  at ($(P1)!0.50!(P3)$);
\coordinate (P23)  at ($(P2)!0.50!(P3)$);

\coordinate (P123) at (0.00,1.25);


\fill[TTFSTop,opacity=.75]
    (P1)--(P12)--(P123)--(P13)--cycle;

\fill[TTFSLeft,opacity=.80]
    (P2)--(P12)--(P123)--(P23)--cycle;

\fill[TTFSRight,opacity=.80]
    (P3)--(P13)--(P123)--(P23)--cycle;

\draw[
    TTFSEdge,
    line width=1.1pt
]
    (P2)--(P1)--(P3)--cycle;

\draw[
    TTFSEdge,
    line width=1.0pt
]
    (P12)--(P123);

\draw[
    TTFSEdge,
    line width=1.0pt
]
    (P13)--(P123);

\draw[
    TTFSEdge,
    line width=1.0pt
]
    (P23)--(P123);

\foreach \P in {P1,P2,P3,P12,P13,P23}
    \fill[TTFSEdge] (\P) circle (2.5pt);

\fill[TTFSPurple] (P123) circle (2.8pt);

\node[above=5pt] at (P1)
    {$p_{\{1\}}$};

\node[below left=2pt] at (P2)
    {$p_{\{2\}}$};

\node[below right=2pt] at (P3)
    {$p_{\{3\}}$};

\node[above left=0pt] at (P12)
    {$p_{\{1,2\}}$};

\node[above right=0pt] at (P13)
    {$p_{\{1,3\}}$};

\node[below=5pt] at (P23)
    {$p_{\{2,3\}}$};

\node[above=4pt] at (P123)
    {$p_{\{1,2,3\}}$};

\end{tikzpicture}%
\hspace{-1cm}
}

\caption{Regular subdivision.}
\label{fig:ttfs_polytope_b}
\end{subfigure}
\hfill
\begin{subfigure}[t]{0.3\textwidth}
\centering
\vspace{0pt}

\scalebox{.7}{%
\hspace{-1.5cm}
\begin{tikzpicture}[
    >=Latex,
    line cap=round,
    line join=round,
    font=\small
]

\node[] at (0,4.5){}; 

\coordinate (P1)   at ( 0.00,3.80);
\coordinate (P2)   at (-2.60,0.00);
\coordinate (P3)   at ( 2.60,0.00);

\coordinate (P12)  at ($(P1)!0.50!(P2)$);
\coordinate (P13)  at ($(P1)!0.50!(P3)$);
\coordinate (P23)  at ($(P2)!0.50!(P3)$);

\coordinate (P123) at (0.00,1.25);

\fill[TTFSTop,opacity=.60]
    (P1)--(P12)--(P123)--(P13)--cycle;

\fill[TTFSLeft,opacity=.65]
    (P2)--(P12)--(P123)--(P23)--cycle;

\fill[TTFSRight,opacity=.65]
    (P3)--(P13)--(P123)--(P23)--cycle;

\draw[
    TTFSEdge,
    line width=1.1pt
]
    (P2)--(P1)--(P3)--cycle;

\draw[
    TTFSEdge,
    -,
    line width=.9pt
]
    (P12)--(P123);

\draw[
    TTFSEdge,
    -,
    line width=.9pt
]
    (P13)--(P123);

\draw[
    TTFSEdge,
    -,
    line width=.9pt
]
    (P23)--(P123);

\foreach \P in {P1,P2,P3,P12,P13,P23}
    \fill[TTFSEdge] (\P) circle (2.5pt);

\node[above=5pt] at (P1)
    {$p_{\{1\}}$};

\node[below left=2pt] at (P2)
    {$p_{\{2\}}$};

\node[below right=2pt] at (P3)
    {$p_{\{3\}}$};

\node[above left=0pt] at (P12)
    {$p_{\{1,2\}}$};

\node[above right=0pt] at (P13)
    {$p_{\{1,3\}}$};

\node[below=2pt] at (P23)
    {$p_{\{2,3\}}$};

\coordinate (q1) at ( 0.00,2.30);
\coordinate (q2) at (-0.95,0.94);
\coordinate (q3) at ( 0.95,0.94);

\filldraw[
    fill=TTFSCenter,
    fill opacity=.5,
    draw=black,
    line width=1.5pt
]
    (q1)--(q2)--(q3)--cycle;

\fill[
black] (q1) circle (2.2pt);
\fill[
black] (q2) circle (2.2pt);
\fill[
black] (q3) circle (2.2pt);

\fill[TTFSPurple, opacity = .1] (P123) circle (3.0pt);

\node[
    font=\scriptsize,
    opacity=.0, 
    anchor=south
] at ($(P123)+(0,.04)$)
    {$p_{\{1,2,3\}}$};

\node[
    font=\small,
    inner sep = 0pt
] at ($(P123)+(0,0.03)$)
    {$R_{\{1,2,3\}}$};

\draw[
    black, 
    line width=1.5pt,
    -{Latex[length=2.7mm,width=1.9mm]}
]
    (q1) -- ++(-3.25,2.22);

\draw[
    black, 
    line width=1.5pt,
    -{Latex[length=2.7mm,width=1.9mm]}
]
    (q2) -- ++(-3.25,2.22);

\draw[
    black, 
    line width=1.5pt,
    -{Latex[length=2.7mm,width=1.9mm]}
]
    (q1) -- ++(3.25,2.22);

\draw[
    black, 
    line width=1.5pt,
    -{Latex[length=2.7mm,width=1.9mm]}
]
    (q3) -- ++(3.25,2.22);

\draw[
    black, 
    line width=1.5pt,
    -{Latex[length=2.7mm,width=1.9mm]}
]
    (q2) -- ++(0,-3.05);

\draw[
    black, 
    line width=1.5pt,
    -{Latex[length=2.7mm,width=1.9mm]}
]
    (q3) -- ++(0,-3.05);

\node[
    font=\small,
    fill opacity=.75,
    text opacity=1,
    inner sep=1pt
] at (0,3.05)
    {$R_{\{1\}}$};

\node[
    font=\small,
    fill opacity=.75,
    text opacity=1,
    inner sep=1pt
] at (-1.8, 0.275)
    {$R_{\{2\}}$};

\node[
    font=\small,
    fill opacity=.75,
    text opacity=1,
    inner sep=1pt
] at (1.8,0.275)
    {$R_{\{3\}}$};

\node[
    font=\small,
    fill opacity=.72,
    text opacity=1,
    inner sep=1pt
] at (-2.5,3)
    {$R_{\{1,2\}}$};

\node[
    font=\small,
    fill opacity=.72,
    text opacity=1,
    inner sep=1pt
] at (2.5,3)
    {$R_{\{1,3\}}$};

\node[
    font=\small,
    fill opacity=.72,
    text opacity=1,
    inner sep=1pt
] at (0,-1.2)
    {$R_{\{2,3\}}$};
\end{tikzpicture}%
\hspace{-1cm}
}

\caption{Dual complex.}
\label{fig:ttfs_polytope_c}
\end{subfigure}
\caption{
Polyhedral description of a single TTFS neuron for
$d=3$ and $w_1=w_2=w_3=\theta_v=1$. 
(a) The lifted TTFS polytope $Q$, shown in the affine hyperplane $\{(\vec u, b)\colon \langle \vec u,\1\rangle=1 \}$. 
(b) Projecting its upper faces gives a regular subdivision of the TTFS polytope. 
(c) The causal region partition is the dual complex of the polytope subdivision. 
The three edges joining $p_{\{1,2,3\}}$ to $p_{\{1,2\}}, p_{\{1,3\}}, p_{\{2,3\}}$ correspond to the three sides of the region $R_{\{1,2,3\}}$, 
while the six outer boundary segments correspond to the remaining boundaries of the unbounded regions $R_{\{1\}}, R_{\{2\}}, R_{\{3\}}, R_{\{1,2\}}, R_{\{1,3\}}, R_{\{2,3\}}$. 
The resulting seven regions are $R_{S}$, $\emptyset\neq S\subseteq [3]$, recovering the causal partition shown previously in Figure~\ref{fig:d3_causal_regions}. 
}
\label{fig:ttfs_polytope_duality}
\end{figure}

This correspondence also gives a simple geometric interpretation of the causal boundaries. 
If two affine pieces indexed by $S$ and $T$ meet, their
common boundary is determined by the equality $t_v^S(\vec t) = t_v^T(\vec t)$, or equivalently, 
$\langle\vec a_S-\vec a_T,\vec t\rangle +(b_S-b_T)=0$. 
Thus, the difference $\vec a_S - \vec a_T$ determines the orientation of the boundary, while $b_S - b_T$ determines its position. 
Since the vectors $\vec a_S$ depend only on weight ratios, the weights determine the orientations of the causal boundaries, whereas the threshold enters through $b_S = \theta_v/W_S$ and controls their offsets. 
In particular, for fixed weights, scaling $\theta_v$ by a factor $c>0$ scales the causal partition by the same factor. 
Conversely, scaling all weights and $\theta_v$ by the same positive factor leaves every $\vec a_S$ and $b_S$ unchanged, and therefore leaves the firing-time map and its causal geometry unchanged.

\begin{theorem}[Structure of the regular subdivision and dual causal complex]
\label{thm:ttfs_dual_complex_structure}
Let $\Sigma_\lambda$ denote the regular subdivision of $-\Delta_{d-1}$
induced by the upper hull of $Q$. For
$\emptyset\neq I\subseteq J\subseteq[d]$, let $
C_{I,J}
\coloneqq
\operatorname{conv}
\{\vec p_S:I\subseteq S\subseteq J\}$. 
Then:
\begin{enumerate}

\item
The cells of $\Sigma_\lambda$ are precisely the polytopes
$C_{I,J}$, $\emptyset\neq I\subseteq J\subseteq[d]$. Moreover,
$\dim C_{I,J}=|J|-|I|$. 
In particular, the maximal cells are $C_{\{i\},[d]}$, $i\in[d]$, 
while the vertices are
$C_{S,S}=\{\vec p_S\}$, $\emptyset\neq S\subseteq[d]$. 

\item
Each $C_{I,J}$ is the central projection of the face
\[
F_{I,J}
\coloneqq
\left\{
\vec x\in B_\lambda:
x_i=\lambda_i\ \text{for }i\in I,\quad
x_i=0\ \text{for }i\notin J
\right\}
\]
of the weighted box $B_\lambda$. Consequently,
$\Sigma_\lambda$ is combinatorially isomorphic to the subcomplex of the
boundary of $B_\lambda$ consisting of the faces that do not contain the
origin. In particular, every cell $C_{I,J}$ is combinatorially a cube of
dimension $|J|-|I|$.

\item
The dual cell $C_{I,J}^*$ in the causal-region complex has codimension
$|J|-|I|$. Its relative interior consists precisely of the inputs for which,
writing $\tau=t_v(\vec t)$,
\[
t_i<\tau \quad (i\in I),\qquad
t_i=\tau \quad (i\in J\setminus I),\qquad
t_i>\tau \quad (i\notin J).
\]
Equivalently, the affine pieces indexed by
$I\subseteq S\subseteq J$ are exactly the pieces that are simultaneously active on
$\operatorname{relint}(C_{I,J}^*)$. 

\end{enumerate}
\end{theorem}

\begin{proof}[Proof of Theorem~\ref{thm:ttfs_dual_complex_structure}] 
For $\emptyset\neq I\subseteq J\subseteq[d]$, consider the face
\[
F_{I,J}
=
\left\{
\vec x\in B_\lambda:
x_i=\lambda_i \text{ for }i\in I,\quad
x_i=0 \text{ for }i\notin J
\right\}.
\]
Its free coordinates are precisely those indexed by $J\setminus I$, and hence
$\dim F_{I,J}=|J|-|I|$. 
Moreover, its vertices are $\vec x_S=\sum_{i\in S}\lambda_i\vec e_i$, $I\subseteq S\subseteq J$.
Since $I\neq\emptyset$, the face $F_{I,J}$ does not contain the origin.
Therefore the central projection
$\pi(\vec x)
=
-\frac{\vec x}{\langle\vec x,\1\rangle}
$
is well defined on $F_{I,J}$ and maps it projectively onto 
\[
C_{I,J}
=
\operatorname{conv}
\{\vec p_S:I\subseteq S\subseteq J\}.
\]
In particular, $\dim C_{I,J}=|J|-|I|$, 
and $C_{I,J}$ is combinatorially a cube of that dimension. 

It remains to show that these are precisely the cells of the regular
subdivision. By Theorem~\ref{thm:ttfs_polytope_structure}, the upper hull of
$Q$ is the graph over $-\Delta_{d-1}$ of
$g_\lambda(\vec p)
=
\min_{i\in[d]}\frac{p_i}{\lambda_i}$. 
For $\vec p\in-\Delta_{d-1}$, let
\[
I(\vec p)
=
\left\{
i:
\frac{p_i}{\lambda_i}=g_\lambda(\vec p)
\right\},
\qquad
J(\vec p)
=
\{i:p_i<0\}.
\]
Since $\sum_i p_i=-1$, both sets are nonempty and
$I(\vec p)\subseteq J(\vec p)$. The relative interior of the region on which
these two sets are fixed is characterized by
\[
\frac{p_i}{\lambda_i}=g_\lambda(\vec p)
\quad(i\in I),\qquad
g_\lambda(\vec p)<\frac{p_i}{\lambda_i}<0
\quad(i\in J\setminus I),\qquad
p_i=0
\quad(i\notin J).
\]
Its closure is exactly $C_{I,J}$. Hence the cells of the regular subdivision
are precisely the $C_{I,J}$.

We now identify the dual cells. Let
\[
\tau=t_v(\vec t),
\qquad
I=\{i:t_i<\tau\},
\qquad
J=\{i:t_i\leq\tau\}.
\]
At the firing time,
$\theta_v
=
\sum_{i\in I}w_i(\tau-t_i)$. 
If $I\subseteq S\subseteq J$, then every index in $S\setminus I$ satisfies
$t_i=\tau$, and therefore
$\theta_v
=
\sum_{i\in S}w_i(\tau-t_i)$. 
Rearranging gives
$
\tau
=
\frac{\theta_v+\sum_{i\in S}w_i t_i}{W_S}
=
t_v^S(\vec t)$. 
Thus every affine piece indexed by $S$ with
$I\subseteq S\subseteq J$ is active at $\vec t$.

Conversely, suppose $t_v^S(\vec t)=\tau$. Then 
$\theta_v
=
\sum_{i\in S}w_i(\tau-t_i)$. 
Comparing this with $\theta_v
=
\sum_{i:t_i<\tau}w_i(\tau-t_i)$ 
and using $w_i>0$, we see that every index with $t_i<\tau$ must belong to $S$, while no index with $t_i>\tau$ can belong to $S$. 
Hence $
I\subseteq S\subseteq J$. 
Therefore the active affine pieces are exactly those indexed by the Boolean
interval
\[
[I,J]
=
\{S:I\subseteq S\subseteq J\}.
\]

By polyhedral duality, the cell dual to $C_{I,J}$ therefore consists of the
inputs satisfying
\[
t_i<\tau \quad(i\in I),\qquad
t_i=\tau \quad(i\in J\setminus I),\qquad
t_i>\tau \quad(i\notin J),
\]
with $\tau=t_v(\vec t)$. Since $\dim C_{I,J}=|J|-|I|$, 
the dual cell has codimension $|J|-|I|$, completing the proof. 
\end{proof}

\begin{corollary}[Explicit description of the causal complex]
\label{cor:ttfs_causal_complex_explicit}
For every nonempty $S\subseteq[d]$, let 
\[
t_v^S(\vec t)
\coloneqq
\frac{\theta_v+\sum_{i\in S}w_i t_i}{W_S},
\qquad
W_S=\sum_{i\in S}w_i,
\]
and, for $j\in[d]$,
\begin{equation}
\label{eq:ttfs_wall_form}
\phi_{S,j}(\vec t)
\coloneqq
W_S t_j-\sum_{i\in S}w_i t_i-\theta_v
=
W_S\bigl(t_j-t_v^S(\vec t)\bigr).
\end{equation}
Then the dual causal complex admits the following explicit description.

For every $\emptyset\neq I\subseteq J\subseteq[d]$, the relative interior of
the dual cell indexed by $(I,J)$ is
\begin{equation}
\label{eq:dual_cell_explicit}
\mathcal R_{I,J}^{\circ}
=
\left\{
\vec t\in\R^d:
\begin{array}{ll}
\phi_{I,i}(\vec t)<0, & i\in I,\\
\phi_{I,j}(\vec t)=0, & j\in J\setminus I,\\
\phi_{I,k}(\vec t)>0, & k\notin J
\end{array}
\right\}.
\end{equation}
Equivalently,
\[
t_i<\tau_I(\vec t)\quad(i\in I),\qquad
t_j=\tau_I(\vec t)\quad(j\in J\setminus I),\qquad
t_k>\tau_I(\vec t)\quad(k\notin J).
\]
\end{corollary}

Corollary~\ref{cor:ttfs_causal_complex_explicit} shows, taking $I=J=S$, that the full-dimensional causal region indexed by a nonempty
$S\subseteq[d]$ is
\begin{equation}
\label{eq:causal_region_explicit}
\mathcal R_S^{\circ}
=
\left\{
\vec t\in\R^d:
\phi_{S,i}(\vec t)<0\ \text{for }i\in S,\quad
\phi_{S,j}(\vec t)>0\ \text{for }j\notin S
\right\},
\end{equation}
with closure obtained by replacing the strict inequalities by weak ones. 

Two full-dimensional regions $R_S$ and $R_T$ share a codimension-one boundary precisely when the corresponding vertices $p_S$ and $p_T$ are joined by an edge of the subdivision, meaning that $|S\triangle T|=1$. 
Thus a codimension-one dual cell has $J=S\cup\{j\}$ for some $j\notin S$. Its defining equality is
\[
t_j=t_v^S(\vec t), 
\]
which is equivalent to $W_S t_j-\sum_{i\in S}w_i t_i=\theta_v$. 
The codimension-one boundary between the adjacent causal regions
$\mathcal R_S$ and $\mathcal R_{S\cup\{j\}}$, where
$\emptyset\neq S\subseteq[d]$ and $j\notin S$, lies in the affine hyperplane
\begin{equation}
\label{eq:ttfs_causal_wall}
\mathcal H_{S,j}(w,\theta_v)
\coloneqq
\left\{
\vec t\in\R^d:
W_S t_j-\sum_{i\in S}w_i t_i=\theta_v
\right\}.
\end{equation}

Hence the causal complex of a positive TTFS neuron is supported on the
parameterized hyperplane family
\begin{equation}
\label{eq:ttfs_hyperplane_family}
\mathcal A(w,\theta_v)
=
\left\{
\mathcal H_{S,j}(w,\theta_v):
\emptyset\neq S\subseteq[d],\ j\notin S
\right\}.
\end{equation}

\begin{proof}[Proof of Corollary~\ref{cor:ttfs_causal_complex_explicit}]
By Theorem~\ref{thm:ttfs_dual_complex_structure}, the relative interior of
the dual cell indexed by $\emptyset\neq I\subseteq J\subseteq[d]$ consists
precisely of the inputs for which, with $\tau=t_v(\vec t)$,
\[
t_i<\tau\quad(i\in I),\qquad
t_j=\tau\quad(j\in J\setminus I),\qquad
t_k>\tau\quad(k\notin J).
\]
On this cell, the indices in $I$ are exactly those contributing strictly
positively to the membrane potential at firing time. Hence $\theta_v
=
\sum_{i\in I}w_i(\tau-t_i)$, 
and therefore 
$
\tau
=
\frac{\theta_v+\sum_{i\in I}w_i t_i}{W_I}
=
t_v^I(\vec t)$. 
Since $W_I>0$,
\[
\phi_{I,j}(\vec t)
=
W_I\bigl(t_j-t_v^I(\vec t)\bigr),
\]
so the three comparisons with $t_v^I(\vec t)$ are equivalent to the corresponding sign conditions in~\eqref{eq:dual_cell_explicit}. 
A codimension-one dual cell has $J=S\cup\{j\}$ for some $j\notin S$.
Its defining equality is
\[
t_j=t_v^S(\vec t), 
\]
which is equivalent to $W_S t_j-\sum_{i\in S}w_i t_i=\theta_v$. 
This gives~\eqref{eq:ttfs_causal_wall}. Since every codimension-one cell is
of this form, the causal complex is supported on the hyperplane family
$\mathcal A(w,\theta_v)$.
\end{proof}

\subsection{Interpretation} 

This construction above is analogous to the lifted Newton-polytope constructions familiar from tropical geometry \cite{MaclaganSturmfels2015Tropical}, which have been used in previous studies of piecewise linear neural networks, including ReLU networks 
\cite{ZhangNaitzatLim2018TropicalGeometryDNN, CharisopoulosMaragos2018}, maxout networks  \cite{montufar2022sharpboundsmaxout, balakin2025maxoutpolytopes}, and max-pooling \cite{maxpooling}. 
In contrast to standard maxout units, which are 
pointwise maxima of parametric affine functions with unconstrained coefficient vectors, the coefficients of the SNN firing-time map are highly constrained. 
Indeed, the $2^d-1$ lifted coefficient vectors $\vec y_S$, $\emptyset\neq S\subseteq[d]$, are determined by the $d$ weights $w_1,\ldots, w_d$ and the firing threshold $\theta_v$. 
Moreover, simultaneous positive scaling of all weights and the threshold leaves the lifted coefficient vectors unchanged, so this family has only $d$ effective degrees of freedom. 
The resulting geometry also differs from the zonotopal geometry arising in polyhedral descriptions of single-hidden layer ReLU networks. 
Already for $d=3$ with $w_1 = w_2 = w_3 = \theta_v=1$, the polytope $Q$ has the seven vertices, corresponding to the seven nonempty subsets of $[3]$. 
Since every positive-dimensional zonotope is centrally symmetric and therefore has an even number of vertices, this seven-vertex polytope cannot be a zonotope.

\section{Network-level causal patterns} 
\label{subsec:full_signature_counts}

For a single neuron, each causal region $R_S$ is a convex polyhedron. 
In a layered network, the same principle applies recursively: 
the firing behavior of each neuron is determined by the vector of spike times it receives from the previous layer, which in turn determines its causal set. 
This suggests that the natural combinatorial object at the network level is the collection of causal sets across all neurons. 
We refer to this collection as the causal pattern of the network.

\begin{definition}
\label{def:layerwise_causal_set_pattern_nodelay}
Let $\Phi = (W^\ell,\Theta^\ell)_{\ell=1}^L$ be a feedforward SNN of depth $L$ with layer widths $N_0,\dots,N_L$, and fix input spike times $\vec t:=t^{(0)} \in \R^{N_0}$. 
For neuron $(\ell,i)$ (the $i$th neuron in the $\ell$th layer), 
denote the presynaptic spike time vector by 
\begin{equation*}
    t^{\ell-1}(\vec t) = (t^{\ell-1}_1(\vec t), \dots, t^{\ell-1}_{N_{\ell-1}}(\vec t)) \in \R^{N_{\ell-1}} . 
\end{equation*}
Following \eqref{eq:region_Rs} with input dimension $N_{\ell-1}$, weights $(w^\ell_{ij})_{j=1}^{N_{\ell-1}}$, and threshold $\theta^\ell_i$, 
the causal set of this neuron at a presynaptic vector $t^{\ell-1}(\vec t)$ is the unique subset 
$
    S^\ell_i 
    \subseteq [N_{\ell-1}]
$ 
such that $t^{\ell-1}(\vec t) \in R_{S^\ell_i 
}$. 

The \emph{causal pattern} of $\Phi$ at input $\vec t$ is the assignment 
\begin{equation*}
    \mathbf{S}(\vec t) \coloneqq \bigl(S^\ell_i(\vec t)\bigr)_{i\in[N_\ell], \ell\in[L]}.
\end{equation*}
Thus $\mathbf{S}$ is a map $\mathbf{S} \colon \R^{N_0} \to \bigtimes_{\ell\in[L]}\bigtimes_{i\in[N_\ell]} 2^{[N_{\ell-1}]}$. 

We define the \emph{region count} of $\Phi$ to be the number of distinct causal patterns realized by inputs in $\R^{N_0}$:  
\begin{equation*}
    R(\Phi) \coloneqq \bigl \lvert \mathbf{S}(\R^{N_0})\bigr \rvert.
\end{equation*}
Furthermore, for fixed layer widths and depth, we define the \emph{maximal region count} to be the maximum of the region count over all possible choices of weights and thresholds: 
\begin{equation*}
    R_{\max}(N_0, \dots, N_L) \coloneqq \max_{\Phi} R(\Phi). 
\end{equation*}
\end{definition}

In analogy with fixing an activation pattern in a ReLU network, the causal pattern $\mathbf{S} = (S_{i}^{\ell})_{i,\ell}$ serves as a combinatorial descriptor of the network’s computation. 
Specifically, it records, layer by layer and neuron by neuron, which presynaptic spikes arrive early enough to influence the firing time of each postsynaptic neuron. In this way, it encodes the causal structure governing the propagation of information through the network. 
This reflects the fact that, unlike neurons in a conventional feedforward network, neurons in an SNN do not process an entire input vector at once. Instead, they integrate incoming spikes over time and fire as soon as their membrane potential threshold is reached.

\begin{remark}
    For a single neuron, the causal set is determined by the feasibility inequalities \eqref{eq:ineqs}. 
    Its causal regions are described as unions of regions of the arrangement $\cA_{\rm TTFS}$. 
    For network, we work with a causal pattern $\mathbf{S}$. 
    Once a full causal pattern is fixed, every neuron uses a fixed causal set, and therefore every firing time becomes an affine function of the network's input $\vec t$. 
    Substituting these affine expressions recursively through the layers turns all consistency conditions for the pattern into linear inequalities in $\vec t$. 
    Hence the causal pattern region 
    \begin{equation*}
        R_{\mathbf{S}} \coloneq \{\vec t \in \R^{N_0} \colon \mathbf{S}(\vec t)=\mathbf{S}\}
\end{equation*}
    is cut out directly by affine halfspaces in the input space. 
    We formalize this next in Proposition~\ref{prop:convex_polyhedron_SPhi}. 
\end{remark}

\begin{proposition}\label{prop:convex_polyhedron_SPhi}
    Let $\Phi = (W^\ell,D^\ell,\Theta^\ell)_{\ell=1}^L$ be a feedforward SNN of depth $L$ with layers of widths $N_0,\dots,N_L$. 
    For every causal pattern $\mathbf{S}$, the corresponding region $R_{\mathbf{S}}$ is a convex polyhedron. 
    Moreover, on $R_{\mathbf{S}}$, each firing time $t_i^\ell(\vec t)$ is an affine function of the input $\vec t \in \R^{N_0}$. 
\end{proposition}
The proof is deferred to the Appendix~\ref{appendix:convex_polyhedron_SPhiproof}.

\section{Causal region complexity for shallow SNNs} 
\label{section:boundsshallowsnn}
In this section, we consider shallow SNNs with input dimension $N_0=d$ and a single hidden layer of width $N_1 = m$, and study the number of distinct causal-set tuples they can realize. 
We first derive general upper and lower bounds for arbitrary positive weights. 
These bounds already show that, for fixed input dimension $d$, the number of realizable causal regions grows polynomially with the hidden-layer width $m$, despite the exponentially many possible causal sets for each individual neuron. 
We then consider the more structured class of SNNs where all hidden neurons share the same weight vector, for which we obtain explicit region counts.

\subsection{Upper bounds} 
\label{sec:upperboundshallowSNN}

We begin with a simple general upper bound that we refine further below. 

\begin{proposition}
\label{prop:naive_bounds}
A single neuron with $d$ inputs admits $2^d-1$ nonempty causal sets. 
Consequently, for a shallow SNN with 
$m$ hidden neurons and input dimension $d$, 
the immediate upper and lower bounds are  
\begin{equation}
\label{eq:naive_bound_shallow}
    R_{\max}(d,m)\ \le\ (2^d-1)^m, 
\end{equation}
and 
\begin{equation*}
    R_{\max}(d,m)\ \ge\ 2^d-1. 
\end{equation*}
\end{proposition}

The upper bound is typically very loose. 
Since the $m$ neurons share the same input $\vec t \in \R^d$ and their causal partitions are geometrically constrained, many tuples of causal sets cannot be realized simultaneously. 
The example below illustrates this. 

\begin{example}
    Let $d=2$ and consider a shallow SNN with $m=2$ hidden neurons. 
    For $r=1,2$, let neuron $r$ have positive weights $(w_1^{(r)}, w_2^{(r)}) \in \R^2_{>0}$ and threshold $\theta_r > 0$. 
    For an input $\vec t = (t_1, t_2) \in \R^2$, we let $\Delta \coloneqq t_2 - t_1$. 
    
    Each neuron $r$ admits the three causal sets $\{1\},\{2\},\{1,2\}$. Their causal regions can be expressed entirely in terms of $\Delta$ as  
    \begin{equation*}
        S_r = \begin{cases}
            \{1\}, & \Delta>\ \theta_r/w^{(r)}_1,\\
            \{2\}, & \Delta<-\theta_r/w^{(r)}_2,\\
            \{1,2\}, & -\theta_r/w^{(r)}_2 \leq \Delta \leq \theta_r/w^{(r)}_1.
            \end{cases}
    \end{equation*}
Thus, although the two neurons may have different weights and thresholds, their causal sets are determined by the same scalar variable $\Delta$. 
Consequently, not all pairs of causal sets can occur, regardless of the choice of parameters. 
For instance, the pair $\bigl(S_1(\vec t), S_2(\vec t)\bigr)=\bigl(\{1\},\{2\}\bigr)$ would require simultaneously $\Delta > \theta_1/w^{(1)}_1>0$ and $\Delta < -\theta_2/w^{(2)}_2<0$, which is impossible. 
Similarly, the pair $(\{2\}, \{1\})$ is also impossible. 
Thus, even though each neuron individually admits three causal sets, not all $3^2 = 9$ causal-set pairs are realizable. 
\end{example}

We next derive a sharper upper bound by considering the finite hyperplane arrangement $\cA$ whose regions refine the causal partition of the input space. 
Since the causal-set tuple is constant on each region of $\cA$, the number of realizable causal-set tuples is bounded by the number of regions of the arrangement.

For each neuron, consider the TTFS arrangement $\cA_{\rm TTFS}$ introduced in Section~\ref{sec:causal-hyperplanes}. 
We first count the number of hyperplanes in this arrangement. 
For each nonempty $S$, there are $d-\lvert S \rvert$ choices of $j\notin S$. 
Thus, the number of hyperplanes in the 
TTFS arrangement of a single neuron is 
\begin{equation}
\label{eq:N_d}
\kappa_d \coloneq \#\{H_{S,j}:\  j\notin S, \emptyset \neq S\subseteq[d]\} = \sum_{\emptyset\neq S\subseteq[d]}(d-|S|)=d(2^{d-1}-1) , 
\end{equation}
where the last equality follows from the standard binomial identities 
$\sum_{k=0}^d\binom{d}{k}=2^d$ and $\sum_{k=0}^d k\binom{d}{k}=d2^{d-1}$.

For a shallow SNN with $m$ hidden neurons and $d$ inputs, 
let $\cA$ denote the union of all TTFS arrangements associated with the $m$ neurons. 
Since each neuron contributes $\kappa_d$ hyperplanes, the arrangement $\cA$ contains at most $m\kappa_d$ hyperplanes. 
Recall from Section~\ref{sec:causal-hyperplanes} that each TTFS arrangement, and hence their union, can be essentialized to $T = \1^\bot \cong\mathbb{R}^{d-1}$. 
Combining this observation with the standard region bound for hyperplane arrangements yields the following upper bound.

\begin{proposition}\label{prop:upper_lower_bound_finite_Hsj}
Consider a shallow SNN with $m$ hidden neurons for input $\vec t \in \R^d$. Let
$\cA$ be the corresponding arrangement. 
Then, the causal set tuple $(S_1,\dots,S_m)$ is constant on each region of 
$\mathbb{R}^d\setminus \cA$. 
Consequently,
\begin{equation}
\label{eq:upper_region_count_refined}
R_{\max}(d,m) \leq \min\Biggl\{\sum_{k=0}^{d-1}\binom{m\kappa_d}{k}, (2^d-1)^m \Biggr\}
\end{equation}
where $\kappa_d = d(2^{d-1}-1)$.  
Moreover, 
\begin{enumerate}[(i)]
    \item ($m \to \infty$ and fixed $d$). Then, 
    \begin{equation}
    \label{eq:fixed_d_asympt}
        R_{\max}(d,m) = \mathcal{O}\bigl((m\kappa_d)^{d-1}\bigr) = \mathcal{O}\bigl(m^{d-1}\bigr).
    \end{equation}

    \item ($d \to \infty$ and fixed $m$). Then,
    \begin{equation}\label{eq:large_d_asympt}
        R_{\max}(d,m) \leq \min \Biggl\{ d\left(\frac{e\,m\kappa_d}{d-1}\right)^{d-1}, (2^d-1)^m \Biggr\} . 
    \end{equation}
    Since $\kappa_d = d(2^{d-1}-1) = \Theta(d2^d)$, this yields $R_{\max}(m,d) \leq 2^{\mathcal{O}(d)}$ for fixed $m$.
\end{enumerate}
\end{proposition}

The causal-set tuple is constant on each region of the arrangement $\cA$, and therefore the number of distinct causal tuples is bounded by the number of regions of the arrangement. 
The result then follows from the classical bound for the number of regions cut out by $m \kappa_d$ hyperplanes in $\R^{d-1}$; see \cite{zaslavsky1975}. The asymptotic bounds are direct.  
The full proof is given in the Appendix~\ref{subsec:appendix_proof_Attfs_ub}. 

For fixed $d$, the bound in \eqref{eq:fixed_d_asympt} grows polynomially in $m$ (of degree at most $d-1$), whereas the naive bound $(2^d-1)^m$ in \eqref{eq:naive_bound_shallow} grows exponentially in $m$. 
On the other hand, for small $m$ and large $d$, the bound 
obtained from the union of arrangements can exceed the naive bound (e.g., when $d=3$ and $m=1$).  
Even in the fixed-$d$ regime, however, the arrangement bound need not give the exact number of realizable causal tuples. 
Indeed, we count all the regions from a large arrangement, while many regions can have the same causal set tuple; see Figure~\ref{fig:causal_regions_d2_d3}.

\subsection{Lower bounds} 
We derive a lower bound on the maximum number of realizable causal set tuples. 
We give a constructive argument showing that the polynomial growth in the width $m$ predicted by the upper bound is asymptotically attainable. 
Further below, in Section~\ref{sec:lower-bounds-shallow-SNN-shared-weights}, we consider the more structured setting in which all neurons share the same weight vector. In this regime, the causal-region partition has sufficient structure to allow an exact count. 
The resulting bound recovers the same asymptotic growth in $m$ for fixed $d$, while also revealing exponential growth in $d$ for fixed $m$.

The next result gives an asymptotic lower bound. 

\begin{proposition}
\label{prop:lower_bound_asymptotic}
Consider a shallow network with $m$ hidden neurons and $d \geq 2$ input neurons. There exist parameters such that
\begin{equation*}
    R_{\max}(d,m) \geq \binom{m-1}{d-1} = \Omega(m^{d-1}), \quad (d \text{ fixed}).
\end{equation*}
\end{proposition}

The proof is constructive. The idea is that in an open region of input space, each hidden neuron contributes a single hyperplane 
distinguishing two causal sets. 
By choosing these $m$ hyperplanes in general position inside $T\cong \R^{d-1}$, one realizes at least as many causal regions as the number of bounded regions of a generic arrangement of $m$ hyperplanes in $\R^{d-1}$, that is $\binom{m-1}{d-1}$. The full proof is given in Appendix~\ref{subsec:appendix_proof_asymp_lb}. 

Propositions~\ref{prop:upper_lower_bound_finite_Hsj} and~\ref{prop:lower_bound_asymptotic} show that, for fixed input dimension $d$, 
the maximum number of causal regions of a shallow network has a tight asymptotic behavior $\Theta(m^{d-1})$ and thus grows polynomially in the width $m$. 
Obtaining an exact formula will require exploiting the precise structure of the causal regions beyond what is visible from the hyperplanes alone.

\subsection{Shallow SNNs with shared weights}
\label{sec:lower-bounds-shallow-SNN-shared-weights}

We now consider a structured class of SNNs that permits exact enumeration of causal regions, and yields sharper general lower bounds. 
Specifically, we study the regime in which all hidden neurons share the same positive weight vector. 
This shared-weight structure imposes a nestedness property of the causal sets, which leads to explicit counting formulas. 
These allow us to obtain improved lower bounds for shallow networks and upper bounds for deep networks with shared weights. 

\begin{lemma}
\label{lem:nested_same_weights}
Let $d\ge2$. 
Then, for every input $\vec t \in \mathbb R^d$, the map $\theta\mapsto t_v(\vec t;\theta)$ from threshold parameter to firing time is strictly increasing. Consequently, for any $0<\theta_1<\theta_2$, we have
\begin{equation*}
t_v(\vec t;\theta_1)\ < \ t_v(\vec t;\theta_2)
\qquad\Longrightarrow\qquad
S(\theta_1)\subseteq S(\theta_2).
\end{equation*}
\end{lemma}

\begin{proof}
For fixed $\vec t$, recall that the membrane potential $P_v(t)$ of a neuron $v$ in \eqref{eq:membrane_potential} is given as 
\begin{equation*}
    P_v(t) = \sum_{i=1}^d w_i \, \sigma(t - t_i),
\end{equation*}
which is continuous and increasing in $t$ because all $w_i>0$. Increasing the firing threshold $\theta$ increases the first hitting time $t_v(\vec t;\theta)=\min\{t\in \R: P_v(t) = \theta\}$. 
If $t_v(\theta_1) < t_v(\theta_2)$, then any index $i$ with $t_i<t_v(\theta_1)$ also satisfies
$t_i<t_v(\theta_2)$, hence $S(\theta_1)\subseteq S(\theta_2)$.
\end{proof}

Lemma~\ref{lem:nested_same_weights} is the basic structural fact behind the next result. It shows that, under shared weights, the causal-set tuple of a shallow SNN must form a chain of nested subsets. 

We illustrate this for $m=2$ and $d=3$ in the following example. 

\begin{example}[$m=2, d=3$]
\label{example:19_same_weights_d3}
Let $d=3$ and consider two hidden neurons with the same positive weight vector $\vec{w} \in \mathbb R_{>0}^3$ and thresholds $0<\theta_1<\theta_2$. Let $S_r$ be the causal set of neuron $r$ at input $\vec t$. 
Then, for every $\vec t$, Lemma~\ref{lem:nested_same_weights} gives $S_1 \subseteq S_2$. 
In particular, the number of distinct nonempty causal-set tuples $(S_1,S_2)$ that can occur 
is at most the number of nonempty pairs $(S_1,S_2) \in 2^{[d]}\times 2^{[d]}$ with $S_1\subseteq S_2$, which equals $19$, namely:

\medskip
\noindent\textbf{(i) $|S_2|=1$ (3 pairs).} $(\{1\},\{1\}), (\{2\},\{2\}), (\{3\},\{3\}).$

\medskip
\noindent\textbf{(ii) $|S_2|=2$ (9 pairs).}
\begin{align*} (\{1\},\{1,2\}), (\{2\},\{1,2\}), (\{1,2\},\{1,2\}), \\
(\{1\},\{1,3\}), (\{3\},\{1,3\}), (\{1,3\},\{1,3\}), \\
(\{2\},\{2,3\}), (\{3\},\{2,3\}), (\{2,3\},\{2,3\}). 
\end{align*} 

\medskip
\noindent\textbf{(iii) $|S_2|=3$ (7 pairs).} 
\begin{align*}
    & (\{1\},\{1,2,3\}), (\{2\},\{1,2,3\}), (\{3\},\{1,2,3\}), \\
    & (\{1,2\},\{1,2,3\}), (\{1,3\},\{1,2,3\}), (\{2,3\},\{1,2,3\}) \\
    & (\{1,2,3\}, \{1,2,3\}).
\end{align*}
\end{example}

The number of nested sequences can be given explicitly as follows. 
\begin{lemma}
\label{lem:nested-squences}
The number of nondecreasing sequences of nonempty subsets of $[d]$ is 
    \begin{equation}
    \label{eq:counting_chains}
        \sum_{k=1}^{d} \binom{d}{k}\,\bigl( m^{k}-(m-1)^{k}\bigr) = (m+1)^d - m^d.
    \end{equation}
\end{lemma}

Lemmas \ref{lem:nested_same_weights} and \ref{lem:nested-squences} immediately yield an upper bound on the number of realizable causal-set tuples in shared-weights shallow SNNs. 
We show that this upper bound is attainable, which yields the following result, giving the exact number of the maximum of causal-set tuples realizable in shallow SNNs. 

\begin{proposition} 
\label{prop:chain_bound_general_dm}
Let $d\ge 2$ and $m\ge 1$. Consider a shallow SNN with input dimension $d$ and $m$ hidden neurons sharing the same positive weight vector $\vec w \in \R_{\geq 0}^d$ and having pairwise distinct thresholds $0<\theta_1<\theta_2<\cdots<\theta_m$. 
For an input $\vec t\in\R^d$, let 
$S_r \coloneqq S(\vec t;\theta_r)\subseteq[d]$ 
denote the nonempty causal set of hidden neuron $r$. 
Then the maximum number of distinct realizable causal-set tuples $(S_1, \dots, S_m)$ is 
\begin{equation}
\label{eq:chain_count}
       R_{\max}^{\mathrm{shared}}(d,m) 
      = 
        (m+1)^d - m^d.
    \end{equation}
\end{proposition}

The full proof is presented in Appendix~\ref{appendix:proof_shallowSNN_chainbound}.

\begin{remark}
    The nesting of the causal sets has a direct geometric interpretation in the polyhedral picture of Section~\ref{subsec:ttfs_polytope}.  
    Since all hidden neurons share the same weight vector, the coefficients $\vec a_S$ are identical across neurons, while only the offsets $b_S = \theta_r/W_S$ vary with the threshold $\theta_r$. 
    Thus, increasing $\theta$ translates each causal boundary parallel to itself. 
    In the $d=3$ example illustrated in Figure~\ref{fig:ttfs_polytope_c}, the central region $R_{\{1,2,3\}}$ expands as $\theta$ increases, while the six unbounded boundaries move outward without changing directions. 
    Hence, for $\theta_1 < \theta_2$, the corresponding causal partitions form nested, scaled copies of one another. 
    This reflects the inclusion $S_1 \subseteq S_2$. 
\end{remark}

\subsection{Interpretation} 

For a single-hidden-layer SNN with input dimension $d$ and $m$ hidden neurons, Propositions~\ref{prop:upper_lower_bound_finite_Hsj} and \ref{prop:lower_bound_asymptotic} imply that, for fixed $d$,  
$$
R_{\max}(d,m) = \Theta(m^{d-1}) \qquad \text{as $m\to\infty$}. 
$$
On the other hand, Proposition~\ref{prop:chain_bound_general_dm} gives, for fixed $m$,  
$$
R_{\max}(d,m) \geq 
R_{\max}^{\mathrm{shared}}(d,m) = 
(m+1)^d - m^d = (1-o(1))(m+1)^d
\qquad \text{as $d\to\infty$}. 
$$

For comparison, consider a single-hidden-layer ReLU-ANN with $d$ inputs and $m$ hidden neurons. 
The classical hyperplane-arrangement formula gives $R_{\max}^{\rm ReLU}(d,m) 
    = \sum_{j=0}^{d}\binom{m}{j}$ (see, e.g., \cite{pascanu2013responseregions}).  
Consequently, for fixed $d$, 
$$
R_{\max}^{\rm ReLU}(d,m) = \Theta(m^d)\qquad  \text{as $m\to\infty$}, 
$$
whereas for fixed $m$ and $d\geq m$,  
$$
R_{\max}^{\rm ReLU}(d,m) = 2^m . 
$$

Thus, the two models have qualitatively different scaling in the two asymptotic regimes. For fixed input dimension $d$, 
the maximum number of regions of a shallow ReLU-ANN grows as $m^d$, which is one power of $m$ faster than the $\Theta(m^{d-1})$ growth of realizable causal-set tuples in a shallow SNN. 
In contrast, 
for fixed width $m$, the ReLU region count saturates at $2^m$ once $d\geq m$, whereas the SNN admits at least 
$(m+1)^d - m^d = (1-o(1))(m+1)^d$ distinct causal-set tuples. 
Hence the SNN lower bound grows exponentially with the input dimension~$d$.

It is also worth noting that arbitrary positive and shared weight SNNs have rather different causal geometry, despite having the same asymptotic dependence on width. With arbitrary positive weights, different neurons can introduce causal boundaries with different orientations, allowing more freedom in how their causal partitions intersect. Under shared weights, by contrast, different neurons have same boundary orientations and varying the thresholds produces the nested structure described above. Thus, the additional geometric freedom provided by neuron-specific weights may affect the exact finite-width count and quantifying this gap would be an interesting direction for future work. 
Figure~\ref{fig:snn_regions_d3L2L3} illustrates this difference in causal region geometry for a simple case of input dimension $d=3$.

\section{Causal region complexity for deep SNNs}
\label{sec:deepSNN}

In this section, we give upper and lower bounds on the maximum number of distinct causal patterns realizable by deep SNNs. 
We first establish a general upper bound for arbitrary positive weights. 
Then we give a constructive lower bound that grows exponentially with depth. 
We further consider the shared weight regime, for which we obtain sharper upper bound.

Consider a feedforward SNN of depth $L$ with layer widths $N_0=d,\; N_1,\dots,N_L$. 
Given an input spike-time vector $\vec t^{(0)} \in \R^d$, 
layer $\ell$ produces the spike-time vector  
\begin{equation*}
    \vec t^{(\ell)}(\vec t^{(0)}) = \bigl(t^{(\ell)}_1(\vec t^{(0)}),\dots,t^{(\ell)}_{N_\ell}(\vec t^{(0)})\bigr) \in \R^{N_\ell} . 
\end{equation*}
For the $r$th neuron in layer $\ell$, define its causal set relative to layer $\ell-1$ by 
\begin{equation*}
    S^{(\ell)}_r \coloneq \bigl\{j\in[N_{\ell-1}] \colon t^{(\ell-1)}_j < t^{(\ell)}_r \bigr\}.
\end{equation*}
Thus, $S^{(\ell)}_r$ records the presynaptic neurons in layer $\ell-1$ that spike strictly before neuron $(\ell,r)$. 
The full causal pattern of the network is the tuple 
\begin{equation*}
    \Bigl((S_r^{(1)})_{r=1}^{N_1}, (S_r^{(2)})_{r=1}^{N_2}, \dots, (S_r^{(L)})_{r=1}^{N_L} \Bigr).
\end{equation*}

\subsection{Upper bounds} 

The idea is to analyze the layers successively. 
Once the causal pattern through layer $\ell - 1$ is fixed, all firing times produced by the corresponding subnetwork are affine functions of the original input. 
Therefore, within each such causal region, the boundaries introduced by layer $\ell$ are obtained by pulling back the layer-$\ell$ arrangement through the affine map realized by the preceding layers. 
Applying the shallow-network bound within each region therefore yields a product of the layerwise upper bounds.

\begin{theorem}\label{thm:deepupperbound}
    For a feedforward SNN of architecture $(d,N_1,\ldots,N_L)$, we have 
    \begin{equation}
    \label{eq:deep_product_shallow_upper}
        R_{\max}(d, N_1, \dots, N_L) \leq \prod_{\ell=1}^{L} \min\Biggl\{
        (2^{N_{\ell-1}}-1)^{N_\ell}, \sum_{k=0}^{N_{\ell-1}-1} \binom{N_\ell\kappa_{N_{\ell-1}}}{k} \Biggr\},
    \end{equation}
    where $\kappa_{N_\ell - 1} = N_{\ell-1}(2^{N_\ell - 1} - 1)$ is given in \eqref{eq:N_d}.

\end{theorem}

\begin{proof}
    We apply the shallow bounds successively to the layers. Consider first layer the $\ell$, and suppose that the causal patterns of all preceding layers $1,\dots,\ell-1$ have been fixed. On the corresponding region of the input space, the presynaptic spike times $\vec t^{\ell - 1}$ entering layer $\ell$ are fixed functions of the input. Viewed in its presynaptic spike-time space $\R^{N_{\ell-1}}$, layer $\ell$ is therefore a shallow SNN with $N_{\ell-1}$ inputs and $N_\ell$ neurons. 
    By Proposition~\ref{prop:upper_lower_bound_finite_Hsj}, the maximum number of regions produced by the arrangement are 
    \begin{equation*}
        \sum_{k=0}^{N_{\ell-1}-1} \binom{N_\ell\kappa_{N_{\ell-1}}}{k}.
    \end{equation*} 
    Restricting this arrangement to the affine image of the preceding region cannot produce more regions than the arrangement in the full space $\R^{N_\ell-1}$.
    Hence, each region generated through layer $\ell-1$ can be subdivided by layer $\ell$ into at most
    \begin{equation*}
        \min\Biggl\{
        (2^{N_{\ell-1}}-1)^{N_\ell}, \sum_{k=0}^{N_{\ell-1}-1} \binom{N_\ell\kappa_{N_{\ell-1}}}{k} \Biggr\},
    \end{equation*}
    where $(2^{N_{\ell-1}}-1)^{N_\ell}$ is the naive causal set bound given by Proposition~\ref{prop:naive_bounds} for the same layer. 
    Multiplying these factors over $\ell=1,\dots,L$ gives \eqref{eq:deep_product_shallow_upper}.  
\end{proof}

The shallow polynomial asymptotic bound in \eqref{eq:fixed_d_asympt} does not directly extend to deep networks when multiple layer widths grow simultaneously, since the presynaptic dimension $N_{\ell-1}$ then also varies with the network widths. 
We therefore retain the finite product bound in \eqref{eq:deep_product_shallow_upper} as our general upper bound for deep network. 
This bound can be quite loose, since it treats the causal sets of different neurons within each layer as if they could vary independently. 
In reality, all neurons in a given layer are driven by the same presynaptic spike-time vector, and deeper layers are further constrained by the causal structure inherited from the preceding layers. 
In Section~\ref{subsec:upperboundsharedweight}, we exploit these dependencies to obtain sharper bounds in the structured shared-weight regime.

\subsection{Lower bounds} 

We now derive a constructive lower bound for an SNN with arbitrary positive weights. 

The construction is based on the following principle. 
Suppose that a subnetwork produces $M$ pairwise distinct causal regions $V_1, \dots, V_M$ and maps each of them affinely and bijectively onto the same output region $W$. 
If $W$ is itself an admissible input region for another copy of the construction, then the same folding mechanism can be applied again to each of the $M$ existing regions. 
Iterating this construction $K$ times therefore produces $M^K$ distinct network-level causal patterns. 
We next summarize the multi-fold construction underlying Lemma~\ref{lemma:multi-fold}, which will be subsequently used to prove Theorem~\ref{thm:multifold_depth}.

\begin{figure}[t]
\centering

\makebox[\textwidth][c]{%
\resizebox{0.96\textwidth}{!}{%
\begin{tikzpicture}[
    >=Latex,
    font=\small,
    line join=round
]

\definecolor{mfcurve}{RGB}{190,92,92}
\definecolor{mfblue}{RGB}{78,125,187}
\definecolor{mfgreen}{RGB}{92,149,103}
\definecolor{mfgray}{RGB}{125,125,125}

\definecolor{mfCOne}{RGB}{205,224,242}
\definecolor{mfCTwo}{RGB}{242,211,210}
\definecolor{mfCThree}{RGB}{216,232,205}
\definecolor{mfCFour}{RGB}{229,215,239}
\definecolor{mfMargin}{RGB}{238,238,238}

\definecolor{mfBranchOne}{RGB}{112,158,203}
\definecolor{mfBranchTwo}{RGB}{200,128,126}
\definecolor{mfBranchThree}{RGB}{126,168,111}
\definecolor{mfBranchFour}{RGB}{165,132,188}

\begin{scope}[xshift=0cm]

\draw[->,thick]
    (0,0) -- (4.10,0)
    node[right] {$t_1$};

\draw[->,thick]
    (0,0) -- (0,4.42)
    node[above] {$t_2$};

\draw[dashed,mfgray!60]
    (0,0) -- (4.05,4.05);

\node[
    font=\scriptsize,
    rotate=45,
    anchor=north
]
    at (3.28,3.28)
    {$t_2=t_1$};

\fill[mfMargin]
    (0,0.55) --
    (0,1.05) --
    (3.10,4.15) --
    (3.60,4.15) -- cycle;

\fill[mfCOne]
    (0,1.05) --
    (0,1.55) --
    (2.60,4.15) --
    (3.10,4.15) -- cycle;

\fill[mfCTwo]
    (0,1.55) --
    (0,2.05) --
    (2.10,4.15) --
    (2.60,4.15) -- cycle;

\fill[mfCThree]
    (0,2.05) --
    (0,2.55) --
    (1.60,4.15) --
    (2.10,4.15) -- cycle;

\fill[mfCFour]
    (0,2.55) --
    (0,3.05) --
    (1.10,4.15) --
    (1.60,4.15) -- cycle;

\fill[mfMargin]
    (0,3.05) --
    (0,3.55) --
    (0.60,4.15) --
    (1.10,4.15) -- cycle;

\draw[densely dotted,mfgray!60,line width=.55pt]
    (-0.42,0.13) -- (0,0.55);

\draw[densely dotted,mfgray!60,line width=.55pt]
    (-0.42,0.63) -- (0,1.05);

\draw[densely dotted,mfgray!60,line width=.55pt]
    (-0.42,1.13) -- (0,1.55);

\draw[densely dotted,mfgray!60,line width=.55pt]
    (-0.42,1.63) -- (0,2.05);

\draw[densely dotted,mfgray!60,line width=.55pt]
    (-0.42,2.13) -- (0,2.55);

\draw[densely dotted,mfgray!60,line width=.55pt]
    (-0.42,2.63) -- (0,3.05);

\draw[densely dotted,mfgray!60,line width=.55pt]
    (-0.42,3.13) -- (0,3.55);

\draw[thick,mfgray!90]
    (0,0.55) -- (3.60,4.15);

\draw[mfgray,line width=.70pt]
    (0,1.05) -- (3.10,4.15);

\draw[mfgray,line width=.70pt]
    (0,1.55) -- (2.60,4.15);

\draw[mfgray,line width=.70pt]
    (0,2.05) -- (2.10,4.15);

\draw[mfgray,line width=.70pt]
    (0,2.55) -- (1.60,4.15);

\draw[mfgray,line width=.70pt]
    (0,3.05) -- (1.10,4.15);

\draw[thick,mfgray!90]
    (0,3.55) -- (0.60,4.15);

\node[
    font=\scriptsize,
    anchor=east,
    fill=white,
    inner sep=.4pt
]
    at (-0.10,0.55)
    {$a$};

\node[
    font=\scriptsize,
    anchor=east,
    fill=white,
    inner sep=.4pt
]
    at (-0.10,1.05)
    {$c_1$};

\node[
    font=\scriptsize,
    anchor=east,
    fill=white,
    inner sep=.4pt
]
    at (-0.10,1.55)
    {$c_2$};

\node[
    font=\scriptsize,
    anchor=east,
    fill=white,
    inner sep=.4pt
]
    at (-0.10,2.05)
    {$c_3$};

\node[
    font=\scriptsize,
    anchor=east,
    fill=white,
    inner sep=.4pt
]
    at (-0.10,2.55)
    {$c_4$};

\node[
    font=\scriptsize,
    anchor=east,
    fill=white,
    inner sep=.4pt
]
    at (-0.10,3.05)
    {$c_5$};

\node[
    font=\scriptsize,
    anchor=east,
    fill=white,
    inner sep=.4pt
]
    at (-0.10,3.55)
    {$b$};

\node at (0.85,2.15) {\footnotesize$V_1$};
\node at (0.85,2.65) {\footnotesize$V_2$};
\node at (0.85,3.15) {\footnotesize$V_3$};
\node at (0.85,3.65) {\footnotesize$V_4$};

\node[
    font=\normalsize,
]
    at (2.125,4.5)
    {$Y_I$};

\draw[->]
    (0.05,-0.82)
    --
    (4.42,-0.82)
    node[right]
    {$x=t_2-t_1$};

\draw
    (0.30,-0.76) -- (0.30,-0.88)
    node[below=2pt,font=\scriptsize] {$0$};

\draw
    (0.78,-0.76) -- (0.78,-0.88)
    node[below=2pt,font=\scriptsize] {$a$};

\draw
    (1.34,-0.76) -- (1.34,-0.88)
    node[below=2pt,font=\scriptsize] {$c_1$};

\draw
    (1.86,-0.76) -- (1.86,-0.88)
    node[below=2pt,font=\scriptsize] {$c_2$};

\draw
    (2.38,-0.76) -- (2.38,-0.88)
    node[below=2pt,font=\scriptsize] {$c_3$};

\draw
    (2.90,-0.76) -- (2.90,-0.88)
    node[below=2pt,font=\scriptsize] {$c_4$};

\draw
    (3.42,-0.76) -- (3.42,-0.88)
    node[below=2pt,font=\scriptsize] {$c_5$};

\draw
    (3.96,-0.76) -- (3.96,-0.88)
    node[below=2pt,font=\scriptsize] {$b$};

\draw[mfgray!55,line width=5pt]
    (0.78,-0.82) -- (1.34,-0.82);

\draw[mfBranchOne,line width=5pt]
    (1.34,-0.82) -- (1.86,-0.82);

\draw[mfBranchTwo,line width=5pt]
    (1.86,-0.82) -- (2.38,-0.82);

\draw[mfBranchThree,line width=5pt]
    (2.38,-0.82) -- (2.90,-0.82);

\draw[mfBranchFour,line width=5pt]
    (2.90,-0.82) -- (3.42,-0.82);

\draw[mfgray!55,line width=5pt]
    (3.42,-0.82) -- (3.96,-0.82);

\node[
    font=\scriptsize,
    above=3pt
]
    at (2.37,-0.82)
    {$I=(a,b)$};

\end{scope}

\begin{scope}[xshift=5.35cm]

\draw[->]
    (0.35,0)
    --
    (4.55,0)
    node[right]
    {$x$};

\draw[->]
    (0.35,0)
    --
    (0.35,4.12)
    node[above]
    {$h(x)$};

\draw[dashed,mfgray!80]
    (0.35,1.00) -- (4.20,1.00);

\draw[dashed,mfgray!80]
    (0.35,2.95) -- (4.20,2.95);

\node[left=2pt] at (0.35,1.00) {$q$};
\node[left=2pt] at (0.35,2.95) {$q+A$};

\draw[densely dotted,mfgray!65]
    (0.72,0) -- (0.72,2.95);
\draw
    (0.72,.07) -- (0.72,-.07)
    node[below=3pt,font=\scriptsize] {$c_1$};

\draw[densely dotted,mfgray!65]
    (1.58,0) -- (1.58,2.95);
\draw
    (1.58,.07) -- (1.58,-.07)
    node[below=3pt,font=\scriptsize] {$c_2$};

\draw[densely dotted,mfgray!65]
    (2.44,0) -- (2.44,2.95);
\draw
    (2.44,.07) -- (2.44,-.07)
    node[below=3pt,font=\scriptsize] {$c_3$};

\draw[densely dotted,mfgray!65]
    (3.30,0) -- (3.30,2.95);
\draw
    (3.30,.07) -- (3.30,-.07)
    node[below=3pt,font=\scriptsize] {$c_4$};

\draw[densely dotted,mfgray!65]
    (4.16,0) -- (4.16,2.95);
\draw
    (4.16,.07) -- (4.16,-.07)
    node[below=3pt,font=\scriptsize] {$c_5$};

\draw[mfBranchOne,line width=4.2pt]
    (0.72,0) -- (1.58,0);

\draw[mfBranchTwo,line width=4.2pt]
    (1.58,0) -- (2.44,0);

\draw[mfBranchThree,line width=4.2pt]
    (2.44,0) -- (3.30,0);

\draw[mfBranchFour,line width=4.2pt]
    (3.30,0) -- (4.16,0);

\node[mfBranchOne,font=\scriptsize]
    at (1.15,-0.42) {$J_1$};

\node[mfBranchTwo,font=\scriptsize]
    at (2.01,-0.42) {$J_2$};

\node[mfBranchThree,font=\scriptsize]
    at (2.87,-0.42) {$J_3$};

\node[mfBranchFour,font=\scriptsize]
    at (3.73,-0.42) {$J_4$};

\draw[
very thick]
    (0.72,1.00)
    --
    (1.58,2.95)
    --
    (2.44,1.00)
    --
    (3.30,2.95)
    --
    (4.16,1.00);

\node[
    mfgray,
    rotate=66,
    font=\scriptsize
]
    at (1.02,2.18)
    {$+\eta/2$};

\node[
    mfgray,
    rotate=-66,
    font=\scriptsize
]
    at (2.1,2.18)
    {$-\eta/2$};

\node[
    mfgray,
    rotate=66,
    font=\scriptsize
]
    at (2.75,2.18)
    {$+\eta/2$};

\node[
    mfgray,
    rotate=-66,
    font=\scriptsize
]
    at (3.84,2.18)
    {$-\eta/2$};

\draw[
line width=2pt]
    (-0.10,1.20) -- (-0.10,2.75);

\node[
    rotate=90
]
    at (-0.36,1.98)
    {$I'$};

\end{scope}

\begin{scope}[xshift=10.70cm]

\draw[->,thick]
    (0,0) -- (4.10,0)
    node[right] {$z_-$};

\draw[->,thick]
    (0,0) -- (0,4.42)
    node[above] {$z_+$};

\draw[dashed,mfgray!60]
    (0,0) -- (4.05,4.05);

\node[
    font=\scriptsize,
    rotate=45,
    anchor=north
]
    at (3.28,3.28)
    {$z_+=z_-$};

\fill[mfMargin]
    (0,0.55) --
    (0,1.05) --
    (3.10,4.15) --
    (3.60,4.15) -- cycle;

\fill[mfCOne]
    (0,1.05) --
    (0,1.55) --
    (2.60,4.15) --
    (3.10,4.15) -- cycle;

\fill[mfCTwo]
    (0,1.55) --
    (0,2.05) --
    (2.10,4.15) --
    (2.60,4.15) -- cycle;

\fill[mfCThree]
    (0,2.05) --
    (0,2.55) --
    (1.60,4.15) --
    (2.10,4.15) -- cycle;

\fill[mfCFour]
    (0,2.55) --
    (0,3.05) --
    (1.10,4.15) --
    (1.60,4.15) -- cycle;

\fill[mfMargin]
    (0,3.05) --
    (0,3.55) --
    (0.60,4.15) --
    (1.10,4.15) -- cycle;

\draw[densely dotted,mfgray!60,line width=.55pt]
    (-0.42,0.13) -- (0,0.55);

\draw[densely dotted,mfgray!60,line width=.55pt]
    (-0.42,0.63) -- (0,1.05);

\draw[densely dotted,mfgray!60,line width=.55pt]
    (-0.42,1.13) -- (0,1.55);

\draw[densely dotted,mfgray!60,line width=.55pt]
    (-0.42,1.63) -- (0,2.05);

\draw[densely dotted,mfgray!60,line width=.55pt]
    (-0.42,2.13) -- (0,2.55);

\draw[densely dotted,mfgray!60,line width=.55pt]
    (-0.42,2.63) -- (0,3.05);

\draw[densely dotted,mfgray!60,line width=.55pt]
    (-0.42,3.13) -- (0,3.55);

\draw[thick,mfgray!90]
    (0,0.55) -- (3.60,4.15);

\draw[mfgray,line width=.85pt]
    (0,1.05) -- (3.10,4.15);

\draw[mfgray,line width=.85pt]
    (0,1.55) -- (2.60,4.15);

\draw[mfgray,line width=.85pt]
    (0,2.05) -- (2.10,4.15);

\draw[mfgray,line width=.85pt]
    (0,2.55) -- (1.60,4.15);

\draw[mfgray,line width=.85pt]
    (0,3.05) -- (1.10,4.15);

\draw[thick,mfgray!90]
    (0,3.55) -- (0.60,4.15);

\node[
    font=\normalsize,
]
    at (2.125,4.5)
    {$\Phi(V_k) = Y_{I'}$};

\node[
    font=\scriptsize,
    anchor=east,
    fill=white,
    inner sep=.4pt
]
    at (-0.10,0.55)
    {$q$};

\node[
    font=\scriptsize,
    anchor=east,
    fill=white,
    inner sep=.4pt
]
    at (-0.10,1.05)
    {$c^\prime_1$};

\node[
    font=\scriptsize,
    anchor=east,
    fill=white,
    inner sep=.4pt
]
    at (-0.10,1.55)
    {$c^\prime_2$};

\node[
    font=\scriptsize,
    anchor=east,
    fill=white,
    inner sep=.4pt
]
    at (-0.10,2.05)
    {$c^\prime_3$};

\node[
    font=\scriptsize,
    anchor=east,
    fill=white,
    inner sep=.4pt
]
    at (-0.10,2.55)
    {$c^\prime_4$};

\node[
    font=\scriptsize,
    anchor=east,
    fill=white,
    inner sep=.4pt
]
    at (-0.10,3.05)
    {$c^\prime_5$};

\node[
    font=\scriptsize,
    anchor=east,
    fill=white,
    inner sep=.4pt
]
    at (-0.10,3.55)
    {$q + A$};

\node[
]
    at (0.65,1.95)
    {\footnotesize$W_1$};

\node[
]
    at (0.65,2.45)
    {\footnotesize$W_2$};

\node[
]
    at (0.65,2.95)
    {\footnotesize$W_3$};

\node[
]
    at (0.65,3.45)
    {\footnotesize$W_4$};

\end{scope}

\end{tikzpicture}%
}
}

\caption{Illustration of the multi-fold construction in Lemma~\ref{lemma:multi-fold}, shown for $m=5$. 
(Left) The relative-time interval $I = (a, b) \subset (0, \infty)$ in the coordinate $x = t_2-t_1$ lifts to the region $Y_I = \{(t_1, t_2) \in \R^2 \colon t_2 - t_1\in I\}$. 
The first-layer neurons introduce the causal boundaries $x = c_r$, for $r = 1, \ldots, 5$, creating four regions $V_1, \ldots, V_4$ within $Y_I$ corresponding to the intervals $J_k = (c_k, c_{k+1})$. 
(Middle) The relative output timing $h(x)$ forms a sawtooth. 
Each restriction $h\vert_{J_k}$ is affine and maps $J_k$ bijectively onto the same output interval $I^\prime$. 
(Right) Each region $V_k$ is mapped affinely and bijectively onto the common output region $Y_{I^\prime} = \{(z_-, z_+) \in \R^2 \colon z_-, z_+) \in I^\prime\}$. 
The next folding block creates regions $W_1, \ldots, W_4$ within this common image, with boundaries $z_+ - z_-= c_r^\prime$. 
This subdivision pulls back to each region~$V_k$. 
}
\label{fig:multi_fold}

\end{figure}
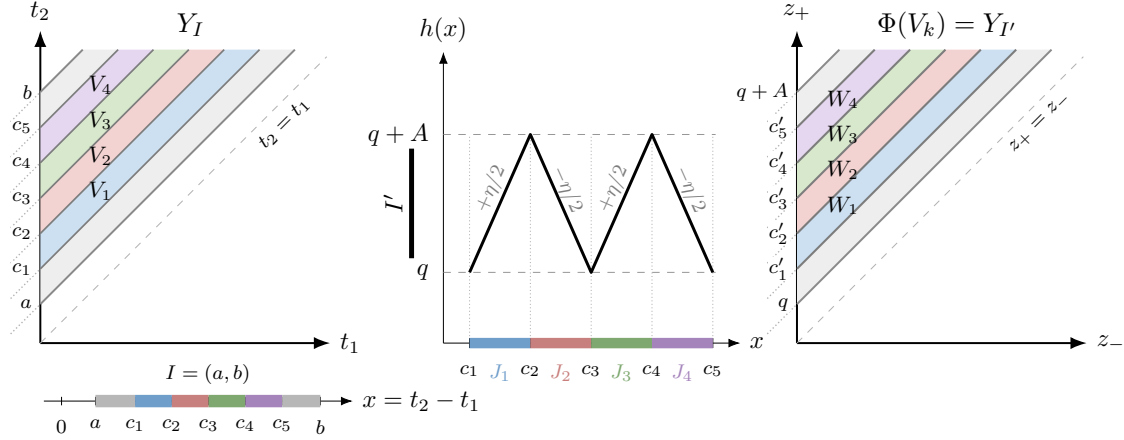

\paragraph{Idea of the multi-fold construction}  

The multi-fold construction is implemented by a two-layer SNN of architecture $(2, m, 2)$. 
Let $x = t_2 - t_1$ be the difference between the two input spike times. 
The first layer introduces $m$ ordered breakpoints 
$c_1 < \cdots < c_m$ 
along the $x$-axis. 
Crossing $c_r$ changes the causal set of neuron $r$, so the $m-1$ intervals $J_k = (c_k,c_{k+1})$ carry pairwise distinct causal patterns. 
The firing times of the first-layer neurons are piecewise-affine functions of $x$. 
The second layer combines these firing times through two output neurons, with firing times $z_-$ and $z_+$. 
We consider their relative firing time $h \coloneq z_+ - z_-$. 
In the construction below, this difference depends only on $x = t_2 - t_1$, and we therefore write it as $h(x)$. 
By choosing the difference between the two output weight vectors so that the slope of $h$ has constant magnitude and alternating sign on successive intervals $J_k$, we obtain a sawtooth function $h$. 
Each interval $J_k$ is mapped affinely and bijectively onto the same output interval $I^\prime$. 
Hence, the two-layer SNN block maps $m-1$ distinct causal regions onto a common two-dimensional output region. This output region has the same form as the original input region and can therefore serve as the input to another copy of the block. 
Figure~\ref{fig:multi_fold} illustrates the construction formalized in Lemma~\ref{lemma:multi-fold}.

\begin{lemma}\label{lemma:multi-fold}
    For any nonempty open interval $J \subseteq (0,\infty)$, let 
    \begin{equation*}
        Y_J \coloneq \Bigl\{(y_1, y_2) \in \R^2 \colon y_2 - y_1 \in J \Bigr\} . 
    \end{equation*}
    Let $m \geq 3$ and let $I = (a,b) \subset(0, \infty)$. 
    Consider a two-layer SNN with architecture $(2,m,2)$, 
    whose TTFS layer maps are $F^{(1)} \colon \R^2 \to \R^m$ and $F^{(2)} \colon \R^m \to \R^2$, and whose two-layer map is $\Phi \coloneq F^{(2)} \circ F^{(1)} \colon \R^2 \to \R^2$. 
    It is possible to choose the weights and thresholds such that the following holds. 
    There exist $m-1$ pairwise disjoint nonempty open sets $V_1, \ldots,V_{m-1} \subset Y_I$, and a nonempty open interval $I^\prime \subset (0, \infty)$ such that 
    \begin{enumerate}[(i)]
        \item the sets $V_1,\dots,V_{m-1}$ have pairwise distinct causal patterns in the first layer, and hence pairwise distinct causal patterns in the two-layer SNN; 
        
        \item for every $k = 1, \dots, m-1$, the restriction $\Phi \big|_{V_k} \colon V_k \to Y_{I^\prime}$ is an affine bijection.
    \end{enumerate}
\end{lemma}

\begin{proof}
    Let $
    c_1 < c_2 < \cdots < c_m
    $ be equally spaced points in the interval $I = (a,b)$,  
    with common spacing $c_{r+1} - c_r = \Delta > 0$. 
    They partition $(c_1, c_m)$ into $m-1$ consecutive intervals 
    \begin{equation*}
        J_k \coloneq (c_k, c_{k+1}), \quad \text{for  } k=1,\dots,m-1.
    \end{equation*}
    The first layer will be constructed so that crossing each point $c_r$ changes the causal set of a neuron, making the intervals $J_k$ correspond to distinct causal regions. 
    Let
    \begin{equation*}
        V_k \coloneq \bigl\{(t_1,t_2) \in \R^2 \colon t_2-t_1\in J_k \bigr\}.
    \end{equation*}
    Since $J_k \subset I \subset (0, \infty)$, every $(t_1, t_2) \in V_k$ satisfies $t_2>t_1$. Therefore, on $V_k$, only the causal sets $\{1\}$ and $\{1,2\}$ can occur.

    For the first layer, let neuron $r \in [m]$ have weights $w_{r1} = w_{r2} = 1/2$ and threshold $\theta_r = c_r/2$. 
    Writing $x = t_2 - t_1$, its firing time is 
    
    \begin{equation}
    \label{eq:multifold_txn}
        y_r = t_1 + \phi_r(x), \quad \text{where  } \phi_r(x) = \begin{cases}
            \dfrac{x+c_r}{2}, & x<c_r,\\[2mm]
            c_r, & x>c_r.
        \end{cases}
    \end{equation}
    Now fix $x \in J_k = (c_k, c_{k+1})$. Since the  
    $c_k$ are ordered, we have 
    \begin{equation*}
        x > c_r \text{ for } r\leq k, \text{ and } x < c_r \text{ for } r\geq k+1.
    \end{equation*}
    Hence, the first $k$ neurons have causal set $\{1\}$, while the remaining $m-k$ neurons have causal set $\{1,2\}$. The first-layer causal pattern on $V_k$ is 
    \begin{equation}
    \label{eq:multifold_pattern}
        C_k = \Bigl(\underbrace{\{1\},\dots,\{1\}}_{k}, \underbrace{\{1,2\},\dots,\{1,2\}}_{m-k} \Bigr).
    \end{equation}
    For different $k$ these patterns are pairwise distinct. The first layer has therefore created $m-1$ distinct causal pieces in the region $Y_I \subset \{t_2 > t_1\}$. 

     We now construct the second layer which has two output neurons with firing times denoted by $z_-$ and $z_+$. 
     Our goal is to choose the second-layer parameters so that, on every causal piece $V_k$, the relative output timing 
    \begin{equation}\label{eq:hx}
        h(x) \coloneq z_+(x) - z_-(x)
    \end{equation}
    ranges bijectively over the same interval $I^\prime$. 
    This provides the folding over the relative time coordinate.
    We verify at the end of the proof that the full two-dimensional map $\Phi|_{V_k}$ is an affine bijection from $V_k$ onto $Y_{I^\prime}$.

    For this, we choose the weights $\alpha_1^{\pm}, \dots, \alpha_m^{\pm} > 0$ 
    such that $\sum_{r=1}^{m} \alpha_r^{\pm} = 1$, and choose the thresholds $\theta^{\pm}$ sufficiently large such that both second layer neurons fire after all $m$ first layer neurons, that is, their causal sets are both $\{1, \dots, m\}$. 
    Then, 
    \begin{equation}
        z_{\pm} = \theta^{\pm} + \sum_{r=1}^{m}\alpha_r^{\pm}y_r.
        \label{eq:z+_}
    \end{equation}
    We control the relative output time $h(x) = z_+(x) - z_-(x)$ in \eqref{eq:hx}. 
    Using \eqref{eq:multifold_txn}, we have 
    \begin{equation}
    \label{eq:multifold_h}
        h(x) = z_+(x) - z_-(x) = \theta^+ - \theta^- + \sum_{r=1}^{m}\gamma_r \phi_r(x),
    \end{equation}
    where $\gamma_r \coloneq \alpha^+_r - \alpha^-_r$. Since $\sum_{r=1}^{m}\alpha_r^{\pm} = 1$, we have $\sum_{r=1}^{m}\gamma_r = 0$. 

    We now determine the coefficients $\gamma_r$ so that the relative output time $h$ folds all intervals $J_k$ onto the same output interval. 
    Fix $k \in \{1, \dots, m-1\}$. On $J_k = (c_k, c_{k+1})$, we have $x>c_r$ for $r\leq k$ and $x<c_r$ for $r\geq k+1$. Thus, by \eqref{eq:multifold_txn}, the derivative of $\phi_r(x)$ satisfies
    \begin{equation*}
        \phi_r^\prime(x) = \begin{cases}
            0, & r\leq k, \\[1mm]
            \dfrac{1}{2}, &r\geq k+1. 
        \end{cases} 
    \end{equation*}
    and hence by \eqref{eq:multifold_h} the derivative of $h(x)$ satisfies
  \begin{equation}
  \label{eq:multifold_slope}
        h^\prime(x) = \frac{1}{2}\sum_{r=k+1}^{m} \gamma_r, \quad \text{for  } x \in J_k. 
    \end{equation}
    
    Equation~\eqref{eq:multifold_slope} determines the slope of $h$ on $J_k$, which is described by the corresponding tail sum of the coefficients $\gamma_r$. 
    Since all intervals $J_k$ have the same length $\Delta$, we can make every interval map onto the same output interval by choosing the slopes to have same magnitude but alternating signs. Then, $h$ increases by the same amount on $J_1$, decreases by the same amount on $J_2$, increases again on $J_3$, and so on. 

    Fix $0 < \eta < 1/m$. We choose the tail sums to have magnitude $\eta$, so that, by \eqref{eq:multifold_slope}, the slope of $h$ on each interval $J_k$ has magnitude $\eta/2$, with alternating sign. 
    We impose 
    \begin{equation}\label{eq:multifold_tail}
        \sum_{r=k+1}^{m} \gamma_r = (-1)^{k+1}\eta, \text{ for } k = 1, \dots, m-1.
    \end{equation}
    Subtracting two consecutive tail sums gives
    \begin{equation*}
        \gamma_r = 2(-1)^r\eta, \text{ for } 2\leq r\leq m-1,
    \end{equation*}
    while the last condition gives $\gamma_m = (-1)^m\eta$. 
    Finally, since $\sum_{r=1}^{m} \gamma_r = 0$, we obtain
    \begin{equation}\label{eq:multifold_gamma}
        \gamma_1 = -\eta, \quad \gamma_r = 2(-1)^r\eta, \quad 2\leq r\leq m-1, \quad \gamma_m = (-1)^m\eta.
    \end{equation}
    We now realize these coefficients $\gamma_r$ as differences of two weight vectors by setting 
    \begin{equation}\label{eq:multifold_alpha}
        \alpha_r^\pm = \frac{1}{m} \pm \frac{\gamma_r}{2}, \text{ for } r = 1, \dots, m.
    \end{equation}
    Then, it can readily be seen that $\gamma_r = \alpha_r^+ - \alpha_r^-$, and since $\sum_r \gamma_r = 0$, we get $\sum_{r=1}^{m} \alpha_r^\pm = 1$. Moreover, 
    \begin{equation*}
        \alpha_r^{\pm} \geq \frac{1}{m} - \frac{\lvert \gamma_r \rvert}{2} \geq \frac{1}{m} - \eta > 0, 
    \end{equation*}
    all second layer weights are positive.

    Once the coefficients $\gamma_r$ are fixed, the slopes of $h$ on the intervals $J_k$ are determined. The remaining freedom in \eqref{eq:multifold_h} is the constant term $\theta^+ - \theta^-$, which only shifts the sawtooth vertically. 

    Fix any $q > 0$, which will determine the lower level of the sawtooth, an define
    
    \begin{equation*}
        \beta \coloneq \sum_{r=1}^{m}\gamma_r \phi_r(c_1) 
        = \gamma_1 c_1 + \frac{1}{2}\sum_{r=2}^{m}\gamma_r(c_1 + c_r)
        = \frac{1}{2}\sum_{r=1}^{m}\gamma_rc_r, 
    \end{equation*}
    where the last equality follows from $\sum_r\gamma_r = 0$. Since,
    \begin{equation*}
        h(c_1) = \theta^+ - \theta^- + \beta,
    \end{equation*} 
    to enforce $h(c_1) = q$, it suffices to define $\delta \coloneq q - \beta$ and choose $\theta^+ - \theta^- = \delta$.

    Finally, to ensure that both output neurons fire after all first-layer spikes, we add a sufficiently large common offset. Specifically, set

    \begin{equation}
        M \coloneq c_m + 1 + \lvert \delta \rvert, \quad \theta^- \coloneq M, \quad \theta^+ \coloneq M + \delta.  
    \end{equation}
    Then, $\theta^+ - \theta^- = \delta = q-\beta$, and $\theta^{\pm} > c_m > 0$. Hence, 
    \begin{equation*}
        h(c_1) = \theta^+ - \theta^- + \beta = q. 
    \end{equation*}

\medskip 
    
    We now describe the fold explicitly. 
    By~\eqref{eq:multifold_slope} and~\eqref{eq:multifold_tail},
    \begin{equation}\label{eq:multifold_alternating}
        h^\prime(x) = \frac{\eta}{2}(-1)^{k+1}, \text{ for } x\in J_k.
    \end{equation}
    Thus, $h$ is increasing on $J_1$, decreasing on $J_2$, increasing on $J_3$, and so on. Since every interval $J_k$ has length $\Delta$, the change of $h$ across $J_k$ is
    \begin{equation*}
        h(c_{k+1}) - h(c_k) = (-1)^{k+1}A,
    \end{equation*}
    where $A \coloneq \frac{\eta\Delta}{2} > 0$. Starting from $h(c_1) = q$, it follows that
    \begin{equation}\label{eq:multifold_breakpoint_values}
        h(c_k) = \begin{cases}
        q, & k \text{ odd},\\[1mm]
        q+A, & k \text{ even}. \end{cases}
    \end{equation}
    Equivalently, on each interval $J_k$, we can also write $h$ as
    \begin{equation}\label{eq:multifold_piecewise_h}
        h(x) = \begin{cases}
        q + \frac{\eta}{2}(x - c_k), & x\in J_k, \quad k \text{ odd}, \\[3mm]
        q + A - \frac{\eta}{2}(x - c_k), & x\in J_k, \quad k \text{ even}.
        \end{cases}
    \end{equation}
    Hence, with $I^\prime \coloneq (q, q+A)$, every restriction
    \begin{equation*}
        h|_{J_k} \colon J_k \to I^\prime
    \end{equation*}
    is an affine bijection. 
    Thus, every affine piece of the sawtooth covers the same output interval $I^\prime$.

\medskip 

    It remains to verify that the full two-dimensional map is affine and bijective on each $V_k$. 
    On $J_k$, using \eqref{eq:multifold_txn} in \eqref{eq:z+_} and $\sum_{r}\alpha_r^- = 1$ gives
    \begin{equation*}
        z_- = t_1 + \theta^- + \sum_{r=1}^{m}\alpha_r^-\phi_r(x) = t_1 + \psi_k(x) 
    \end{equation*}
    where
    \begin{equation*}
        \psi_k(x) = \theta^- + \sum_{r=1}^{k}\alpha_r^-c_r + \frac{1}{2}\sum_{r=k+1}^{m}\alpha_r^-(x+c_r)
    \end{equation*}
    is affine in $x$. 
    
    Since by definition, $z_+(x) = z_-(x) + h(x)$, we may equivalently use $(z_-, h)$ as output coordinates. 
    On $V_k$, 
    we therefore consider the map
    \begin{equation}
    \label{eq:multifold_branch_map}
        (t_1, x) \mapsto (t_1 + \psi_k(x), h(x)),
    \end{equation}
    whose linear part is $\begin{pmatrix} 1 & \psi_k^\prime \\ 0 & h^\prime\end{pmatrix}$, and the determinant of this matrix is $h^\prime = \frac{\eta}{2}(-1)^{k+1} \neq 0$. 
    Thus, the affine map is injective on $V_k$. 
    Moreover, $h(J_k) = I^\prime$, and for each fixed $x\in J_k$, the coordinate $z_- = t_1 + \psi_k(x)$ ranges over all of $\R$ as $t_1$ ranges over $\R$. Hence, its image is exactly $Y_{I^\prime}$. Therefore, every point of $Y_{I^\prime}$ has exactly one preimage in $V_k$, so, $\Phi|_{V_k} \colon V_k \to Y_{I^\prime}$ is an affine bijection.
\end{proof}

The following example illustrates the construction for $m=5$; see Figure~\ref{fig:multi_fold} for the corresponding geometry.

\begin{example}[The case $m=5$.] 
    Consider the construction in Lemma~\ref{lemma:multi-fold} with $m=5$. Choose five equally spaced points $0 < a < c_1 < \dots < c_5 < b$ inside the interval $I = (a, b)$, and let $J_k = (c_k, c_{k+1})$ for $k = 1,\dots, 4$. 
    
    Recall that $Y_I = \bigl\{(t_1,t_2) \in \R^2 \colon t_2 - t_1 \in I
    \bigr\}$ is the two dimensional region between the parallel lines $t_2 - t_1 = a$ and $t_2 - t_1 = b$; see Figure~\ref{fig:multi_fold}. The intervals $J_1, \dots, J_4$ correspond to the four regions $V_k = \bigl\{(t_1, t_2) \in \R^2 \colon t_2 - t_1 \in J_k \bigr\}$ for $k = 1, \dots, 4$.

    The first layer has five neurons. Neuron $r$ contributes the breakpoint $c_r$, that is, crossing the line $x = t_2 - t_1 = c_r$ changes the causal set of neuron $r$. Hence, the four regions $V_1, \dots, V_4$ carry four distinct first-layer causal patterns. 

    The second layer combines the first-layer firing times into two output spike times $z_-$ and $z_+$. Their relative firing time $h(x) = z_+(x) - z_-(x)$ forms a sawtooth with one affine piece on each $J_k$, and each restriction $h\vert_{J_k} \colon J_k\to I^\prime$ is an affine bijection onto the same interval $I^\prime$. 

    Therefore, for every $k = 1,\dots, 4$, the full two-layer map restricts to an affine bijection $\Phi\vert_{V_k} \colon V_k\to Y_{I^\prime}$, where $Y_{I^\prime} = \bigl\{(z_-, z_+) \in \R^2 \colon z_+ - z_- \in I^\prime
    \bigr\}$. Thus, the four distinct causal regions $V_1, \dots, V_4$ are all folded onto the same output region $Y_{I^\prime}$. Applying the same construction once more partitions $Y_{I^\prime}$ into four new regions $W_1, \dots, W_4$. Since each map $\Phi\vert_{V_k}$ is bijective, every $W_j$ has one preimage inside each $V_k$. Consequently, after two blocks one obtains $16$ distinct causal regions.
\end{example} 

The example illustrates the general recursive mechanism. Lemma~\ref{lemma:multi-fold} shows that a two-layer SNN can map $m-1$ distinct causal regions $V_1,\ldots, V_{m-1}$ contained in $Y_I$ affinely and bijectively onto a common output set $Y_{I^\prime}$, 
and this output set has the same form as the original input set $Y_I$. 
Consequently, the construction can be iterated.  
Suppose that a second block creates $m-1$ causal regions $W_1, \dots, W_{m-1}$ within $Y_{I^\prime}$. Then, for each preceding region $V_k$, the preimages 
\begin{equation*}
    (\Phi|_{V_k})^{-1}(W_1), \dots, (\Phi|_{V_k})^{-1}(W_{m-1})
\end{equation*} 
are nonempty and pairwise disjoint. 
Thus, each existing causal region gives rise to $m-1$ new causal regions. 
The first layer of each block creates multiplicity and the second layer resets the resulting regions onto a common reusable set. 
Iterating the two-layer block yields the following width-dependent exponential lower bound.

\begin{theorem}
\label{thm:multifold_depth}
    Let $L\geq 1$ and $m_1, \dots, m_L \geq 2$. 
    Consider a depth-$2L$ SNN with input dimension $N_0\geq 2$ and 
    widths satisfying
    $$
        N_{2\ell-1} \geq m_\ell, \quad N_{2\ell} \geq 2, \quad \ell=1,\ldots, L. 
    $$ 
    Then 
    \begin{equation}
    \label{eq:lbdepth}
        R_{\max}
        (N_0, N_1, \ldots, N_{2L})
        \geq \prod_{\ell = 1}^{L} (m_{\ell} - 1).
    \end{equation}
\end{theorem}

\begin{proof} 
    We consider the subnetwork with widths $(2,m_1,2,\ldots, m_L,2)$. Any network satisfying the stated width conditions can reproduce this by choosing the additional neurons so that they remain causally inactive on the sets used in the construction, for example by assigning them sufficiently large thresholds.

    Let $Y_{I_0}$ be an initial input set of the form considered in Lemma~\ref{lemma:multi-fold}. Apply the lemma to the first two-layer block, whose intermediate width is $m_1$. 
    This produces $m_1 - 1$ pairwise distinct causal regions, each of which is mapped affinely and bijectively onto the same output set $Y_{I_1}$. 

    Since $Y_{I_1}$ has the same form as $Y_{I_0}$, we may apply Lemma~\ref{lemma:multi-fold} again, now to a block with intermediate width $m_2$. 
    On each of the $m_1 - 1$ existing regions, the second block creates $m_2-1$ new causal regions, and maps all of them onto a common output set $Y_{I_2}$. 
    The first two blocks therefore produce $(m_1 - 1)(m_2 - 1)$ distinct network-level causal regions. 

    Continuing inductively, after the first $\ell$ blocks the network has $\prod_{k = 1}^{\ell} (m_k - 1)$ distinct causal regions. After $L$ blocks, this gives \eqref{eq:lbdepth}. 
\end{proof}

Theorem~\ref{thm:multifold_depth} provides a basic lower bound that grows with width and depth. 
The construction exploits only a single relative spike-time coordinate and can be strengthened in several ways. In particular, several copies of the multi-fold construction can be run in parallel on disjoint pairs of spike times. If nonnegative weights were allowed, the cross-pair weights could be set to zero, so that the folds decouple exactly.
If, in block $\ell$, $r_{\ell,j}$ neurons are assigned to the $j$th of $p$ pairs, then the resulting Cartesian-product construction yields $\prod_{j=1}^p(r_{\ell,j}-1)$ regions per block. Under the strictly positive weight setting considered here, however, this exact decoupling may no longer be possible, since every neuron necessarily receives a nonzero contribution from the other input pairs. It may nevertheless be possible to develop a robust version of the construction, in which the zero-cross pair weights are replaced by sufficiently small positive weights while preserving the folding behavior on suitable subregions.

Another possible improvement is to exploit both ordering regions $t_1 < t_2$ and $t_2 < t_1$, in contrast to the present construction which uses only $x = t_2 - t_1 > 0$, and maps each fold back into an interval $I^\prime \subset (0, \infty)$. Using $x < 0$ side as well could therefore produce additional regions, but would require a modified reusable output set and a folding construction that can be iterated across both spike time orderings. We leave these possible extensions of the folding construction for future work. 

Finally, the last layer need not have width two: it may instead be replaced by a width-one layer or an affine layer, since the output of the final block need not be mapped onto a reusable set for further iteration.

\subsection{Deep SNNs with shared weights} 
\label{subsec:upperboundsharedweight}

We now consider the shared-weights regime and derive sharper upper bounds for this regime. 
More precisely, we assume that within each layer $\ell$, all neurons share the same positive weight vector. We assume without loss of generality that within each layer the thresholds are strictly ordered.

By Lemma~\ref{lem:nested_same_weights}, shared weights and ordered thresholds induce a corresponding ordering of the firing times; 
for instance, $t^{(\ell)}_1 < t^{(\ell)}_2 < \cdots < t^{(\ell)}_{N_\ell}$.  
For deep SNNs, the key simplification is that, on any region of the input space where the presynaptic spike times entering a layer have a fixed strict ordering, 
every causal set in the next layer is a prefix of that ordering. 
More precisely, let $U \subseteq \R^d$ be a region of the input space on which the spike times in layer $\ell-1$ satisfy, 
for some permutation $\pi$ of $[N_{\ell-1}]$, 
\begin{equation*}
    t^{(\ell-1)}_{\pi(1)} < t^{(\ell-1)}_{\pi(2)} < \cdots < t^{(\ell-1)}_{\pi(N_{\ell-1})}, \quad \text{for every  } \vec t^{(0)} \in U.
\end{equation*}
Then, for every $\vec t^{(0)} \in U$, the causal set of any neuron $r$ in layer $\ell$ is necessarily of the form 
\begin{equation*}
    S_r^{(\ell)} = \{\pi(1), \dots, \pi(q_{l,r})\}
\end{equation*}
for some $q_{\ell,r}\in[N_{\ell-1}]$. We call $q_{\ell,r}$ the \emph{prefix length} of neuron $(\ell, r)$. 

For example, suppose $N_{\ell-1} = 3$ and that, on some region $U$, we have $t^{(\ell-1)}_2 < t^{(\ell-1)}_1 < t^{(\ell-1)}_3$. 
Then the possible causal sets of a neuron in layer $\ell$ are 
$\{2\}$, $\{2,1\}$, and $\{2,1,3\}$, with prefix lengths $1$, $2$ and $3$, respectively. 

Moreover, if we index the neurons such that the thresholds increase with the neuron index, then the prefix lengths are increasing with the neuron index. 
Hence, at each fixed input $\vec t^{(0)}$, the prefix lengths satisfy 
\begin{equation*}
    1\leq q_{\ell,1} \leq q_{\ell,2} \leq \cdots
    \leq q_{\ell,N_\ell}\leq N_{\ell-1}.
\end{equation*} 

We collect these observations in the following Lemma, whose proof is given in Appendix~\ref{appendix:deepSNN}. 

\begin{lemma}
\label{lem:prefix_lengths_deepnetwork}
Assume the shared-weights and distinct threshold setting. 
Fix a layer $\ell \geq 2$ and suppose the neurons are indexed such that $0<\theta_1<\theta_2<\cdots<\theta_{N_\ell}$. 
Let $U \subseteq \R^d$ be an input region on which the presynaptic spike times from layer $\ell-1$ have a fixed strict ordering 
\begin{equation*}
    t^{(\ell-1)}_{\pi(1)} < t^{(\ell-1)}_{\pi(2)} < \cdots < t^{(\ell-1)}_{\pi(N_{\ell-1})}
\end{equation*}
for some permutation $\pi$ of $[N_{\ell-1}]$. 
Then, for every $\vec t^{(0)} \in U$ and every neuron $(\ell,r)$, there exists 
a prefix length $q_{\ell,r} \in [N_{\ell-1}]$ such that 
\begin{equation*}
    S^{(\ell)}_r= 
    \{\pi(1),\ldots, \pi(q_{\ell,r})\}. 
\end{equation*}
Moreover, the prefix lengths are nondecreasing across neurons, so that 
\begin{equation*}
    1\le q_{\ell,1} \leq q_{\ell,2}\leq \cdots \leq q_{\ell,N_\ell} \leq N_{\ell-1}. 
\end{equation*}
\end{lemma} 

Lemma~\ref{lem:prefix_lengths_deepnetwork} allows us bound the number of causal regions via a combinatorial argument.  
On any region with a fixed ordering of the presynaptic spike times, the layer-$\ell$ causal-set tuple at each input is uniquely encoded by a weakly increasing sequence of prefix lengths $(q_{\ell,1}, \ldots, q_{\ell,N_\ell}) \in [N_{\ell-1}]^{N_\ell}$. 
These weakly increasing sequences form the basic counting objects that we will use in the upper bound below. 
A detailed example illustrating this correspondence is given in Appendix~\ref{appendix:deepSNN}. 

\begin{theorem}
\label{thm:deep_recursive_upper_bound}
Let $\Phi$ be a feedforward SNN with architecture $(d, N_1, \ldots, N_L)$. Suppose that, for every layer $\ell=1, \ldots, L$, all neurons in layer $\ell$ share the same positive presynaptic weight vector in $\R^{N_{\ell-1}}_{> 0}$, and their thresholds satisfy $\theta_{\ell,1} < \theta_{\ell,2} < \cdots < \theta_{\ell,N_\ell}$. 
Then the number of distinct global causal patterns realized by $\Phi$ is bounded by 
\begin{equation}
\label{eq:deep_recursive_upper_bound}
   R_{\max}^{\rm shared}(d, N_1, \dots, N_L) 
\leq C(d,N_1)\cdot \prod_{\ell=2}^{L}\binom{N_{\ell-1}+N_\ell-1}{N_\ell},
\end{equation}
where $C(d,N_1) = (N_1+1)^d - N_1^d$ is the 
region-count upper bound for the first layer. 
\end{theorem}

\begin{proof}
By Proposition~\ref{prop:chain_bound_general_dm}, the first-layer causal-set tuple $(S^{(1)}_1,\dots,S^{(1)}_{N_1})$ has at most $C(d,N_1)$ possible values. 
Now fix a layer $\ell \geq 2$. 
Since the neurons in layer $\ell-1$ share the same positive weight vector and have strictly ordered thresholds, Lemma~\ref{lem:nested_same_weights} gives $t_1^{(\ell-1)}<t_2^{(\ell-1)}<\cdots<t_{N_{\ell-1}}^{(\ell-1)}$ for every input. 
Hence, by Lemma~\ref{lem:prefix_lengths_deepnetwork}, every causal-set tuple in layer $\ell$ is uniquely determined by a weakly increasing sequence of prefix lengths $
1
\leq q_{\ell,1}
\leq \cdots
\leq q_{\ell,N_\ell}
\leq N_{\ell-1}$.
The number of such sequences is
$$
\binom{N_{\ell-1}+N_\ell-1}{N_\ell}. 
$$
Therefore, for each causal pattern realized through layer $\ell-1$, there are at most $\binom{N_{\ell-1}+N_\ell-1}{N_\ell}$ possible extensions to layer $\ell$. 
Multiplying these bounds recursively over layers $\ell = 2, \ldots, L$, together with the first-layer bound $C(d,N_1)$, gives 
\eqref{eq:deep_recursive_upper_bound}. 
\end{proof}

\begin{remark}\label{remark:low_rank_deepSNN_ub}
    The product bound in \eqref{eq:deep_recursive_upper_bound} counts all admissible prefix-length tuples layer by layer. 
    In deep networks, however, not all such combinations need be realizable on a fixed input region. 
    The reason is that, on a given region of the input space, the output of a layer may lie in a lower-dimensional affine subset of $\R^{N_\ell}$. 

    A simple source of this dimension loss is repetition of prefix lengths. 
    Indeed, suppose that on some region, the prefix-length tuple of a layer is 
    \begin{equation*}
        \vec q = (q_1, \dots, q_N), \quad q_1 \leq \cdots \leq q_N,
    \end{equation*}
    and that $\vec q$ contains only $s$ distinct prefix lengths. Neurons having the same prefix length use the same presynaptic causal set, and hence, their firing times differ only by a constant. Therefore, the image of the corresponding affine layer map is contained in an affine subspace of dimension at most $s$. 

    Thus, the set of outputs of the layer forms a lower-dimensional family rather than a full $N_\ell$-dimensional one, and hence the next layer may not be able to realize all weakly increasing prefix-length tuples counted in the combinatorial upper bound. 
\end{remark}

For shared weights networks, it is not clear whether it is possible to obtain a lower bound that grows exponentially in the depth of the network. Lemma~\ref{lemma:shared-weight-obstruction} shows an obstruction to the folding construction used above. In the case of shared-weights, the $(2,m,2)$ construction has a monotone relative output time, so the alternating sawtooth needed for the multi-fold cannot be obtained by this construction. This does not rule out other mechanisms for
obtaining depth-dependent lower bounds in the shared weight setting. We discuss an alternative route in the interpretation below.

\medskip

\subsection{Interpretation} 

The folding construction used above is inspired by the folding constructions for ReLU networks in  \cite{montufar2014linearregions,montufar2017notes,serra2018regions}, which exploit the principle that successive layers can repeatedly fold different parts of the input space onto a common region. 
The construction above shows that an analogous mechanism is available in SNNs. The construction is implemented using blocks of two layers. The first layer creates many distinct causal pieces, while the second combines their firing times so that these pieces are mapped onto the same reusable spike-time region. 
There is a distinction with the ReLU folding constructions. In the constructions of \cite{montufar2014linearregions, montufar2017notes, serra2018regions}, a collection of ReLU units creates multiple pieces, while an affine combination of their activations folds these pieces onto a common region. This affine combination can be absorbed into the preactivations of the following layer. 
Moreover, several such one-dimensional folds can be applied in parallel along different input coordinates. Their Cartesian product then maps multiple full-dimensional input regions onto a common full-dimensional output region. 
The use of two-layer blocks in our construction for SNNs should not be interpreted as indicating that two layers are necessary for width-dependent folding. A potentially stronger construction could use a single layer to perform both tasks, creating multiple causal regions while simultaneously mapping them onto a common reusable output region. Such a construction is more constrained, since the same neuron parameters would have to determine both the causal subdivision and the coincidence of the corresponding vector-valued affine images. 

These observations suggest several directions for strengthening the SNN construction. One possibility is to run multiple independent multi-folds in parallel along different relative spike-time coordinates, producing a Cartesian-product folding construction. More generally, one could seek intrinsically higher-dimensional folds. Either mechanism could lead to stronger lower bounds.

For shared weight networks, another possible route for obtaining lower bounds with exponential growth in depth would be to strengthen one layer attainability result in Lemma~\ref{lemma:single_layer_local_braid_cone_attainability}. Lemma~\ref{lemma:single_layer_local_braid_cone_attainability} shows that the relevant causal patterns are attainable somewhere in the presynaptic spike time space, but this is not sufficient for iteration; the next layer only sees the particular affine image produced by the preceding layers. Thus, a recursive lower bound argument would require a more uniform attainability statement, for example, that the desired causal patterns can be realized on every arbitrary nonempty open set, or more generally on the affine images produced by the preceding layers. Remark~\ref{remark:low_rank_deepSNN_ub} complements this observation by showing that such images may be lower dimensional when prefix lengths are repeated.

There is an analogous loss of geometric freedom for ReLU networks with shared weights. In a shallow ReLU network, sharing the weight vector while varying only the biases produces parallel hyperplanes. Thus, $m$ such parallel hyperplanes divide the space into at most $m+1$ regions, independently of the input dimension. For deep shared weight ReLU networks, this restriction persists through depth. Indeed, after the first layer, all activations depend on the input only through the scalar coordinate $s = \langle \vec w, \vec x \rangle$, where $\vec w$ is the shared weight vector. Hence, the entire network factors through a one-dimensional input. 
Although depth can generate exponentially many one-dimensional linear pieces, 
weight sharing prevents the network from exploiting several independent input directions. Thus, in both SNN and ReLU networks, weight sharing restricts the geometry available for generating regions, although the resulting partitions have different structures.

\section{Experiments}
\label{section:experiments}

The goal of our experiments is to estimate the number of linear regions the input space is partitioned into by randomly initialized SNNs. 
We focus on three main questions. 
First, we study the effect of width and depth. 
Second, we compare SNN and ReLU networks under matched width and depth configurations. 
Third, we examine the effect of restricting the SNN weights by comparing networks with arbitrary positive
weights and networks where all neurons in a layer use the same incoming weight vector. The code and implementation for the experiments can be found on GitHub at this \href{https://github.com/manjot4/ttfs-linear-regions/}{link}.

\subsection{Experimental settings}

\paragraph{Models} An SNN consists of an affine linear encoder that maps inputs to spike times, followed by one or more SNN layers, and an affine decoder that produces class logits. 
The ReLU baseline is chosen to match the corresponding SNN in encoder size, hidden widths, depth, and decoder size. We report the average results over five random seeds. Full implementation and initialization details are given in Appendix~\ref{appendix:experiments}.

\paragraph{Width and depth configurations} 
For the width experiment, we use a single SNN layer and vary the width over $\{16, 32, 64, 128, 256, 512\}$. For the depth experiments, we consider networks of depths up to eight. 
We consider networks with up to eight layers and compare constant-width networks with architectures whose width decreases progressively through depth, denoted \emph{Gradual}, with widths $[256,256,256,128,128,128,64,64]$, and \emph{More gradual}, with widths $[512,256,256,128,128,64,64,64]$. 
These choices allow us to examine the trade-off between maintaining a large width and adding more layers.

\paragraph{Estimating region complexity} We estimate region complexity using a trajectory-based method. 
We randomly sample pairs of inputs $(x_0, x_1)$ from the dataset, interpolate along the line segment $x(t) = (1-t)x_0 + t x_1$ for $t \in [0,1]$, and track distinct activation patterns along the trajectory, using $20,\!000$ uniformly spaced intervals. 
For an SNN, a region is identified by the causal pattern, that is, a binary pattern indicating which input neurons belong to the causal prefix of each non-input neuron. 
Each unique causal pattern corresponds to a distinct linear region. Along a trajectory, we count the number of distinct causal patterns encountered. 
For ReLU networks, a region is defined by the binary activation pattern of all ReLU units. 
For each model and configuration, we report the average number of unique regions encountered per trajectory. 
The region statistics are averaged over many trajectories and random seeds. 
We complement this with exact enumeration of regions over two-dimensional slices of the input space in selected experiments.

\paragraph{Data sets} 
We estimate region complexity along straight line interpolations between randomly selected pairs of CIFAR-10 \cite{krizhevsky2009learning} and MNIST \cite{deng2012mnist} samples.  
Unless otherwise stated, all region statistics reported in this section are measured at initialization, without training. 
Additional experiments examining region complexity during training, comparisons between positive and arbitrary weight SNNs, and initialization ablations are provided in Appendix~\ref{appendix:experiments}.

\subsection{Experimental results}

\paragraph{Effect of width and depth}
Figure~\ref{fig:cifar_mnist_init1} shows the CIFAR-10 and MNIST results, respectively. 
For both datasets, the number of regions encountered along a trajectory increases rapidly with the width. 
The effect of depth is more architecture dependent. At smaller widths, increasing the depth has little effect on the observed region count, whereas at larger widths, deeper networks produce more regions. 
One possible explanation is that sufficiently wide layers preserve a higher-dimensional affine image across depth, allowing subsequent layers to introduce additional subdivisions, while narrow layers may create an early dimensional bottleneck that limits further region growth.

\begin{figure}[h]
    \centering

\begin{tabular}{
        @{}
        >{\centering\arraybackslash}m{0.01\textwidth}
        >{\centering\arraybackslash}m{0.77\textwidth}
        @{}
    }
        \rotatebox{90}{MNIST}
        &
        \includegraphics[width=\linewidth]{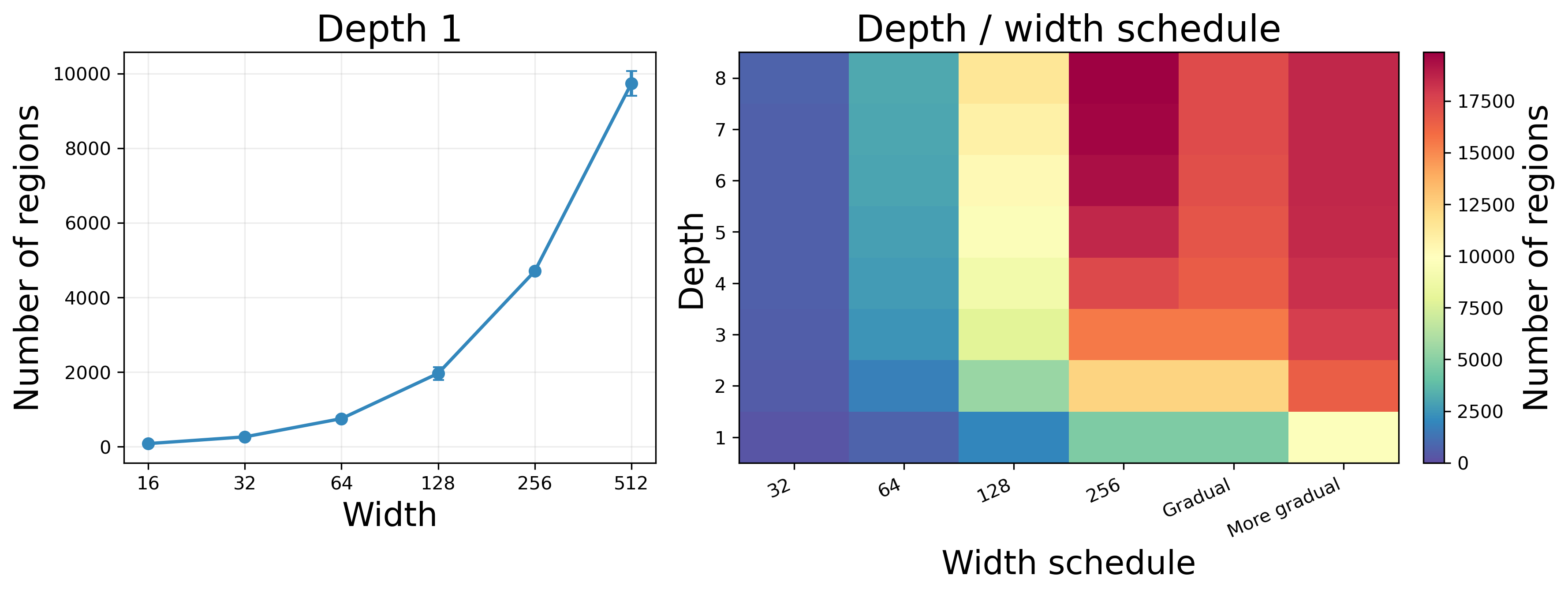}
        \\ 
        \rotatebox{90}{CIFAR-10}
        &
        \includegraphics[width=\linewidth]{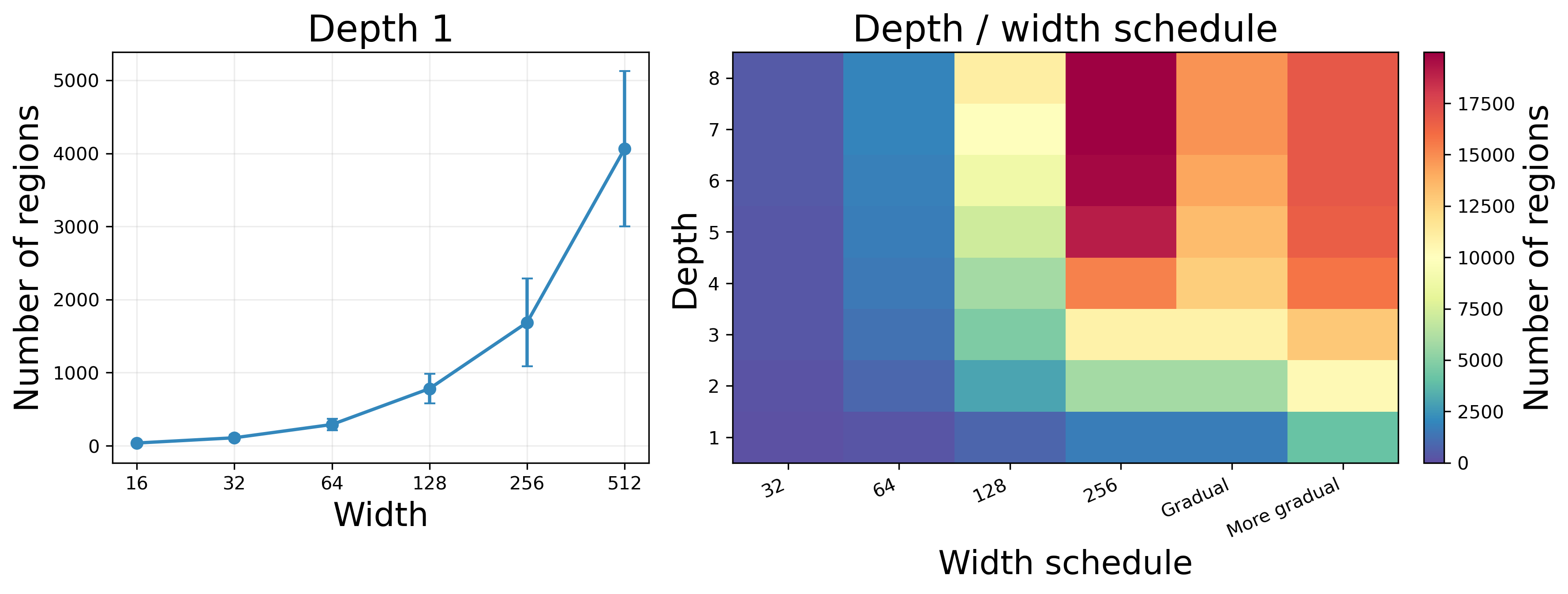}
    \end{tabular}

    \caption{
    Effect of width and depth on the number of causal regions of SNNs at initialization, evaluated for MNIST (top) and CIFAR-10 (bottom).} 
    \label{fig:cifar_mnist_init1}
\end{figure}

\paragraph{Comparison of independent vs shared weights} 
We compare SNNs with arbitrary positive weights to those with shared weights; see Figure~\ref{fig:cifar_mnist_shared_arbitrary_init1}. 
For a single hidden layer, the shared and arbitrary positive weight networks produce very similar region counts over a broad range of widths on both datasets. 
The depth experiments also remain qualitatively similar, although the arbitrary positive weight SNN generates more regions for some architectures and depths. 
A more controlled analysis during training could help isolate the effect of weight sharing and quantify its contribution more precisely. 

\begin{figure}[h]
    \centering
\begin{tabular}{
        @{}
        >{\centering\arraybackslash}m{0.01\textwidth}
        >{\centering\arraybackslash}m{0.97\textwidth}
        @{}
    }
        \rotatebox{90}{MNIST}
        &
            \includegraphics[width=\linewidth]
                {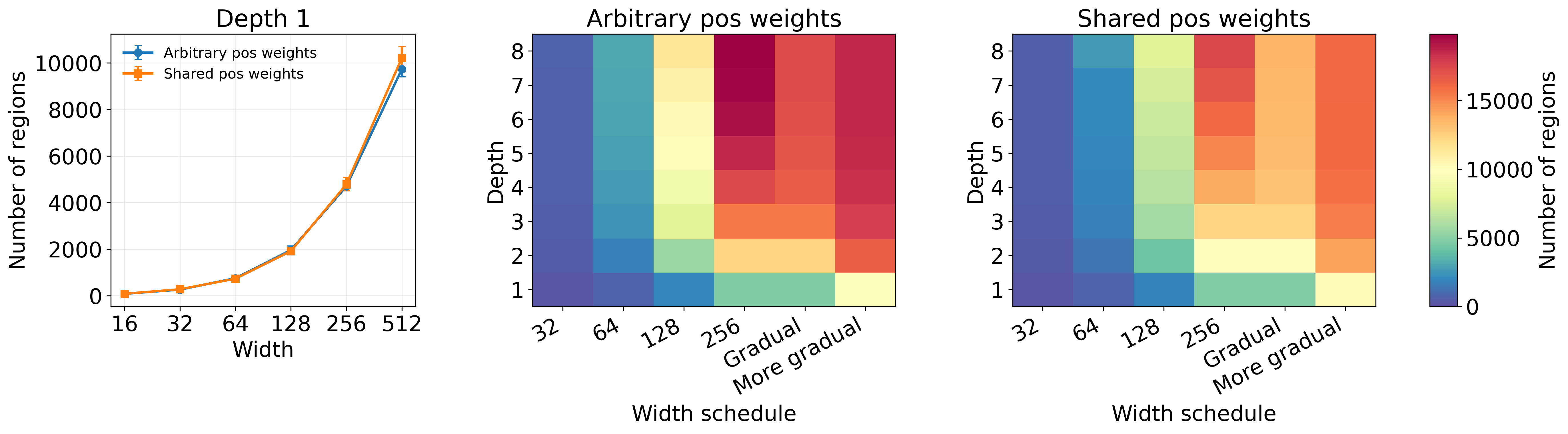}
        \\ 
        \rotatebox{90}{CIFAR-10}
        &
            \includegraphics[width=\linewidth]
                {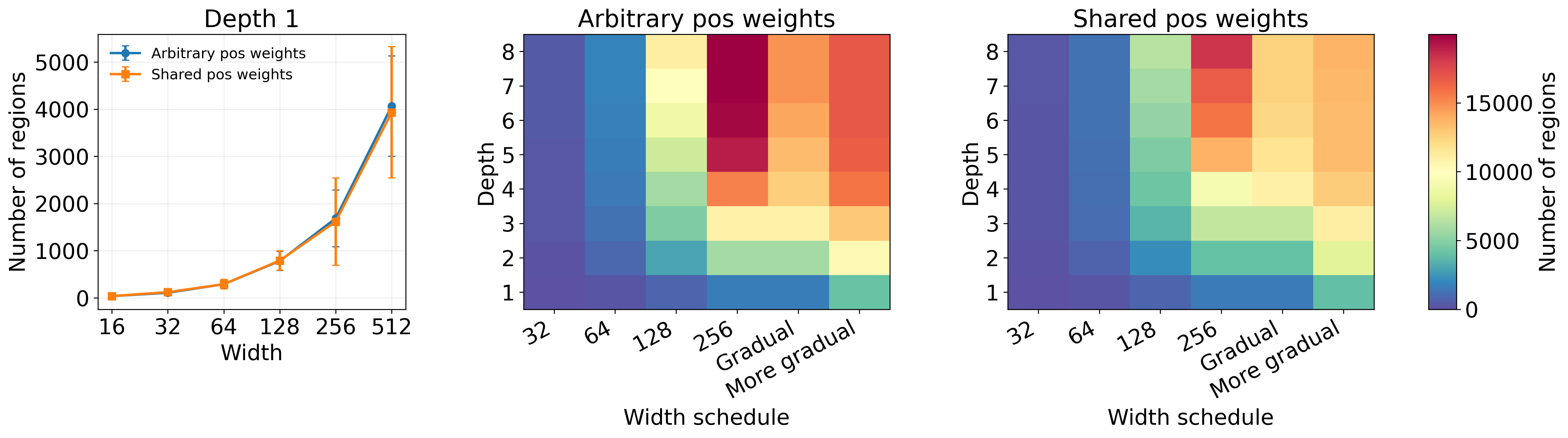}
    \end{tabular}

    \caption{
    Comparison of SNNs with arbitrary positive weights and SNNs with shared weights, showing the effect of width and depth, 
    evaluated for MNIST (top) and CIFAR-10 (bottom). 
    }
    \label{fig:cifar_mnist_shared_arbitrary_init1}
\end{figure}

\paragraph{Comparison of SNNs vs ReLU networks} The estimated number of regions for SNNs is consistently larger than that of the corresponding ReLU networks.
The difference is already visible for shallow networks and remains pronounced across depth on both CIFAR-10  and MNIST,  shown in Figure~\ref{fig:cifar_mnist_relu_ttfs_init1}. 

\begin{figure}[h]
    \centering
    \begin{tabular}{
        @{}
        >{\centering\arraybackslash}m{0.01\textwidth}
        >{\centering\arraybackslash}m{0.97\textwidth}
        @{}
    }
        \rotatebox{90}{MNIST}
        &
        \includegraphics[width=\linewidth]
            {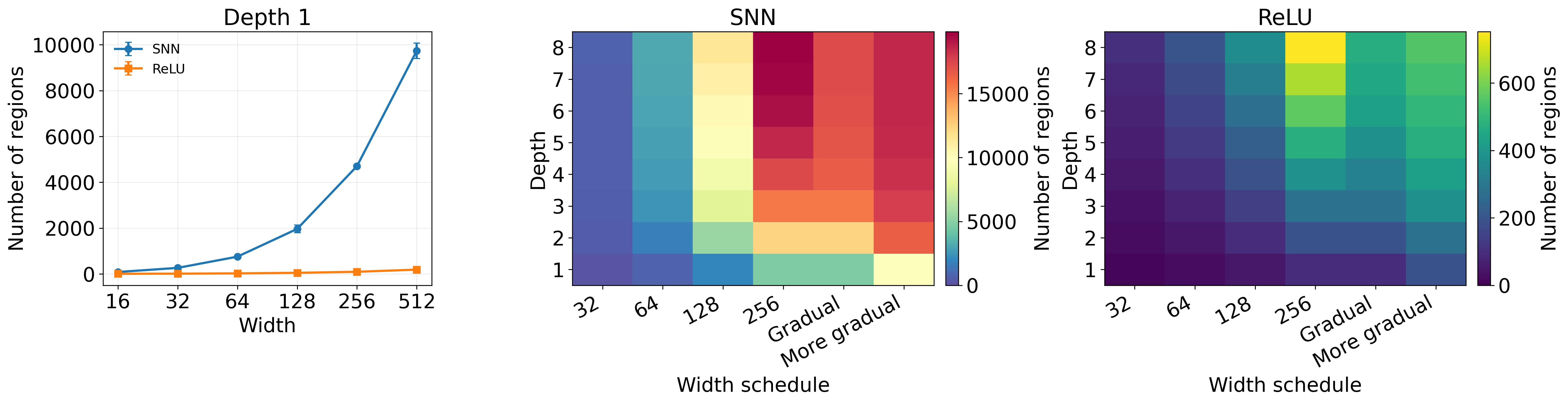}
        \\[0.5em]
        \rotatebox{90}{CIFAR-10}
        &
        \includegraphics[width=\linewidth]
            {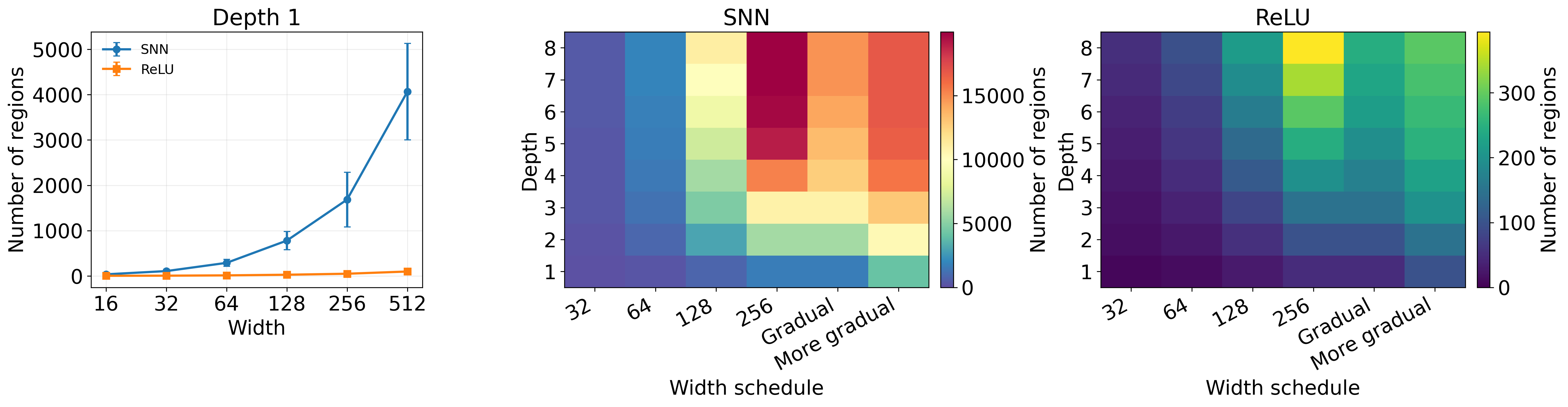}
    \end{tabular}

    \caption{
    Comparison of SNNs and matched ReLU networks, showing the effect of width and depth on the estimated region count, 
    evaluated for MNIST (top) and CIFAR-10 (bottom).  
    }
    \label{fig:cifar_mnist_relu_ttfs_init1}
\end{figure}

\paragraph{Exact enumeration over two-dimensional affine slices} 
To complement the trajectory-based region estimates at initialization, we additionally perform exact region enumeration on bounded two-dimensional affine slices of the input space during training, following the general type of exact polyhedral analysis considered in~\cite{tseran2021maxoutexpectedcomplexity}. 
This experiment is intended as a small scale complement to the  trajectory-based study. 
Figure~\ref{fig:mnist_w10_d3_ttfs} 
illustrates how the exact region partition and the corresponding decision boundary evolve during the first few training epochs for SNN and ReLU networks of depth $3$ and width $10$, respectively. The two-dimensional enumeration reveals substantially richer local partition than is observed along individual one-dimensional trajectories, and the SNN realizes a considerably finer partition 
than the matched ReLU network. 
Full details of the enumeration procedure are provided in Appendix~\ref{appendix:experiments}.

\begin{figure}[h]
    \centering
    \begin{tabular}{c|c}
    \includegraphics[width=0.49\textwidth, clip=true, trim=30cm 0cm 30cm 0cm]{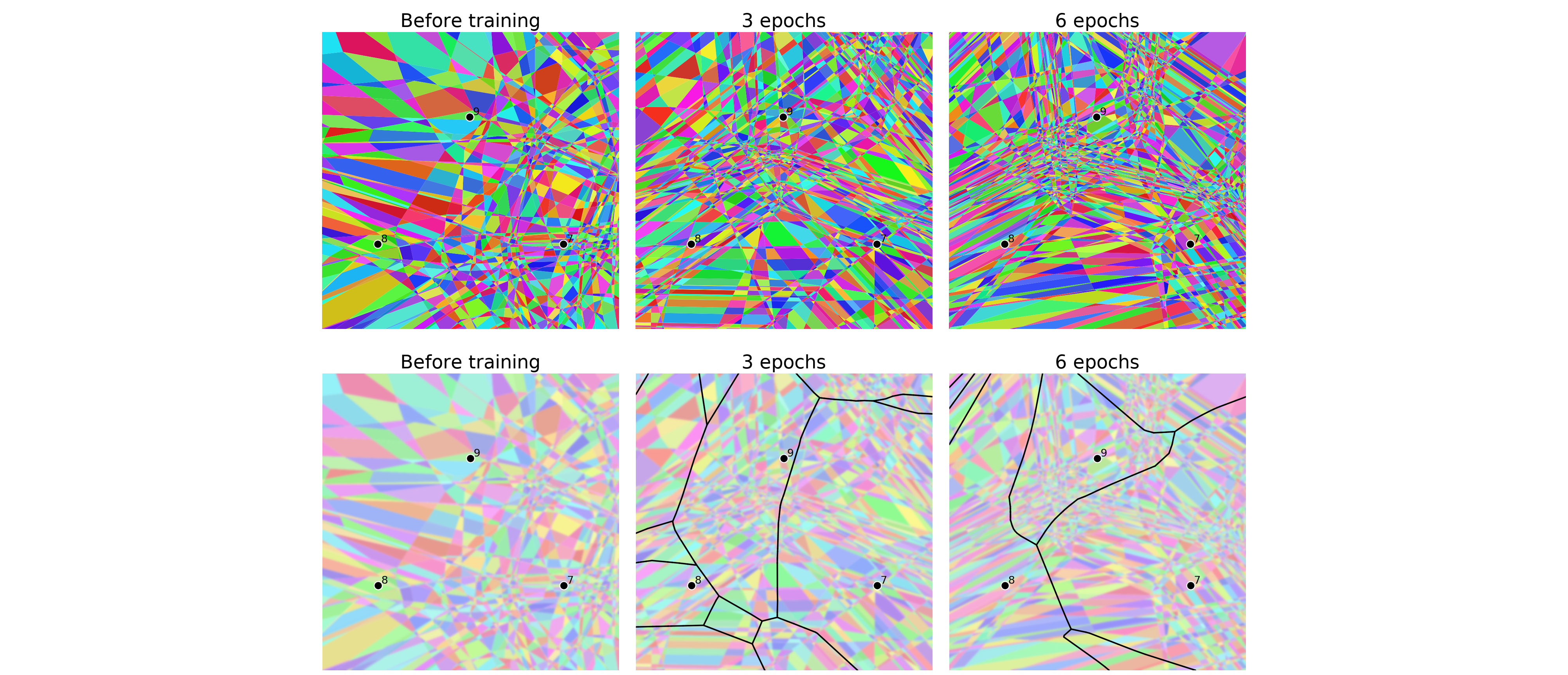}
    &
    \includegraphics[width=0.49\textwidth, clip=true, trim=30cm 0cm 30cm 0cm]{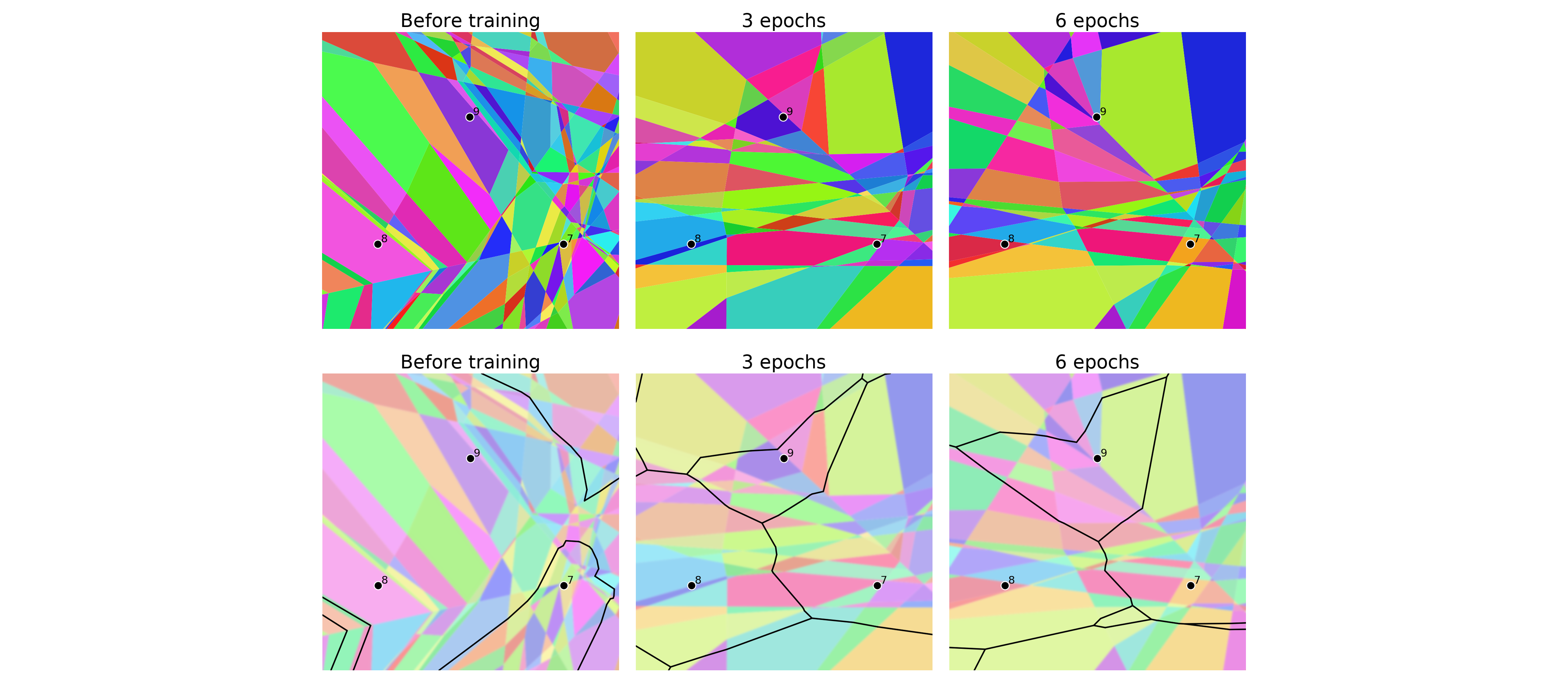}\\
    SNN & ReLU 
    \end{tabular} 
    \caption{Evolution of linear regions (top) and decision boundary (bottom) during training for an SNN (left panels) and a ReLU network (right panels) of depth $3$ and width $10$ on MNIST, visualized on a fixed two-dimensional slice of the input space given by the affine span of three randomly selected input samples.}
    \label{fig:mnist_w10_d3_ttfs}
\end{figure}

\paragraph{Additional experiments} Appendix~\ref{appendix:experiments} reports several complementary experiments. First, we show that the region complexity can depend substantially on the choice of initialization. Second, we compare positive and arbitrary weight SNNs to examine how allowing negative weights affects the estimated region count. Finally, we track the evolution of region complexity during training together with the corresponding test accuracy.

\section{Conclusion}
\label{section:conclusion}

We developed a geometric framework for studying the linear regions of the input-output maps computed by TTFS spiking neural networks in the continuous piecewise-linear setting. 
A central feature of this framework is that the region structure is described in terms of \emph{causal sets}: subsets of presynaptic spikes that arrive early enough to influence the firing time of a neuron. This provides a geometric and combinatorial description of TTFS computation that is adapted to asynchronous spike-time dynamics. 

At the level of a single neuron, we characterized the causal regions and showed that the firing-time map is the pointwise minimum of an exponential number of highly structured affine functions. In particular, a TTFS neuron with $d$ inputs can exhibit up to $2^d-1$ distinct causal sets, giving it a substantially different computational profile from a ReLU neuron with the same number of inputs and parameters. Building on this characterization, we derived upper and lower bounds on the maximal number of linear regions realized by shallow TTFS networks. 
For fixed input dimension, the maximal region count grows polynomially with width, while the underlying causal structure imposes strong dependencies among the regions generated by different neurons. 

For deep networks, we showed that these causal partitions can compose to produce exponential growth with depth. Our construction gives an explicit mechanism by which successive TTFS layers fold the input space and multiply the number of linear regions. Thus, despite the strong causal constraints of individual neurons, deep TTFS networks can generate combinatorially complex input-output maps. We also analyzed networks with shared weights, where the causal sets across neurons become nested. In this setting, we obtained substantially sharper region counts, including an exact count for shallow networks. These results illustrate how architectural constraints such as weight sharing translate directly into constraints on causal geometry and hence on expressivity.

Although both TTFS SNNs and ReLU ANNs compute CPWL maps, their region partitions arise through different mechanisms. In ReLU networks, the regions are determined by binary activation patterns, with each neuron being either active or inactive for a given input. In TTFS SNNs, by contrast, the regions are determined by causal patterns, with each neuron associated with a set of presynaptic spikes that causally contribute to its firing.  This distinction gives rise to fundamentally different combinatorial structures.

Our experiments complement this theoretical picture by examining the region geometry realized by TTFS networks at initialization. They show that TTFS networks can already realize large numbers of linear regions before training and provide an initial comparison with corresponding ReLU architectures. For ReLU networks, variance-preserving schemes such as He initialization \cite{heinit2015} are designed to prevent signals from degenerating through depth. In TTFS networks, the scale of the signals is only part of the picture: the relative ordering and separation of spike times determine causal patterns and the locations of transitions between causal regions. Developing initialization and normalization schemes that explicitly account for this temporal geometry is therefore an interesting direction for future work. We view the experiments presented here as an initial investigation of this question rather than an exhaustive study of SNN initialization.

\paragraph{Future work} 

Several theoretical directions remain open. An immediate objective is to tighten the upper and lower bounds obtained here. One possible route is suggested by the polyhedral representation developed in Section~\ref{section:single_neuron_ttfsmodel}. 
For maxout networks, linear-region counts can be studied through vertices of polytopes constructed by iterated Minkowski sums and convex hulls of the polytopes associated with individual neurons
\cite{ZhangNaitzatLim2018TropicalGeometryDNN,CharisopoulosMaragos2018,montufar2022sharpboundsmaxout,balakin2025maxoutpolytopes}. 
Since TTFS neurons admit a maxout-like representation with constrained parameters, it is natural to ask whether an analogous polytope calculus can be developed for compositions of TTFS neurons. 
Understanding the combinatorics of the resulting structured polytopes may lead to sharper region-counting results for both shallow and deep SNNs.

A second direction is to understand more systematically the role of synaptic delays. Their effect on causal geometry depends strongly on how they are parameterized. The proof of Theorem~\ref{thm:deep_recursive_upper_bound} relies on two structural facts: increasing thresholds produce nested causal sets across neurons within a layer, as in Lemma~\ref{lem:nested_same_weights}, and, within a fixed region, all neurons see the same ordered presynaptic firing times, so their causal sets are prefixes of a common order, as in Lemma~\ref{lem:prefix_lengths_deepnetwork}. Both properties hold in the zero-delay setting considered in the theorem. If delays preserve a common effective presynaptic order within each layer, for example, if all neurons in a layer share the same delay vector from the preceding layer, the same nested-prefix argument continues to apply. By contrast, neuron-specific delay vectors can induce different arrival-time orderings for different postsynaptic neurons. Causal sets then need not be prefixes of a common order, potentially allowing substantially richer causal partitions and requiring different counting techniques.

Finally, our theoretical analysis concerns maximal region complexity, whereas the complexity typically realized by a network may be substantially smaller. Developing an expected-region theory for SNNs under random initialization, and ultimately after training, would complement the extremal bounds developed here and provide a more direct connection between causal geometry and practical network behavior. More generally, the causal-region framework provides a way to relate architectural choices, including depth, width, weight sharing, thresholds, and delays, to the geometry and complexity of the functions computed by TTFS networks. We hope that this framework can serve as a basis for a systematic geometric theory of asynchronous computation in spiking neural networks.

\subsection*{Acknowledgements}

M.\ Singh and G.\ Kutyniok acknowledge support by the project Next Generation AI Computing (gAIn), funded by the Bavarian Ministry of Science and the Arts and the Saxon Ministry for Science, Culture, and Tourism. 
M.\ Singh and G.\ Kutyniok are also partially supported by  DAAD programme Konrad Zuse Schools of Excellence in Artificial Intelligence, sponsored by the Federal Ministry of Education and Research, and additionally, they also acknowledge support from the Munich Center for Machine Learning (MCML). 

G.\ Kutyniok furthermore acknowledges support by the German Research Foundation under Grants DFG-SPP-2298, KU 1446/31-1 and KU 1446/32-1, and by the Bavarian Ministry for Digital Affairs. 

G.\ Mont\'ufar was supported in part by NSF grants DMS-2145630, CCF-2212520, and DMS-2522495; DARPA grant HR00112520014 through the Artificial Intelligence Quantified (AIQ) program; DFG project 464109215 within the Priority Programme SPP 2298 ``Theoretical Foundations of Deep Learning''; and the
BMFTR through DAAD project 57616814 (SECAI).

\subsection*{Use of AI-assisted tools} 
The authors used ChatGPT (OpenAI) to assist with language editing, exploration and refinement of some theoretical statements, and the development of computer code. 
All mathematical arguments, results, and conclusions were developed and independently verified by the authors, who take full responsibility for the content of the manuscript.

\bibliography{bib_file}

\appendix

 \newpage 

\section{Supplementary details on causal sets and causal regions}

\subsection{Proof of Lemma~\ref{lem:causal_feas_and_order_constant}}
\label{appendix:proofs_causalsetsregions}

\begin{proof}[Proof of Lemma~\ref{lem:causal_feas_and_order_constant}]
    Each inequality in \eqref{eq:ineqs} is affine in $\vec t$, and its boundary is one of the hyperplanes in $\cA_{\rm TTFS}$. 
    Since $C$ is a connected component of the complement of all such hyperplanes, none of these affine expressions can change sign on $C$. Thus, the causal set $S$ is constant on $C$. 
\end{proof}

\subsection{Proof of Proposition~\ref{prop:convex_polyhedron_SPhi}}
\label{appendix:convex_polyhedron_SPhiproof}

\begin{proof}[Proof of Proposition~\ref{prop:convex_polyhedron_SPhi}]
    Fix a causal pattern $\mathbf{S} = \bigl(S_i^\ell\bigr)_{\ell,i}$. We show by induction on layers that, on the set of inputs realizing a causal pattern $\mathbf{S}$ up to layer $\ell$, all firing times in the original input $\vec t$, and the defining constraints are linear in $\vec t$. 

    For layer $\ell = 1$. For each neuron $(1,i)$, the presynaptic vector is $\vec t^{(0)}=\vec t\in\R^{N_0}$. Fixing $S_i^1$ means $\vec t$ lies in the single-neuron region $R_{S_i^1}$ determined by weights $(w_{ij}^1)_{j=1}^{N_0}$ and threshold $\theta_i^1$. On this region,
    \begin{equation*}
        t_i^1(\vec t)=t_{v}^{S_i^1}(\vec t) = \frac{\theta_i^1+\sum_{j\in S_i^1} w_{ij}^1\, t_j}{\sum_{j\in S_i^1} w_{ij}^1},
    \end{equation*}
    which is affine in $\vec t$. The constraints defining $S_i^1$ are linear inequalities (see \eqref{eq:halfspace_condition} and \eqref{eq:strict_ineq}) in $\vec t$. Intersecting over all $i \in [N_1]$ shows that the set of inputs realizing the layer-1 part of $\mathbf{S}$ is a open convex polyhedron and that all $t_i^1(\vec t)$ are affine there. 

    Now, we do the inductive step. Assume the claim holds for layer $\ell - 1$. On the set of inputs realizing $S_\Phi$ up to layer $\ell-1$, each presynaptic time $t_j^{(\ell-1)}(\vec t)$ is affine in $\vec t$. Fix a neuron $(\ell,i)$ and its prescribed causal set $S_i^\ell\subseteq[N_{\ell-1}]$. Then, we have
    \begin{equation*}
        t_i^\ell(\vec t) = \frac{\theta_i^\ell+\sum_{j\in S_i^\ell} w_{ij}^\ell\, t_j^{(\ell-1)}(\vec t)}{\sum_{j\in S_i^\ell} w_{ij}^\ell}.
    \end{equation*}
    Since each $t_j^{(\ell-1)}(\vec t)$ is affine linear in $\vec t$, it follows that $t_i^\ell(\vec t)$ is affine in $\vec t$. Moreover, each causal constraint $t_j^{(\ell-1)}(\vec t)<t_i^\ell(\vec t)$ or $t_j^{(\ell-1)}(\vec t)\ge t_i^\ell(\vec t)$ becomes a strict or weak linear inequality in $\vec t$. Intersecting these linear constraints over all neurons $(\ell,i)$ and all layers $\ell$ yields that $R_{\mathbf{S}}$ is a (relatively) open convex polyhedron. The affine formulas show that all firing times are affine linear on $R_{\mathbf{S}}$.
    
    The final statement follows because the sets $\{R_{\mathbf{S}}\}$ partition $\R^{N_0}$ up to boundaries where equalities occur.
\end{proof}

\section{Supplementary details for shallow SNN}
\label{section:appendix_proofs_shallow_snn}

In this section, we provide proofs for the results in the case of shallow SNNs from Section~\ref{section:boundsshallowsnn}. We begin with the arrangement-based upper bound and the asymptotic lower bound, and then turn to the shared weight regime, where exact counting and attainability is established. 

\subsection{Proof of Proposition~\ref{prop:upper_lower_bound_finite_Hsj}}
\label{subsec:appendix_proof_Attfs_ub}

\begin{proof}[Proof of Proposition~\ref{prop:upper_lower_bound_finite_Hsj}]
By Lemma~\ref{lem:causal_feas_and_order_constant}, a causal set $S$ can only change when some relevant defining affine inequality becomes tight, that is, on some hyperplane of the form $t_v^S=t_j, j\notin S$. 
Thus, the tuple of causal sets is constant on each region of the arrangement. Hence, the number of distinct tuples is at most the maximal region count cut out by $N$ hyperplanes in $\R^{d-1}$ with $N=m\kappa_d$. 
Combining this with the naive causal set bound from Proposition~\ref{prop:naive_bounds}, we obtain the desired bound. 

For part (i). For fixed $d$ and large $m$, the top term dominates polynomially, using the fact that $\binom{m\kappa_d}{k} \leq \frac{(m\kappa_d)^k}{k!}$ and $\kappa_d$ only depending on $d$, we can bound
\begin{equation*}
    \sum_{k=0}^{d-1}\binom{m\kappa_d}{k} \leq \sum_{k=0}^{d-1}\frac{(m\kappa_d)^k}{k!} = \mathcal{O}\bigl((m\kappa_d)^{d-1}\bigr) = \mathcal{O}(m^{d-1}),
\end{equation*}

For part (ii). For the case of fixed $m$ and large $d$, we use the crude bound $\sum_{k=0}^{d-1}\binom{m\kappa_d}{k} \leq d\binom{m\kappa_d}{d-1}$ and the inequality $\binom{n}{r} \leq (en/r)^r$ to obtain 
\begin{equation*}
    R_{\max}(d,m) \leq d\left(\frac{em\kappa_d}{d-1}\right)^{d-1}.
\end{equation*}
Since $\kappa_d = d(2^{d-1}-1) = \Theta(d2^d)$, for fixed $m$, there is a constant $C>0$ such that $\frac{em\kappa_d}{d-1} \leq C2^d$ for large $d$. Hence, 
\begin{equation*}
    R_{\max}(d,m) \leq d(C2^d)^{d-1} = 2^{\mathcal{O}(d^2)}.
\end{equation*}
On the other hand, the naive bound satisfies 
\begin{equation*}
    (2^d-1)^m \leq 2^{md} = 2^{\mathcal O(d)}
\end{equation*}
for fixed $m$. Taking the smaller of the two estimates, we obtain the desired result. 
\end{proof}

\subsection{Proof of Proposition~\ref{prop:lower_bound_asymptotic}}
\label{subsec:appendix_proof_asymp_lb}
We now give proof for the asymptotic lower bound. On a suitable local domain in the essentialized space $T$, each hidden neuron acts as a single hyperplane, and the resulting count is given by the bounded region count of a generic hyperplane arrangement. 

\begin{proof}[Proof of Proposition~\ref{prop:lower_bound_asymptotic}]
    We work in the input space $T$. For $\vec t \in T$, let $x_i = t_i - t_d$ for $i = 1,\dots, d-1$. Consider an open box 
    \begin{equation*}
        U \coloneqq \{ \vec x \in \R^{d-1} \colon \lvert x_i + 1 \rvert < \varepsilon \}, \quad \text{for } 0<\varepsilon<\frac{1}{4},
    \end{equation*}
    so that on $U$, we have $x_i<0$ and hence $t_i<t_d$ for all $i<d$, with all times $t_1,\dots,t_{d-1}$ within $2\varepsilon$ of each other. For each hidden neuron $r \in [m]$, let $\vec w^{(r)} \in \R^{d-1}_{>0}$ be the positive weights for the first $d-1$ inputs, with $W_r = \sum_{i=1}^{d-1}w^{(r)}_i$, and choose $\theta_r > 0$ so that
    \begin{equation}\label{eq:eps_condition}
       1-\varepsilon < \theta_r/W_r < 1+\varepsilon.
    \end{equation}
    On $U$, condition \eqref{eq:eps_condition} forces any firing that occurs before $t_d$ to happen only after all of inputs $1,\dots,d-1$ have arrived, hence the causal set is $[d-1]$. Moreover, 
    \begin{equation*}
        t_{v,r}^{[d-1]}(\vec t) = t_d + \frac{\theta_r+ \sum_{i=1}^{d-1}w_i^{(r)}x_i}{W_r},
    \end{equation*}
    so $t_{v,r}^{[d-1]}(\vec t) \leq t_d$ holds exactly on the halfspace $\sum_{i=1}^{d-1}w_i^{(r)} x_i + \theta_r \leq 0$. Thus, on $U$ each neuron divides the space into two causal regions
    \begin{equation*}
        S_r(x)=
            \begin{cases}
            [d-1], & \sum_{i=1}^{d-1}w_i^{(r)} x_i + \theta_r \leq 0,\\
            [d], & \sum_{i=1}^{d-1}w_i^{(r)} x_i + \theta_r > 0.
            \end{cases}
    \end{equation*}
    Choose the $m$ hyperplanes $ H_r \colon \sum_{i=1}^{d-1}w_i^{(r)} x_i + \theta_r \leq 0$ in general position and with all bounded regions contained in $U$ (e.g. choose the offset $\theta_r$ in such a way that all $H_r$ pass sufficiently close to $x^0 = (-1,\dots,-1) \in U$)). Then, distinct bounded regions yield distinct causal sets, hence,
    \begin{equation*}
        R_{\max}(d,m) \geq \#\{\text{bounded regions of }\{H_r\}\} = \binom{m-1}{d-1} \ =\ \Omega(m^{d-1}).
    \end{equation*}
\end{proof}

\subsection{Proof of Proposition~\ref{prop:chain_bound_general_dm}}
\label{appendix:proof_shallowSNN_chainbound}

We now move to the shared-weight regime. The attainability result given in Lemma~\ref{lemma:single_layer_local_braid_cone_attainability} 
shows that every nondecreasing prefix-length sequence can be realized on a nonempty open set inside a fixed ordering cone. 
Figure~\ref{fig:braidcone_attainability_d3_clean} illustrates this construction in the case of input dimension $d=3$ and layer width $m=2$.

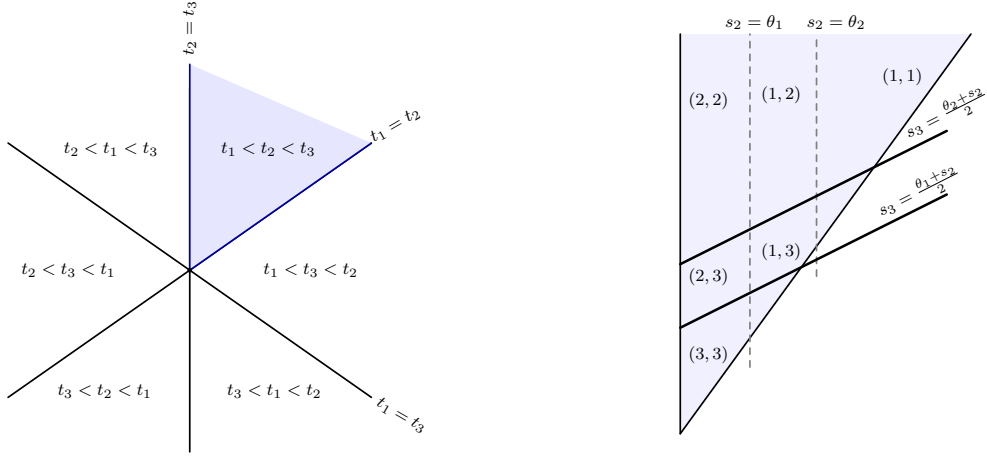
\begin{figure}[t]
\centering
\resizebox{0.85\linewidth}{!}{%

\begin{tikzpicture}[line cap=round, line join=round, font=\small, >=Latex]

\draw[thick] (-3,-2.1) -- (3,2.1);     
\draw[thick] (-3, 2.1) -- (3,-2.1);    
\draw[thick] (0.0,-3) -- (0.0,3);      

\node[font=\scriptsize, rotate=30, fill=white, inner sep=1pt] at (3.4,2.4) {$t_1=t_2$};
\node[font=\scriptsize, rotate=-30, fill=white, inner sep=1pt] at (3.50,-2.4) {$t_1=t_3$};
\node[font=\scriptsize, rotate=90, fill=white, inner sep=1pt] at (0,4) {$t_2=t_3$};

\fill[blue!10] (0,0) -- (0,3.4) -- (3,2.1) -- cycle;
\draw[blue!60!black, thick] (0,0) -- (0,3.4);
\draw[blue!60!black, thick] (0,0) -- (3,2.1);

\node[font=\scriptsize] at (1.3,2) {$t_1<t_2<t_3$};
\node[font=\scriptsize] at (-1.3,2) {$t_2<t_1<t_3$};
\node[font=\scriptsize] at (-2,0) {$t_2<t_3<t_1$};
\node[font=\scriptsize] at (-1.4,-2) {$t_3<t_2<t_1$};
\node[font=\scriptsize] at (1.4,-2) {$t_3<t_1<t_2$};
\node[font=\scriptsize] at (2,0) {$t_1<t_3<t_2$};

\fill (0,0) circle (1.1pt);

\coordinate (V) at (8.1,-2.7);

\coordinate (R1) at (8.1,3.9);     
\coordinate (R2) at (12.9,3.9);    

\fill[blue!6] (V) -- (R1) -- (R2) -- cycle;

\draw[thick] (V) -- (R1);
\draw[thick] (V) -- (R2);

\draw[gray, dashed, thick] (9.25,-1.6) -- (9.25,3.9);

\draw[gray, dashed, thick] (10.35,-0.1) -- (10.35,3.9);

\draw[very thick] (8.1,-0.95) -- (12.5,1.25);

\draw[very thick] (8.1,0.10) -- (12.5,2.30);

\node[font=\scriptsize, anchor=west] at (11.3,3.2) {$(1,1)$};
\node[font=\scriptsize, anchor=west] at (9.3,2.9) {$(1,2)$};
\node[font=\scriptsize, anchor=west] at (9.3,0.3) {$(1,3)$};
\node[font=\scriptsize, anchor=west] at (8.1,2.8) {$(2,2)$};
\node[font=\scriptsize, anchor=west] at (8.1,-0.1) {$(2,3)$};
\node[font=\scriptsize, anchor=west] at (8.1,-1.4) {$(3,3)$};

\node[font=\scriptsize, anchor=west,  inner sep=1pt] at (10.15,4.1) {$s_2=\theta_2$};
\node[font=\scriptsize, anchor=west,  inner sep=1pt] at (8.8,4.1) {$s_2=\theta_1$};

\node[font=\scriptsize, rotate=27, inner sep=1pt] at (12.05,1.3)
{$s_3=\frac{\theta_1+s_2}{2}$};
\node[font=\scriptsize, rotate=27, inner sep=1pt] at (12.5,2.60)
{$s_3=\frac{\theta_2+s_2}{2}$};

\end{tikzpicture}
}
\caption{Attainability for $d=3, m=2$ (unit weights, thresholds $\theta_1, \theta_2 > 0$). (a) The braid arrangement partitions input space into six braid cones, each ordering the spike times $t_{\pi(1)} < t_{\pi(2)} < t_{\pi(3)}$ for some permutation $\pi$. (b) Fixing one cone reduces feasible causal sets to prefixes of $\pi$ (fix $t_{\pi(1)}< t_{\pi(2)} < t_{\pi(3)}$ and set $s_i = t_{\pi(i)}$). Inside that cone, the attainability inequalities subdivide it into six open regions labeled by the realizable size pairs $(|S(\theta_1)|,|S(\theta_2)|)\in\{(1,1),(1,2),(1,3),(2,2),(2,3),(3,3)\}$. Choosing the cone $\pi$, and choosing a point in the corresponding subregion realizes the desired nested causal set $(S_1,S_2)$.}
\label{fig:braidcone_attainability_d3_clean}
\end{figure}

\begin{lemma}
\label{lemma:single_layer_local_braid_cone_attainability}
    Let $K, N \geq 1$, and consider an SNN layer with $K$ presynaptic spike times $\vec t \coloneq (t_1, t_2, \dots, t_K) \in \R^K$ and $N$ neurons sharing the same positive weight vector $\vec w \coloneq (w_1, \dots, w_K)$ 
    and having strictly ordered thresholds $0<\theta_1<\theta_2<\dots<\theta_N$. 
    Fix a permutation $\pi$ of $[K]$ and consider the cone $\cB_{\pi} \coloneq \{\vec t  \in \R^K \colon t_{\pi(1)} < t_{\pi(2)} < \cdots < t_{\pi(K)}\}$. 
    Then, for every weakly increasing prefix-length tuple 
    \begin{equation*}
        \vec q = (q_1, q_2, \dots, q_N) \in [K]^N, \quad \text{where  } q_1 \leq \cdots \leq q_N, 
    \end{equation*}
  there exists a nonempty open set $Z \subset \cB_{\pi}$ such that, for every $\vec t \in Z$ and every neuron $r \in [N]$, 
  the causal set is 
  $S_r(\vec t) = S_r = \{\pi(1), \dots, \pi(q_r)\}$. 
\end{lemma}

\begin{proof}[Proof of Lemma~\ref{lemma:single_layer_local_braid_cone_attainability}]
    For $a = 1, \dots, K$, write the ordered input spike times as $y_a \coloneq t_{\pi(a)}$. 
    For a threshold $\theta > 0$, the candidate time corresponding to the prefix $\{\pi(1),\dots,\pi(a)\}$ is 
    \begin{equation*}
        t^a(\theta;\vec t) = \frac{\theta+\sum_{j=1}^a w_{\pi(j)}y_j}{W_a}, \quad \text{where  } W_a \coloneq \sum_{j=1}^{a}w_{\pi(j)}.
    \end{equation*}
    The prefix $\{\pi(1),\dots,\pi(a)\}$ is feasible exactly when $y_a < t^a(\theta;\vec t) \leq y_{a+1}$, with the convention $y_{K+1} \coloneq +\infty$. 
    We now express these feasibility inequalities in a convenient form. 
    Define
    \begin{equation*}
        A_a \coloneq \sum_{j=1}^a w_{\pi(j)}(y_a - y_j)
    \end{equation*}
    to be the total weighted input accumulated from the first $a$ spikes by the time the $a$th spike arrives. 
    A direct computation shows that 
    \begin{equation}\label{eq:propattain_1}
        t^a(\theta;\vec t) - y_a = \frac{\theta-A_a}{W_a}. 
    \end{equation}
    Hence, $t^a(\theta;\vec t) > y_a \iff \theta > A_a$. 
    Similarly, since $A_{a+1} = A_a + W_a(y_{a+1} - y_a)$, one obtains
    \begin{equation}\label{eq:propattain_2}
        y_{a+1}-t^a(\theta;\vec t) = \frac{A_{a+1}-\theta}{W_a}. 
    \end{equation}
    Therefore, from \eqref{eq:propattain_1} and \eqref{eq:propattain_2}, 
    the candidate with prefix length $a$ 
    is feasible precisely when $A_a < \theta \leq A_{a+1}$. 
    
    Thus, it suffices to choose 
    increasing levels $0 = A_1 < A_2 <\cdots < A_K < A_{K+1} \coloneq +\infty$ 
    such that, for every neuron $r\in[N]$, 
    $$
    A_{q_r} < \theta_r <  A_{q_r+1}.
    $$
    Such a choice is possible because both the prefix lengths and thresholds are ordered,  $q_1 \leq \cdots \leq q_N$ and $\theta_1 < \cdots < \theta_N$. 
    Indeed, for each $a=2,\dots,K$, all neurons $r$ satisfying $q_r<a$ precede those satisfying $q_r\geq a$, and hence
\[
\max_{r:q_r<a}\theta_r
<
\min_{r:q_r\geq a}\theta_r.
\]
    Therefore, we can place $A_a$ between the thresholds assigned to prefix lengths less than $a$ and those assigned to prefix lengths at least $a$.

    Once such levels $A_a$ are chosen, it remains to realize them by actual spike times. 
    Set $y_1=0$ and recursively define 
    \begin{equation*}
        y_{a+1} - y_a = \frac{A_{a+1} - A_a}{W_a}, \qquad a=1,\dots,K-1.
    \end{equation*}
    These gaps are positive, so the resulting point $\vec t$ lies in $\cB_\pi$. 
    Moreover, the identity
$A_{a+1}=A_a+W_a(y_{a+1}-y_a)$ shows recursively that the resulting spike
times realize precisely the chosen levels $A_a$.
    By construction, every neuron $r$ satisfies $A_{q_r} < \theta_r < A_{q_r+1},$ and therefore
    \begin{equation*}
        t_{\pi(q_r)} < t^{q_r}(\theta_r;\vec t) < t_{\pi(q_r+1)}.
    \end{equation*}
    Hence, neuron $r$ has causal set $S_r = \{\pi(1),\dots,\pi(q_r)\}.$ Since all inequalities used above are strict, they remain valid on some neighborhood of the constructed point $\vec t$ inside $\cB_\pi$. Hence, there exists a nonempty open set $Z \subset \cB_\pi$ on which the same causal sets persist. 
\end{proof}

We now prove the result for shallow SNN in the shared weight regime by first counting the number of nondecreasing sequences of nonempty subsets of $[d]$, and then applying the local attainability result above to show that each such sequence is realizable by a shared weight shallow SNN.

\begin{proof}[Proof of Lemma~\ref{lem:nested-squences}] 
    Consider a nondecreasing sequence $\emptyset \neq S_1 \subseteq S_2 \subseteq \cdots \subseteq S_m \subseteq [d]$. Fix the last set $B \coloneqq S_m$ with $\lvert B \rvert = k$.
    Any chain $S_1 \subseteq S_2 \subseteq \cdots \subseteq S_m = B$ with $S_1\neq\emptyset$ is uniquely determined by, for each $j\in B$, the first index $p_j\in\{1,\dots,m\}$ at which $j$ enters the chain, i.e., $p_j \coloneq \min\{r \in [m] \colon j\in S_r\} \in [m]$. Thus, there are $m^k$ chains ending at $B$. Among them, $S_1=\emptyset$ holds exactly when $p_j\in\{2,\dots,m\}$ for all $j\in B$, which gives $(m-1)^k$ chains. Therefore, the number of chains ending at $B$ with $S_1 \neq \emptyset$ equals $m^k-(m-1)^k$. Summing over all choices of $B\subseteq[d]$ with $\lvert B \rvert = k$ 
    gives 
    \begin{equation}\label{eq:mkmkminus1}
        \sum_{k=1}^{d} \binom{d}{k}(m^k - (m-1)^k)
    \end{equation}
    and by the application of binomial theorem to \eqref{eq:mkmkminus1}, we obtain \eqref{eq:counting_chains}. 
\end{proof}

\begin{proof}[Proof of Proposition~\ref{prop:chain_bound_general_dm}] By Lemma~\ref{lem:nested_same_weights}, for each fixed $\vec t$ the causal sets are nested as thresholds increase, hence, $S_1 \subseteq \cdots \subseteq S_m$. Therefore, Lemma~\ref{lem:nested-squences} gives 
\begin{equation*}
    R_{\max}^{\rm {shared}} \leq (m+1)^d - m^d.
\end{equation*}
It remains to show that every nested sequence counted by Lemma~\ref{lem:nested-squences} is realizable. 
Let $\emptyset \neq S_1\subseteq S_2\subseteq\cdots\subseteq S_m\subseteq[d]$ be an arbitrary such sequence. 
Set $B = S_m$ and $k = \lvert B \rvert$ as before. 
For each $j$, let $p_j = \min\{r \in [m] \colon j \in S_r\}$ denote the first index at which $j$ enters the chain.
Order the elements $b_1, \dots, b_k$ in $B$ such that $p(b_1) \leq \cdots \leq p(b_k)$, and extend to a permutation $\pi$ of $[d]$ by arranging the elements of $[d] \setminus B$ after $b_1, \dots, b_k$ in arbitrary order. For each $r \in [m]$, define $q_r \coloneq \lvert S_r \rvert \in [k]$. Then, because the chain is nested and the ordering $b_1, \dots, b_k$ is compatible with their entry indices, 
we have 
\begin{equation*}
    S_r = \{b_1, \dots, b_{q_r}\} = \{\pi(1), \dots, \pi(q_r)\}, \qquad r \in \{1,\dots, m\},
\end{equation*} 
and the sequence $(q_1,\dots, q_m)$ is nondecreasing. Now, we apply Lemma \ref{lemma:single_layer_local_braid_cone_attainability} with $K = d, N=m$, and the braid cone $\cB_{\pi}$. Since $(q_1, \dots, q_m) \in [d]^m$ is nondecreasing, the lemma yields a nonempty open set $U \subset \cB_\pi$ such that for every $\vec t \in U$, neuron $r$ has causal set $S_r = \{\pi(1), \dots, \pi(q_r)\}$. Hence, the prescribed chain is realized on $U$. Since the chain was arbitrary, every chain counted in \eqref{eq:chain_count} is realizable. Therefore the upper bound in \eqref{eq:chain_count} is tight. 
\end{proof}

\section{Supplementary details for deep SNN}
\label{appendix:deepSNN}

In this subsection, we give the proofs used in the deep network analysis. 

\begin{proof}[Proof of Lemma~\ref{lem:prefix_lengths_deepnetwork}] As a result of Lemma~\ref{lem:nested_same_weights}, the threshold ordering and shared weights imply nested causal sets across neurons in the same layer, and strict ordering of firing times. Indeed, if $t^{(\ell)}_r \leq t^{(\ell)}_{r+1}$, then any presynaptic index $j$ satisfying $t^{(\ell-1)}_j < t^{(\ell)}_r$ also satisfies $t^{(\ell-1)}_j < t^{(\ell)}_{r+1}$. This yields $S^{(\ell)}_r \subseteq S^{(\ell)}_{r+1}$, which for prefixes is exactly $q_{\ell,r} \leq q_{\ell,r+1}$.
\end{proof}

The next example illustrates how Lemma~\ref{lem:prefix_lengths_deepnetwork} and Theorem~\ref{thm:deep_recursive_upper_bound} combine to show, in the case of a two-layer network, all potential causal sets for the neurons in the second layer.

\begin{example}[$d=N_1=N_2=3$]
    Consider a two hidden layer network with $N_0 = d = 3, N_1=N_2 = 3$ in the shared weight regime. 
    
    \noindent \emph{Layer 1 bound.}
    Proposition~\ref{prop:chain_bound_general_dm} gives
    \begin{equation*}
        C(3,3) = \sum_{k=1}^{3}\binom{3}{k}\bigl(3^k-2^k\bigr) = \binom31(3-2)+\binom32(9-4)+\binom33(27-8) = 3+15+19=37.
    \end{equation*}

    \noindent \emph{Layer 2 bound via prefixes (relative to the layer-1 order).} On some region of the input space, the three layer-1 spike times admit a strict order, that is, there exists a permutation $\pi$ of $\{1,2,3\}$ such that $t^{(1)}_{\pi(1)}<t^{(1)}_{\pi(2)}<t^{(1)}_{\pi(3)}$. 
    Lemma~\ref{lem:prefix_lengths_deepnetwork} implies that each layer-2 neuron has a prefix causal set with respect to this order, that is, one of $\{\pi(1)\},\quad \{\pi(1),\pi(2)\},\quad \{\pi(1),\pi(2),\pi(3)\}$ and, the corresponding prefix lengths $(q_{2,1},q_{2,2},q_{2,3})\in\{1,2,3\}^3$ form a weakly increasing tuple. Hence the number of possible layer-2 label tuples is
    \begin{equation*}
        \binom{N_1+N_2-1}{N_2}=\binom{3+3-1}{3}=\binom{5}{3}=10.
    \end{equation*}
    Therefore, Theorem~\ref{thm:deep_recursive_upper_bound} yields $R_\max(3;3;3) \leq C(3,3) \cdot \binom53 = 370$. The $10$ weakly increasing triples in $\{1,2,3\}^3$
    \begin{equation*}
        (1,1,1),(1,1,2),(1,1,3),(1,2,2),(1,2,3),(1,3,3),(2,2,2),(2,2,3),(2,3,3),(3,3,3).
    \end{equation*}
\end{example}

\begin{lemma}[Obstruction to a two-layer shared-weight multi-fold]
\label{lemma:shared-weight-obstruction}
Consider the two-layer multi-fold architecture $(2,m,2)$ without
synaptic delays. 
Suppose that the neurons within each layer share their
incoming positive weights and may differ only in their thresholds. 
For the ordered-threshold first-layer construction described above, the relative firing time of the two second-layer neurons,
\[
    h(x)=z_+(x)-z_-(x),
\]
is monotone whenever their thresholds are ordered. In particular, the
second layer cannot fold two or more consecutive intervals \(J_k\)
affinely and bijectively onto a common nondegenerate output interval.
\end{lemma}

\begin{proof}
Write \(x=t_2-t_1\) and, by translation equivariance, set \(t_1=0\).
On each interval \(J_k=(c_k,c_{k+1})\), the first-layer firing times satisfy
\[
    y_r'(x)=
    \begin{cases}
        0, & r\leq k,\\
        \alpha, & r>k,
    \end{cases}
    \qquad
    \alpha\coloneqq\frac{b}{a+b}>0,
\]
and remain ordered as \(y_1(x)<\cdots<y_m(x)\).

Let \(q_1,\ldots,q_m>0\) be the weights shared by the two second-layer
neurons, and suppose that their thresholds satisfy \(\eta_-<\eta_+\).
Their causal sets are therefore ordered prefixes
\[
    S_-=\{1,\ldots,s_-\},
    \qquad
    S_+=\{1,\ldots,s_+\},
    \qquad
    s_-\leq s_+.
\]
Writing \(Q_s=\sum_{r=1}^s q_r\), the firing time of a neuron with causal
set \(\{1,\ldots,s\}\) has derivative
\[
    z_s'(x)
    =
    \alpha\frac{\sum_{r=k+1}^s q_r}{Q_s}
    =
    \alpha
    \begin{cases}
        0, & s\leq k,\\[1ex]
        1-\dfrac{Q_k}{Q_s}, & s>k.
    \end{cases}
\]
For fixed \(k\), this quantity is nondecreasing in \(s\). Consequently,
\(z_+'(x)\geq z_-'(x)\) wherever the derivatives exist, and hence
\[
    h'(x)=z_+'(x)-z_-'(x)\geq0.
\]
Since \(h\) is continuous and piecewise affine, it is nondecreasing.
It therefore cannot map two consecutive intervals affinely and
bijectively onto the same nondegenerate interval.
\end{proof}

\section{Experimental details}
\label{appendix:experiments}

Unless otherwise stated, all experiments are performed at random initialization without training. 
For each configuration, region complexity is estimated along linear trajectories using $20{,}000$ intervals, corresponding to $20{,}001$ sampled points. 
Therefore, the largest observable number of regions along a trajectory is $20{,}001$.
We additionally report an experiment tracking region complexity during training. 

\paragraph{SNN forward pass} 
The affine linear encoder maps the real-valued input to input spike times, and these spike times are then processed by the hidden and output SNN layers. 
Consider a neuron $j$ with $d$ presynaptic spike times $t_1, \dots, t_d$. We first sort them as $t_{(1)} \leq \cdots \leq t_{(d)}$, together with their corresponding weights. For the prefix $S_{j,k} = \{(1), \dots, (k)\}$, the candidate firing time is
\begin{equation}\label{eq:expSjcausalset}
    t_j^{S_{j,k}} = \frac{\theta_j + \sum_{r = 1}^{k}w_{r}(t_{(r)}}{\sum_{r=1}^{k}w_{(r)j}}.
\end{equation}
For positive weights, the causal set $S_j$ is the first prefix for which $t_j^{S_{j,k}} < t_{(k+1)}$, with the full prefix always admissible. The output spike time is then $t_j = t_j^{S_j}$. 
The weighted numerator and denominator are accumulated sequentially as the sorted inputs are processed. Thus, after sorting, finding $S_j$ requires checking at most $d$ prefix candidates rather than testing the $2^d$ possible subsets of presynaptic neurons. 
In this sense, the forward pass is explicitly event-based and asynchronous, that is, the effective input to each neuron is not a full activation vector in the ANN sense, but an ordered set of arrival times together with weights and delays. 
This forward pass construction follows the implementation in \cite{Neuman2025SNNTTFS}, but is adapted here to our experimental setting.

\paragraph{Initialization} 
For the ReLU ANN, all weights are initialized with Kaiming normal (He) initialization \cite{heinit2015}, while biases are independently sampled as $b_j \sim \mathcal N(0,0.01^2)$. 

\underline{\emph{Init 1:}} 
For SNN, the weights are sampled according to $w_{ij} = \exp\bigl(\sigma z_{ij}-\frac{\sigma^2}{2}\bigr)$, where $z_{ij} \sim \mathcal N(0,1)$, with $\sigma=0.5$. 
The lognormal distribution guarantees positive weights while the above parametrization keeps their mean equal to one. Thresholds are sampled independently as $\log \theta_j \sim U(\log 0.25,\log 4)$. We refer to this setting as \emph{Init~1}. It is the default SNN initialization used for the experiments reported in the main text.

\underline{\emph{Init 2:}} 
 For SNNs we additionally consider a second initialization, denoted \emph{Init~2} to illustrate the sensitivity of the region complexity to initialization, while emphasizing that a systematic study of SNN initialization is beyond the scope of this work. 
The weights use the same lognormal parametrization as Init~1 but with the broader scale $\sigma=1.5$. 
To initialize the thresholds, let $t_{(1)}< \cdots< t_{(d)}$ denote the sorted presynaptic spike times.
For a candidate causal prefix of size $k$, the threshold at which the neuron changes between prefixes of sizes $k$ and $k+1$ is
\begin{equation*}
    \theta^\star_{j,k}(x) = \sum_{r=1}^{k} w_{(r)j} \bigl(t_{(k+1)} - t_{(r)}\bigr).
\end{equation*}
We assign different neurons different values $k_j$ and initialize
\begin{equation*}
    \theta_j = \operatorname{median}_{x\in\mathcal D_{\mathrm{cal}}} \theta^\star_{j,k_j}(x) \exp(\epsilon_j), \quad \text{where  }
    \epsilon_j\sim\mathcal N(0,0.08^2),
\end{equation*}
where $\mathcal D_{\mathrm{cal}}$ is a fixed small subset of the training data used only to calibrate the initialization. 

Thus, different neurons are initialized near different causal-prefix switching conditions. After each SNN layer, we additionally apply a fixed neuron-wise affine transformation $\widetilde t_j=\gamma_j t_j+\beta_j$, calibrated at initialization to help preserve variation in relative spike times across neurons through depth. These quantities are fixed after initialization and are never trained.

\underline{\emph{Init 3:}} 
For training experiments, we also consider another SNN training-oriented initialization \emph{Init~3}. 
Init~3 retains the positive lognormal weights of Init~1, but normalizes incoming weights and calibrates thresholds near a range of causal-prefix switching conditions, following the idea used in Init~2 and motivated more broadly by prior work on SNN-specific normalization and data-dependent initialization~\cite{wu2018spatiotemporal, bojkovic2024datadriven, che2024ettfs}.

\underline{\emph{Init 4:}} 
We also consider an initialization with arbitrary weights, 
\begin{equation*}
    w_{ij} = \sigma z_{ij}, \quad \text{where  } z_{ij} \sim \mathcal N(0,1).
\end{equation*}
Threshold distribution is kept the same as in Init~1. 
We denote this configuration as \emph{Arbitrary} in Figure~\ref{fig:cifar_mnist_pos_neg}. For arbitrary weights, a causal prefix is admissible only when its cumulative weight is positive and the corresponding candidate firing time occurs before the next presynaptic arrival.

\paragraph{Number of regions at initialization} 

Figure~\ref{fig:cifar_mnist_init2} shows that Init~2 produces larger estimated region counts than Init~1 shown in
Figure~\ref{fig:cifar_mnist_init1}. 
The qualitative behavior is consistent in both initializations, growing both with width and depth. The difference between the two initializations is already visible for depth-one networks. 
In the considered architectures, the empirical growth with depth does not appear exponential; rather, adding depth has a substantially larger effect when the earlier layers are wider. This partially supports the observation that sufficiently wide layers preserve a higher-dimensional affine image across depth, allowing subsequent layers to introduce additional subdivisions. Nonetheless, the results suggest that the number of regions can change significantly with the parameter initialization. 

\begin{figure}[t]
    \centering
    \begin{tabular}{
        @{}
        >{\centering\arraybackslash}m{0.01\textwidth}
        >{\centering\arraybackslash}m{0.97\textwidth}
        @{}
    }
        \rotatebox{90}{MNIST}
        &
        \includegraphics[width=\linewidth]
            {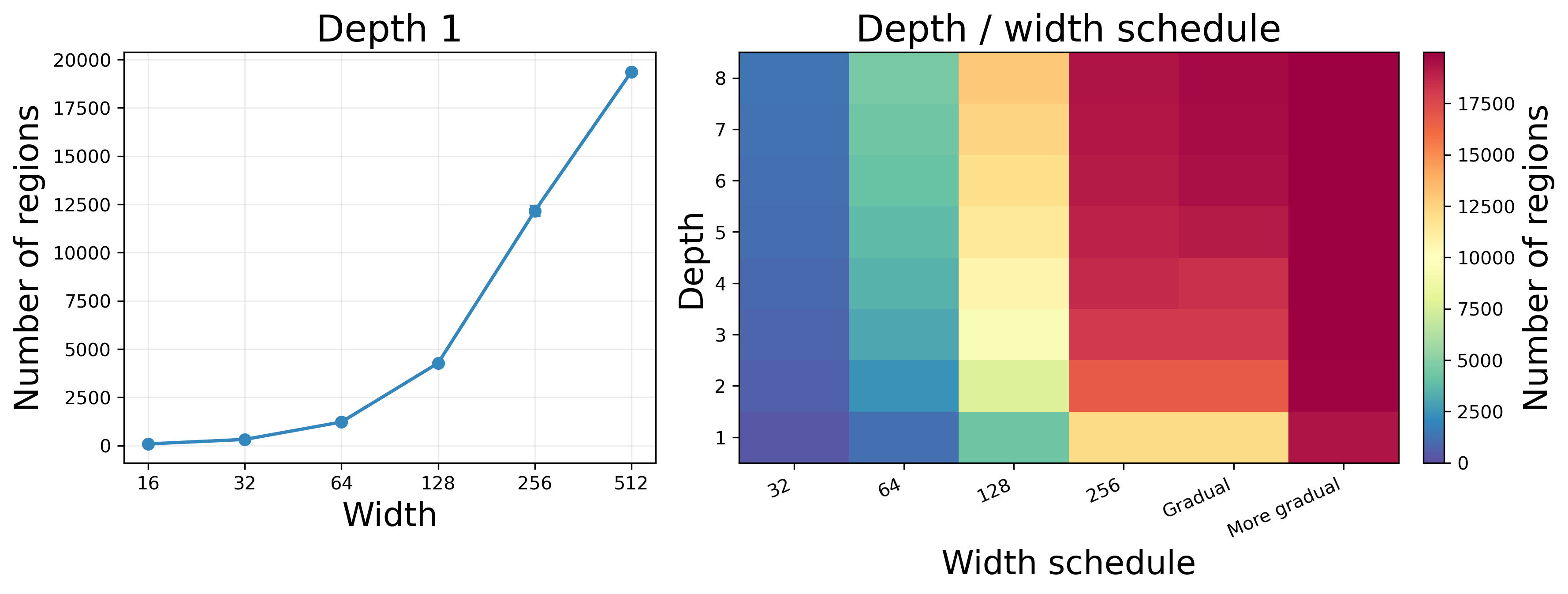}
        \\[0.5em]
        \rotatebox{90}{CIFAR-10}
        &
        \includegraphics[width=\linewidth]
            {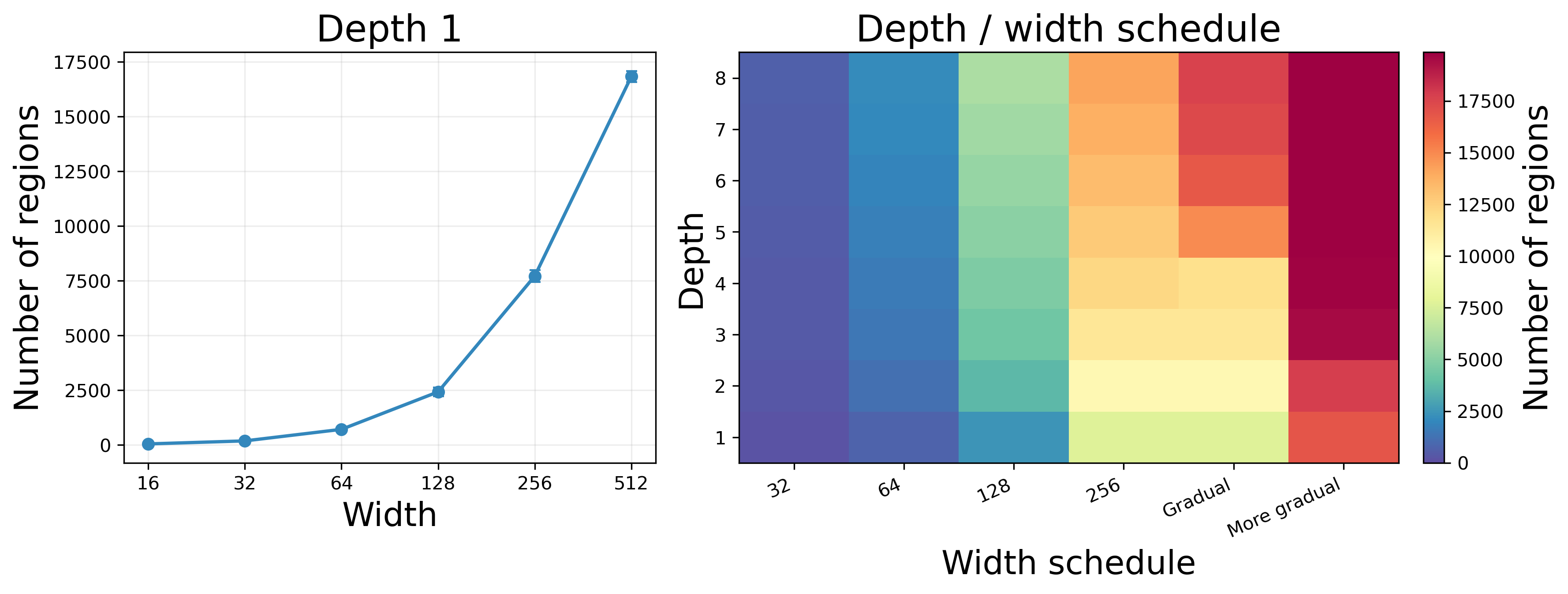}
    \end{tabular}
    \caption{
    Effect of width and depth on the number of causal regions of SNNs with alternative initialization (Init~2), evaluated for MNIST (top) and CIFAR-10 (bottom).}
    \label{fig:cifar_mnist_init2}
\end{figure}

\paragraph{Number of regions for positive vs arbitrary weights} 
Figure~\ref{fig:cifar_mnist_pos_neg} compares the two positive weight initializations Init~1 and Init~2 with the arbitrary weight initialization Init~4. 
It shows that the arbitrary weight SNN generally realizes more linear regions than Init~1; 
although Init~2 attains the largest region counts across most of the considered architectures. 

One might expect arbitrary weights to reduce the number of admissible causal prefixes, since negative cumulative input cannot trigger a spike. 
The observation that Init~4 nevertheless realizes more regions that Init~1 may be explained by cancellations between positive and negative weights, which can make cumulative prefix sums more sensitive to changes in arrival time ordering, leading to more frequent changes in the selected causal prefix. 
We leave a systematic study of this effect, including the interaction between sign, weight magnitude, and firing admissibility, to future work.

\begin{figure}[t]
    \centering
\begin{tabular}{
        @{}
        >{\centering\arraybackslash}m{0.01\textwidth}
        >{\centering\arraybackslash}m{0.97\textwidth}
        @{}
    }
        \rotatebox{90}{MNIST}
        &
            \includegraphics[width=\linewidth]
                {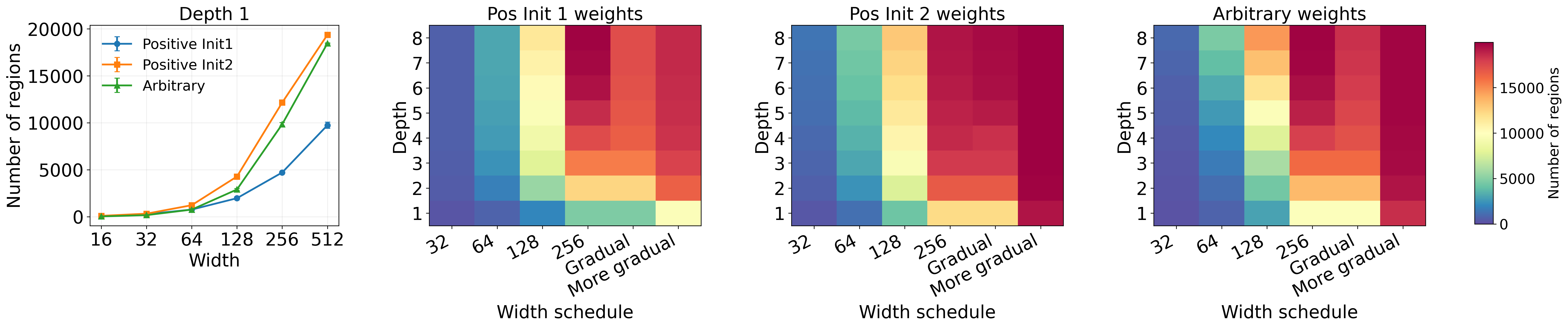}
        \\ 
        \rotatebox{90}{CIFAR-10}
        &
            \includegraphics[width=\linewidth]
                {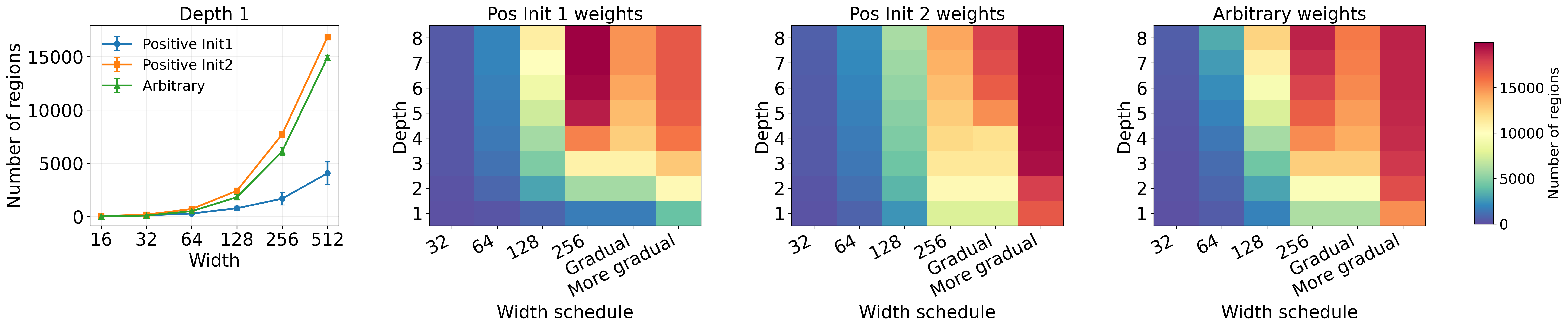}
    \end{tabular}

    \caption{
    Comparison of two positive weight initializations with arbitrary weights. Shown are the estimated region counts for arbitrary weight SNN, illustrating the effects of width and depth, evaluated for MNIST (top) and CIFAR-10 (bottom).}
    \label{fig:cifar_mnist_pos_neg}
\end{figure}

\begin{figure}[t]
    \centering
    \includegraphics[width=\textwidth]{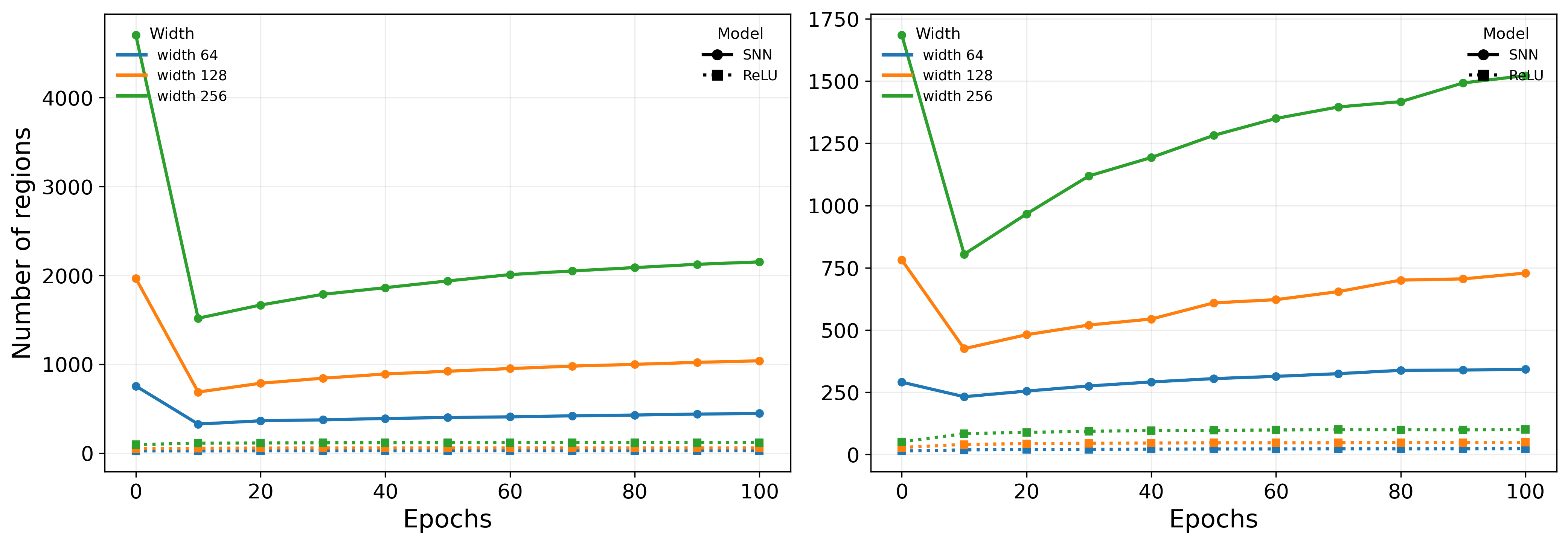}
    \caption{Region complexity during training. Estimated numbers of causal regions for depth one SNN and ReLU networks on MNIST (left) and CIFAR-10 (right) for widths $64$, $128$, and $256$. Region complexity drops sharply during the early stages of training relative to initialization and then partially recovers as training proceeds. Across widths and throughout training, the SNN realizes substantially more regions than the corresponding ReLU network. Solid and dotted curves denote SNN and ReLU networks, respectively.}
    \label{fig:snnrelutrainingregions}
\end{figure}

\begin{figure}[t]
    \centering
    \includegraphics[width=0.95\textwidth]{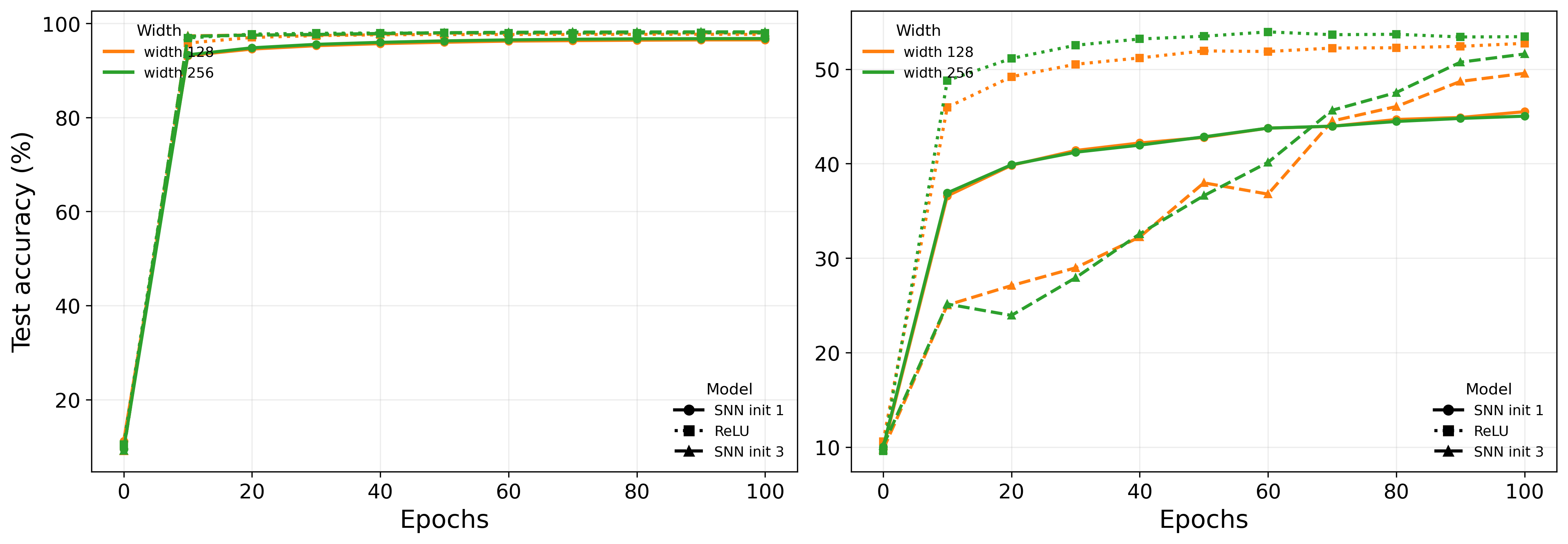}
    \caption{Classification accuracy during training. Test accuracy for depth one SNN and ReLU networks on MNIST (left) and CIFAR-10 (right) for widths $128$ and $256$. Both SNN and ReLU achieve comparable performance on MNIST. On CIFAR-10, the SNN with initialization Init~1 attains lower accuracy than ReLU, while the SNN with initialization Init~3 substantially improves performance and approaches the ReLU accuracy. Solid, dotted, and dashed curves denote SNN Init~1, ReLU, and SNN Init~3, respectively.}
    \label{fig:snnrelutestaccuracy}
\end{figure}

\paragraph{Number of regions during training} 
We train depth one SNN and ReLU networks for $100$ epochs and evaluate region complexity every $10$ epochs. 
We consider hidden widths $64$, $128$, and $256$. 
For each width, the SNN and ReLU have matched encoder, hidden-layer, and decoder dimensions. With delays set to zero, the two models also have same number of parameters. 
For the generic SNN--ReLU comparison, the batch size is $128$, the gradient norm is clipped at $5$, and the learning rate is $10^{-3}$ on MNIST and $10^{-4}$ on CIFAR-10. 
We restrict the training experiment to depth one, since training deeper TTFS networks require additional inter-layer scaling and normalization mechanisms that we do not address here. SNN weights and thresholds are trained while constrained to remain positive, and optimization uses AdamW.

As shown in Figure~\ref{fig:snnrelutrainingregions}, SNN region complexity
drops sharply during early training and subsequently partially recovers, while remaining substantially larger than that of the matched ReLU network. 
This behavior is qualitatively similar to observations reported for ANN \cite{hanin2019complexitylinearregions, hanin2019surprisinglyfewregion, tseran2021maxoutexpectedcomplexity}. 
The specific mechanism leading to this dynamics is still unknown, although some advances are being pursued in \cite{DNNs_grok_humayun_2024,local_complexity_regions_DNNs_patel_2024}.
We further observe that the estimated number of regions is lower for CIFAR-10 data than for MNIST data. The same hidden widths and corresponding SNN-ReLU architectures are used for both MNIST and CIFAR-10; only the input dimension changes with the dataset. A controlled experiment evaluating the number of regions at initialization while varying the input dimension would be an interesting direction for future work.

In Figure~\ref{fig:snnrelutestaccuracy}, we report test accuracy for the SNN (Init~1), ReLU, and the SNN (Init~3). 
On MNIST, all three models obtain comparable classification accuracy. 
Despite their different region counts, the models achieve comparable accuracy on MNIST. 
On CIFAR-10, the SNN with Init~1 performs below ReLU, whereas Init~3 substantially improves SNN performance and approaches the ReLU accuracy. 
These results highlight that the initialization strategy used for training SNNs can have a significant impact on the test performance. We made similar observations for the training error as well, thus the effect appears to be related to how initialization influences the hardness of optimization. 
Whereas initialization strategies have been explored intensively for ANNs, we still are missing a comparative systematic analysis for SNNs, which we identify as a promising direction for future work.
We also note that the observed performance difference cannot be attributed specifically to the positivity constraint. Both SNN initializations used in this experiment employ positive weights, and we do not train an arbitrary weight SNN. A systematic study separating the effects of initialization, positivity constraints, and optimization is left for future work.

\paragraph{Exact enumeration on two-dimensional affine slices} 
We complement the trajectory-based estimates with exact enumeration on bounded two-dimensional affine slices of the input space. 
Given three samples $x_1, x_2, x_3$, we construct an orthonormal basis $e_1, e_2$ for the span of $x_2 - x_1$ and $x_3-x_1$ and parameterize the slice as $x(u,v) = x_1 + ue_1 + ve_2$. We restrict $(u,v)$ to the bounded rectangular domain containing the selected samples. 

Over the affine slice, linear pieces correspond to linear pieces in $z = (u,v)$ and linear regions are represented as a polygons
 \begin{equation}
    R = \{z \in \R^2 \colon Az \leq b\}.
\end{equation}
We retain only full-dimensional polygons. 
For each candidate region, we compute its center by solving the following linear program
\begin{equation}
\begin{aligned}
    \max_{z,r} \quad & r,\\
    \rm{s.t.}\quad&
    A_i z + \|A_i\|_2 r \leq b_i, \quad \text{for } i = 1, \ldots, m,\\
    &r \geq 0 .
\end{aligned}
\end{equation}
A region is retained only if $r > 10^{-9}$. This excludes intersections that are only line segments or points and this simultaneously provides an interior point for subsequent subdivision.

For a ReLU neuron, its preactivation is affine inside every activation region. We therefore intersect each polygon with the two half-spaces corresponding to positive and negative preactivation and retain each full-dimensional subregion. 

For an SNN neuron $j$, let $t_i(z)$ denote its presynaptic spike times inside the current parent region. For a causal set $S_j$, accumulation of the spikes in $S_j$ gives the candidate firing time given in \eqref{eq:expSjcausalset}, which is affine in $z$. 
The causal set is valid exactly where
\begin{equation}
    t_i(z) \leq t^{S_j}(z), \quad i\in S_j \text{ and } t^{S_j}(z) \leq t_i(z), \quad i\notin S_j,
\end{equation}
which again gives linear inequalities in $(u,v)$. 
Importantly, regions are indexed by causal sets and not by the complete ordering of the presynaptic spike times. 
We do not test all $2^d$ possible causal sets. Starting from the interior point of a parent polygon, we use the sorted-prefix forward pass described above to obtain one feasible $S_j$, and then traverse neighboring feasible causal regions across their facets. After crossing a candidate facet, the forward pass is reevaluated; a new region is retained only if the causal set changes and its complete inequality system has positive two-dimensional interior. The resulting regions carry affine spike-time representations, allowing the same procedure to be applied recursively across neurons and layers.

All linear programs are solved with
\texttt{scipy.optimize.linprog} using the HiGHS backend. The ReLU construction follows the general polyhedral
enumeration principle of~\cite{tseran2021maxoutexpectedcomplexity}, while the
SNN case uses the causal set subdivision and adjacency traversal described above.

\end{document}